\documentclass[10pt,twocolumn,letterpaper]{article}
\usepackage[pagenumbers]{cvpr}      

\usepackage{amsmath,amsfonts,bm}

\def\eqref#1{Eq.~\ref{#1}}

\def\1{\bm{1}}

\DeclareMathAlphabet{\mathsfit}{\encodingdefault}{\sfdefault}{m}{sl}
\SetMathAlphabet{\mathsfit}{bold}{\encodingdefault}{\sfdefault}{bx}{n}

\newcommand{\pdata}{p_{\rm{data}}}

\newcommand{\E}{\mathbb{E}}

\newcommand{\R}{\mathbb{R}}

\ifdefined\pdfcompresslevel
\fi
\usepackage{amsmath,amssymb,amsthm,mathtools,bm}
\usepackage{capt-of,colortbl,array,float,placeins}
\usepackage{multirow}

\usepackage{algorithm,algorithmicx,algpseudocode}

\newcommand{\tdafaded}[1]{\textcolor{black!45}{#1}}
\newcommand{\tdamid}[1]{\textcolor{black!65}{#1}}
\newcommand{\tdadark}[1]{\textcolor{black!85}{#1}}
\newcommand{\headspace}{\hspace{.2em}}
\newcommand{\pref}{p_z}
\renewcommand{\pdata}{p}

\newcommand{\rvec}{\operatorname{rvec}}

\newcommand{\cL}{\mathcal{L}}

\newtheorem{proposition}{Proposition}

\usepackage{longtable}

\newsavebox{\tdaTrainingAlgorithmBody}
\newsavebox{\tdaSolverAlgorithmBody}
\newlength{\tdaAlgorithmBodyHeight}

\renewcommand{\paragraph}[1]{\vspace{1.25mm}\noindent\textbf{#1}}

\definecolor{cvprblue}{rgb}{0.21,0.49,0.74}
\usepackage[pagebackref,hypertexnames=false,breaklinks,colorlinks,allcolors=cvprblue]{hyperref}
\makeatletter
\let\TDACOriginalPageAnchor\Hy@EveryPageAnchor
\renewcommand{\Hy@EveryPageAnchor}{%
  \begingroup\Hy@hypertexnamesfalse\TDACOriginalPageAnchor\endgroup}
\makeatother
\usepackage{bookmark}

\def\confName{CVPR}
\def\confYear{2027}

\title{One-Step Generative Modeling via Training Dynamics Action}

\author{Zhangyong Liang$^{1}$ \qquad Ying Huang$^{2}$ \qquad Haibin Ling$^{1}$\thanks{Corresponding author: Haibin Ling (linghaibin@westlake.edu.cn).}\\
$^{1}$Department of Artificial Intelligence, Westlake University\quad $^{2}$Hangzhou Normal University\\
{\tt\small liangzhangyong@westlake.edu.cn}\quad {\tt\small yw52@hznu.edu.cn}\quad {\tt\small linghaibin@westlake.edu.cn}}

\makeatletter
\patchcmd{\@maketitle}{\vspace*{24pt}}{\vspace*{8pt}}{}{}
\patchcmd{\@maketitle}{\vspace*{12pt}}{\vspace*{0pt}}{}{}
\makeatother

\begin{document}
\twocolumn[{%
\renewcommand{\twocolumn}[1][]{#1}%
\maketitle
\begin{minipage}{\textwidth}
\centering
\includegraphics[width=0.94\linewidth]{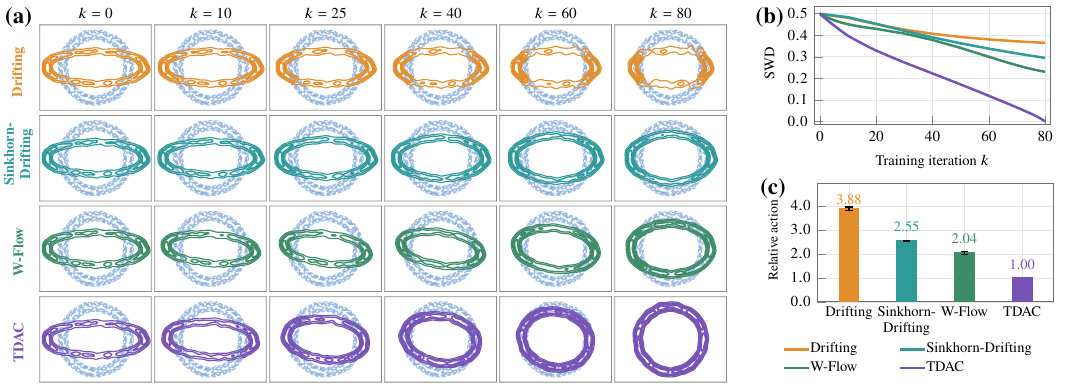}
\captionof{figure}{\textbf{Training Dynamics Action (TDAC).}
Comparison of TDAC with Drifting, Sinkhorn-Drifting, and W-Flow.
\textbf{(a)} Distribution trajectories during training toward a fixed data distribution.
\textbf{(b)} Distributional discrepancy measured by sliced Wasserstein distance (SWD).
\textbf{(c)} Action efficiency measured as action per unit progress, normalized to TDAC within each seed.
All methods retain one-step inference.}
\label{fig:tdaction_mechanism}
\par\vspace*{24pt}
\end{minipage}

}]
\begin{abstract}
Evolving the generated distribution during training allows a static generator to learn one-step generation. However, sample movements with similar distributional improvement may require different parameter-update costs and leave different fitting errors. To address this issue, we propose Training Dynamics Action (TDAC), which selects transport plans according to the joint fitting cost of sample movements. For a linearized generator, local generator action is the minimum sum of parameter-update cost and a feature-fitting penalty over shared parameter increments. We derive this action in closed form and evaluate it with a low-rank covariance estimate from feature responses to random parameter directions. Minimizing regularized action while retaining a prescribed fraction of the reference plan's predicted progress yields a transport plan that defines fixed regression targets for generator training. We bound actual first-order progress based on action and the update's excess quadratic fitting cost. Planning occurs during training; inference requires one generator evaluation. On ImageNet 256 × 256, TDAC achieves FID below 1.1 without distillation. Furthermore, TDAC reduces optimizer-time action while maintaining distributional accuracy and full mode coverage. At comparable FID, TDAC requires fewer generator evaluations than multi-step models while outperforming one-step baselines in generation quality. The code for our work will be released soon.
\end{abstract}

\section{Introduction}
\label{sec:introduction}

Generative modeling learns a map that transforms samples from a simple reference distribution into generated samples whose distribution matches a target data distribution.
Diffusion models~\citep{sohl2015deep,ho2020denoising} and score-based models~\citep{song2020score} make transport tractable through a sequence of denoising transformations.
Flow-matching~\citep{lipman2022flow} and rectified-flow~\citep{liu2022flow} models construct transport through learned velocity fields.
These models map reference samples to data through iterative updates during sampling.
An alternative is to evolve the generated distribution during training and learn a static map requiring a single Neural Function Evaluation (1-NFE) at inference~\citep{deng2026generative,han2026wflow}.

From this perspective, successive parameter updates may be viewed as a training-time evolution of the generator's output distribution.
A field prescribes desired movements of current generated samples for generator training.
Drifting~\citep{deng2026generative} constructs sample movements through attraction toward data and repulsion among generated samples.
Sinkhorn-Drifting~\citep{he2026sinkhorndrifting} uses two-sided Sinkhorn balancing.
W-Flow~\citep{han2026wflow} derives a field from a distributional energy, combining transport toward data with generated-sample self-interaction.

However, prescribed sample movements may differ from movements produced by shared parameter updates.
Each sample responds to the same parameter update, so fitting one movement may affect other samples.
Batch movements with similar predicted progress may require different parameter changes and leave different fitting errors. Learning cost combines parameter-update cost with a penalty on remaining feature error. Under a linear approximation of the generator, the minimum combined cost over parameter updates is local generator action\footnote{Action Matching~\citep{neklyudov2023action} connects dynamics with a variational action objective, motivating the action perspective used here.}.

Distribution matching permits different correspondences between noise inputs and data samples, allowing movement selection by distributional improvement and joint fitting cost.
Movements that are easy to fit may still make little progress toward the target data distribution.

This raises a natural question:
\begin{center}
\emph{How should the cost of learning sample movements guide the training dynamics of one-step generators?}
\end{center}
To answer this question, we propose \textbf{T}raining \textbf{D}ynamics \textbf{AC}tion (\textbf{TDAC}), a framework that selects training targets by local generator action for one-step generative modeling.
For a candidate batch movement, we evaluate local generator action in closed form, obtaining the Action-to-Go (ATG) value.
We measure predicted distributional progress by the first-order decrease in a distribution-matching objective.
When the reference plan predicts positive first-order distributional progress, TDAC selects a transport plan by minimizing regularized ATG while retaining a prescribed fraction of that progress.
Random parameter directions and low-rank computations enable action evaluation, and the selected plan supplies fixed regression targets.
A bound connects the selected plan's predicted progress to first-order progress of the actual parameter update.
In this way, TDAC guides training-time distribution evolution by selecting sample movements according to generator action while retaining predicted progress (see \cref{fig:tdaction_mechanism}).

Our main contributions are:
\begin{itemize}
    \item \textbf{Action-guided one-step generation.} We propose TDAC, a framework that selects transport plans according to the cost of learning sample movements. The selected plan retains a prescribed fraction of the reference plan's predicted distributional improvement.
    \item \textbf{Closed-form action for training dynamics.} We derive a closed-form ATG value under generator linearization. Random parameter directions enable low-rank action evaluation, and a bound connects planned progress to first-order progress of actual parameter updates.
    \item \textbf{High-quality one-step image generation.} On ImageNet $256\times256$, TDAC achieves FID 1.05 without distillation with one generator evaluation. TDAC outperforms the one-step baselines in generation quality (Table~\ref{tab:imagenet_results}).
\end{itemize}

\begin{figure*}[!t]
\centering
\includegraphics[width=\textwidth]{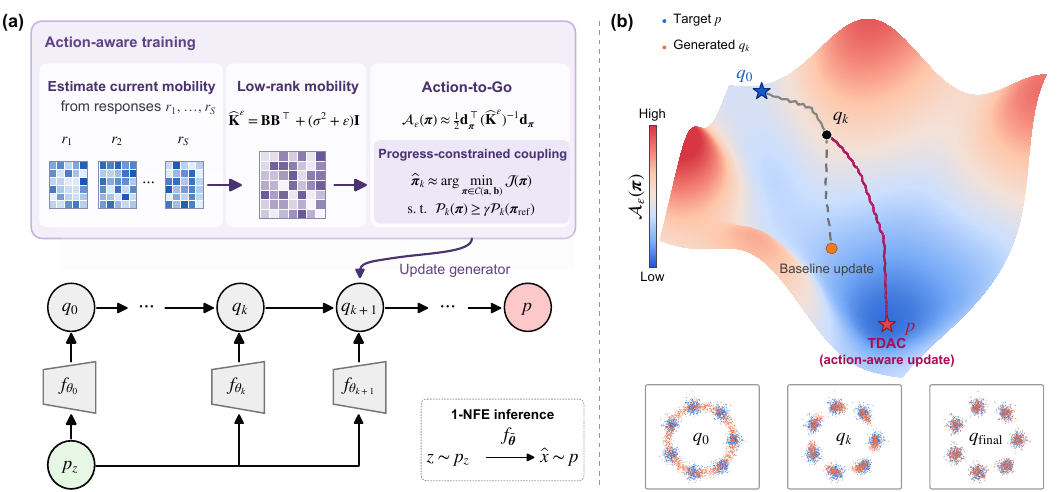}
\caption{\textbf{Action-aware one-step generation.}
\textbf{(a)}~Estimated mobility guides TDAC plan selection, followed by fixed-target training.
\textbf{(b)}~Action landscape shows illustrative paths on a relative-action surface; insets show data (blue) and generated samples (orange).}
\label{fig:tdaction_idea}
\end{figure*}

\section{Related Work}
\label{sec:related_work}

\paragraph{One- and few-step generation.}
Progressive Distillation~\citep{salimans2022progressive} reduces NFE. DMD~\citep{Yin_2024_CVPR} and DMD2~\citep{yin2024improved} match teacher distributions.
Consistency Models~\citep{song2023consistency} support distillation or direct training. sCM~\citep{lu2025simplifying} stabilizes continuous-time training.
Shortcut Models~\citep{frans2024one} condition on step size. MeanFlow~\citep{geng2025mean,geng2025improved} uses average velocity. Flow Map Matching~\citep{boffi2025flow} learns maps between times.
Minibatch optima may differ from population optima~\citep{boite2026expected}.
Model-Aligned Coupling (MAC)~\citep{lin2025mac} selects pairs by prediction error.

\paragraph{Distribution matching and training dynamics.}
Kernel matching~\citep{pmlr-v37-li15,li2017mmd} and entropic transport~\citep{genevay2018learning} train generators.
Maximum mean discrepancy (MMD) gradient flow~\citep{arbel_mmd_2019} evolves particles. FD-Loss~\citep{yang2026fdloss} post-trains generators using frozen features.
Drifting~\citep{deng2026generative} uses attraction--repulsion fields.
Sinkhorn-Drifting~\citep{he2026sinkhorndrifting} uses Sinkhorn balancing. W-Flow~\citep{han2026wflow} derives movements from distributional energies.
Non-conservative drifting fields~\citep{franz2026driftingfields} differ from energy gradients.
Neural tangent kernels~\citep{NEURIPS2018_5a4be1fa} capture output coupling.

\paragraph{Dynamic optimal transport and variational action.}
Optimal transport (OT) and Sinkhorn divergences measure distributional discrepancy~\citep{peyre2019computational,feydy2019interpolating}. Squared-distance OT measures movement in sample or feature space.
Benamou--Brenier~\citep{benamou2000computational} minimizes kinetic action between endpoint distributions.
The Jordan--Kinderlehrer--Otto scheme~\citep{jordan1998variational} balances energy against Wasserstein transport cost.
Variational~\citep{pmlr-v162-fan22d} and parameterized~\citep{jin2024parameterized} Wasserstein gradient flows evolve distributions from samples or parameterized models.
OT-Flow~\citep{onken2021otflow} regularizes continuous normalizing flows with transport costs.
Action Matching~\citep{neklyudov2023action} learns dynamics from distribution paths through variational action.

\section{TDAC for Generation}
\label{sec:method}

We propose TDAC, which selects sample movements by joint fitting cost while retaining predicted distributional progress. We relate sample movements to shared parameter responses, define action, select a transport plan, and train the generator in feature space (see Fig.~\ref{fig:tdaction_idea}).

\subsection{Sample Movements during Training}
\label{sec:atg}

We first relate prescribed sample movements to the changes produced by shared parameter updates.
Let $\pref$ denote the input distribution on $\mathcal Z$ and $p$ the data distribution on $\mathcal X$.
A static generator $f_{\bm\theta}:\mathcal Z\to\mathcal X$ induces
\begin{equation}
    \mathbf z\sim\pref,\quad \mathbf x=f_{\bm\theta}(\mathbf z),\quad
    q_{\bm\theta}=(f_{\bm\theta})_\#\pref .
    \label{eq:one_step_transport_map}
\end{equation}
where $\#$ denotes the pushforward of a distribution through a map. Training matches $q_{\bm\theta}$ to $p$ through parameter updates. Inference evaluates the generator once, using exponential moving average (EMA) parameters.
We write $q_k=q_{\bm\theta_k}$ at training step $k$ and suppress conditioning variables.
A sample-space field $V_k$ specifies sample movements,
\begin{equation}
    \mathcal T_k(\mathbf x)=\mathbf x+hV_k(\mathbf x),\qquad
    \widetilde q_{k+1}=(\mathcal T_k)_\#q_k,
    \label{eq:local_velocity_map}
\end{equation}
where $h>0$ is the field step size. The finite parameter update produces $q_{\bm\theta_{k+1}}$, which can differ from $\widetilde q_{k+1}$.

For a Sinkhorn energy, W-Flow~\citep{han2026wflow} supplies the reference field for predicted distributional progress (Appendix~\ref{app:sinkhorn_reference}). We use $\rho$ for entropy regularization, $S_\rho$ for the Sinkhorn divergence, and $T^\rho$ for Sinkhorn barycentric maps.
Transport planning operates on generated and data features obtained through the frozen feature map $\psi$.

\setcounter{dbltopnumber}{1}
\begin{figure*}[t]
\begin{lrbox}{\tdaTrainingAlgorithmBody}
\begin{minipage}[t]{0.485\textwidth}
\strut\textbf{Input:} $f_{\bm\theta}$, $\pref$, $\pdata$, $\psi$, $\{\phi_\ell\}_{\ell=1}^{L}$, $\mathbf R$, $\mathcal F$
\begin{algorithmic}[1]
\State Initialize optimizer; $\bar{\bm\theta}\leftarrow\bm\theta$, $k=0$, $\widehat{\bm\pi}_{-1}\leftarrow\textsc{None}$
\Repeat
\If{$k=0$ or a support-and-mobility refresh step}
\State Sample $\mathbf Z$ and data $\{\mathbf y_j\}_{j=1}^M\sim p$
\State Cache $\mathbf Y_{j,:}\leftarrow\psi(\mathbf y_j)^\top$, $j=1,\ldots,M$
\State Probe estimation/validation directions (Eq.~\ref{eq:active_probe_mobility})
\State Form $\widehat{\mathbf K}$ and reserve validation directions
\EndIf
\State Compute current detached $\mathbf U\leftarrow\psi(f_{\bm\theta}(\mathbf Z))$
\State Compute $\mathbf g_k$ using Eq.~\ref{eq:distributional_progress}
\State $\widehat{\bm\pi}_k\leftarrow\textsc{ATG}(\mathbf U,\mathbf Y,\widehat{\mathbf K},\mathbf g_k,\widehat{\bm\pi}_{k-1},\varepsilon,\rho,\lambda,\gamma)$
\State Form fixed regression targets $\{\widetilde{\mathbf u}_i^\ell\}_{\ell=1}^{L}$ using Eq.~\ref{eq:shared_targets}
\State $\bm\theta\leftarrow\operatorname{Optimizer}(\bm\theta,\nabla_{\bm\theta}\cL)$ using Eq.~\ref{eq:shared_targets}
\State Update EMA, log diagnostics, and set $k\leftarrow k+1$
\Until{the training budget is exhausted}
\State \strut\textbf{return} the one-step generator $f_{\bar{\bm\theta}}$
\end{algorithmic}
\end{minipage}\end{lrbox}
\begin{lrbox}{\tdaSolverAlgorithmBody}
\begin{minipage}[t]{0.485\textwidth}
\strut\textbf{Input:} $\mathbf U,\mathbf Y,\widehat{\mathbf K},\mathbf g_k,\widehat{\bm\pi}_{k-1}$; $\varepsilon,\rho,\lambda,\gamma,n_{\mathrm{plan}},\eta_{\mathrm{plan}}$
\algrenewcommand{\algorithmiccomment}[1]{\hfill\textcolor{black!55}{$\triangleright$ #1}}
\begin{algorithmic}[1]
\State $\mathbf a\leftarrow N^{-1}\mathbf1_N,\;\mathbf b\leftarrow M^{-1}\mathbf1_M$ \Comment{Source/target weights}
\State $\bm\pi_{\mathrm{ref}}\leftarrow\textsc{Sinkhorn}(\|\mathbf Y_{j,:}-\mathbf U_{i,:}\|_2^2,\mathbf a,\mathbf b,\rho)$
\State Compute $\mathcal P_k(\bm\pi_{\mathrm{ref}})$ from Eq.~\ref{eq:distributional_progress} \Comment{Progress}
\If{$\mathcal P_k(\bm\pi_{\mathrm{ref}})\leq0$}
\State \strut\textbf{return} $\bm\pi_{\mathrm{ref}}$ \Comment{Fallback}
\EndIf
\If{supports align}
\State $\bm\pi_k^0\leftarrow\widehat{\bm\pi}_{k-1}$ \Comment{Plan reuse}
\Else
\State $\bm\pi_k^0\leftarrow\bm\pi_{\mathrm{ref}}$ and $\lambda\leftarrow0$ \Comment{Reset}
\EndIf
\For{$t=0,\ldots,n_{\mathrm{plan}}-1$} \Comment{Mirror steps}
\State $\widetilde{\bm\pi}_k^{t+1}\leftarrow\bm\pi_k^t\odot\exp[-\eta_{\mathrm{plan}}\nabla_{\bm\pi}\mathcal J_{\widehat{\mathbf K}}(\bm\pi_k^t)]$
\State $\bm\pi_k^{t+1}\leftarrow\textsc{Project}(\widetilde{\bm\pi}_k^{t+1})$ \Comment{Feasibility}
\EndFor
\State \strut\textbf{return} $\bm\pi_k^{n_{\mathrm{plan}}}$ if feasible, otherwise $\bm\pi_{\mathrm{ref}}$
\end{algorithmic}
\end{minipage}\end{lrbox}
\setlength{\tdaAlgorithmBodyHeight}{\dimexpr\ht\tdaTrainingAlgorithmBody+\dp\tdaTrainingAlgorithmBody\relax}
\ifdim\dimexpr\ht\tdaSolverAlgorithmBody+\dp\tdaSolverAlgorithmBody\relax>\tdaAlgorithmBodyHeight
  \setlength{\tdaAlgorithmBodyHeight}{\dimexpr\ht\tdaSolverAlgorithmBody+\dp\tdaSolverAlgorithmBody\relax}
\fi
\begin{minipage}[t]{0.485\textwidth}
\begin{algorithm}[H]
\caption{TDAC Training}
\label{alg:tdaction_training}
\vbox to \tdaAlgorithmBodyHeight{\copy\tdaTrainingAlgorithmBody\vfil}
\end{algorithm}
\end{minipage}
\hfill
\begin{minipage}[t]{0.485\textwidth}
\begin{algorithm}[H]
\caption{Progress-Constrained ATG Solver}
\label{alg:batch_atg_plan}
\vbox to \tdaAlgorithmBodyHeight{\copy\tdaSolverAlgorithmBody\vfil}
\end{algorithm}
\end{minipage}
\end{figure*}

\label{sec:generator_mobility}

\paragraph{Feature displacements.}
For sampled latent inputs $\mathbf Z=(\mathbf z_1,\ldots,\mathbf z_N)$, let $\mathbf x_i=f_{\bm\theta}(\mathbf z_i)$ be generated samples and $\mathbf y_j\sim p$, $j=1,\ldots,M$, be data samples.
A frozen planning map $\psi:\mathcal X\to\mathbb R^d$ maps generated and data samples to column feature vectors $\mathbf u_i=\psi(\mathbf x_i)$ and $\mathbf v_j=\psi(\mathbf y_j)$. The batch feature matrices and Jacobian are
\begin{equation}
    \begin{aligned}
    \mathbf U&\coloneqq[\psi(\mathbf x_1),\ldots,\psi(\mathbf x_N)]^\top\in\mathbb R^{N\times d},\\
    \mathbf Y&\coloneqq[\psi(\mathbf y_1),\ldots,\psi(\mathbf y_M)]^\top\in\mathbb R^{M\times d},\\
    \mathbf u_{\bm\theta}&\coloneqq\rvec(\mathbf U)\in\mathbb R^{Nd},\qquad
    \mathbf J\coloneqq\frac{\partial\mathbf u_{\bm\theta}}{\partial\bm\theta^\top}.
    \end{aligned}
    \label{eq:batch_feature_map}
\end{equation}
where $\rvec$ denotes row-wise vectorization, stacking sample rows into a column vector, and $P=\dim\bm\theta$ denotes the generator parameter dimension.
The matrix $\mathbf J\in\mathbb R^{Nd\times P}$ is the Jacobian of the generated features with respect to the shared parameters.

Denote $\mathbb R_+=[0,\infty)$ and $\mathbb R_{>0}=(0,\infty)$. Let $\mathbf a\in\mathbb R_{>0}^{N}$ and $\mathbf b\in\mathbb R_{>0}^{M}$ be the probability weights of the generated and data samples, respectively. The weights are strictly positive because Eq.~\ref{eq:batch_displacement} inverts $\mathbf D$.
A transport plan is a nonnegative matrix $\bm\pi\in\mathbb R_+^{N\times M}$, whose scalar entry $\pi_{ij}$ moves probability mass from generated sample $i$ to data sample $j$.
The feasible plans form the coupling set
\begin{equation}
    \mathcal C(\mathbf a,\mathbf b)
    \coloneqq\{\bm\pi\in\mathbb R_+^{N\times M}:\ \bm\pi\mathbf1_M=\mathbf a,\ \bm\pi^\top\mathbf1_N=\mathbf b\},
    \label{eq:coupling_polytope}
\end{equation}
where $\mathbf1_N$ and $\mathbf1_M$ are all-ones vectors of the matching dimensions.
Coupling $\bm\pi\in\mathcal C(\mathbf a,\mathbf b)$ defines the displacement
\begin{equation}
    \mathbf d_{\bm\pi}\coloneqq\rvec\!\left(\mathbf D^{-1}\bm\pi\mathbf Y-\mathbf U\right),
    \qquad \mathbf D=\operatorname{diag}(\mathbf a),
    \label{eq:batch_displacement}
\end{equation}
where each factor $1/a_i$ converts coupling masses into conditional target weights.
Each row of $\mathbf D^{-1}\bm\pi\mathbf Y$ is the coupling-weighted mean of data features. Subtracting generated features gives the requested displacement. Couplings can give different displacements with fixed marginals.

\paragraph{Shared-parameter response.}
We compare requested feature movements with changes from a shared parameter update. Each sample responds to parameter changes, but these responses are coupled through the network.
Locally, a parameter increment $\delta\bm\theta$ changes the stacked features as
\begin{equation*}
    \mathbf u_{\bm\theta+\delta\bm\theta}-\mathbf u_{\bm\theta}
    \approx \mathbf J\,\delta\bm\theta .
\end{equation*}
For positive-definite $\mathbf R$ and $\varepsilon>0$, define
\begin{equation}
    \mathbf K\coloneqq \mathbf J\mathbf R^{-1}\mathbf J^\top,\qquad \mathbf K^\varepsilon\coloneqq\mathbf K+\varepsilon\mathbf I,
    \label{eq:batch_mobility}
\end{equation}
where $\mathbf K$ describes how batch features respond jointly to shared parameter changes.
The cross-sample blocks of $\mathbf K$ are $\mathbf J_i\mathbf R^{-1}\mathbf J_j^\top$, where $\mathbf J_i=\partial\mathbf u_i/\partial\bm\theta^\top\in\mathbb R^{d\times P}$ is the per-sample feature Jacobian, and similarly for $j$.
The Jacobian, metric, and mobility are evaluated at the current generator parameters and held fixed during each planning solve.

\label{sec:field_energy_action}
Distributional objectives evaluate progress, and local generator action evaluates joint fitting cost. \Cref{tab:field_energy_action} compares these training principles. The regularized target-selection problem is given in Eq.~\ref{eq:progress_constrained_batch_atg}.

\begin{table}[t]
\centering
\caption{\textbf{Training-time transport principles.}
TDAC selects transport by local generator action while preserving progress.}
\label{tab:field_energy_action}
{\small
\setlength{\tabcolsep}{4.0pt}
\renewcommand{\arraystretch}{1.18}
\resizebox{\columnwidth}{!}{%
\begin{tabular}{lcc}
\toprule
\textbf{Training Principle} & \textbf{Defining Object} & \textbf{Training-Time Rule} \\
\midrule
Drifting
& $V_{\mathrm{drift}}(q,p)$
& $\mathbf x^{+}=\mathbf x+hV_{\mathrm{drift}}(\mathbf x)$ \\
Sinkhorn-Drifting
& Sinkhorn-corrected field
& $V^\rho=T^\rho_{q,p}-T^\rho_{q,q}$
\\
W-Flow
& $\mathcal F_\rho(q)=S_\rho(q,p)$
& $V=-\nabla_{\mathbf x}\!\dfrac{\delta\mathcal F_\rho}{\delta q}$ \\
\midrule
\textbf{TDAC}
& $\displaystyle \mathcal A_\varepsilon(\bm\pi)=\tfrac12\mathbf d_{\bm\pi}^\top(\mathbf K^\varepsilon)^{-1}\mathbf d_{\bm\pi}$
& $\bm\pi_k^\star$ from Eq.~\ref{eq:progress_constrained_batch_atg} \\
\bottomrule
\end{tabular}}
}
\end{table}

\subsection{Training Action for Sample Movements}
\label{sec:batch_atg}

Accordingly, Algorithms~\ref{alg:tdaction_training} and~\ref{alg:batch_atg_plan} outline TDAC training: estimate mobility, select a progress-constrained plan, form fixed feature targets, and update the generator.

\paragraph{Defining the action.} 
Next, we use the shared parameter response to measure the cost of learning a proposed batch movement.
The Benamou--Brenier formulation~\citep{benamou2000computational} and Action Matching~\citep{neklyudov2023action} motivate evaluating motion through a variational action.
For a fixed positive-definite $\mathbf R$, the quadratic update cost equals the minimum kinetic action over a unit interval for a given parameter increment (Appendix~\ref{app:batch_atg}).
Under the current generator linearization, local generator action with $\varepsilon>0$ is
\begin{equation}
    \mathcal A_\varepsilon(\bm\pi)
    \coloneqq
    \min_{\delta\bm\theta\in\mathbb R^{P}}
    \left\{
    \frac12\delta\bm\theta^\top \mathbf R\delta\bm\theta
    +\frac{1}{2\varepsilon}
    \left\|\mathbf J\delta\bm\theta-\mathbf d_{\bm\pi}\right\|_2^2
    \right\},
    \label{eq:batch_atg_problem}
\end{equation}
where the two terms measure parameter-update cost and error between requested and predicted feature movements.
Residual error results from displacement components beyond the linearized response range or from the parameter-update penalty.
This minimum is the ATG value. Smaller $\varepsilon$ increases the fitting penalty.

\paragraph{Closed-form solution.} 
The first-order optimality condition for Eq.~\ref{eq:batch_atg_problem} is
\begin{equation*}
    \bigl(\mathbf R+\varepsilon^{-1}\mathbf J^\top \mathbf J\bigr)
    (\delta\bm\theta)^\star_{\bm\pi}=\varepsilon^{-1}\mathbf J^\top \mathbf d_{\bm\pi}.
\end{equation*}
Solving this system gives the action value and optimal local parameter increment in closed form,
\begin{equation}
    \begin{aligned}
    \mathcal A_\varepsilon(\bm\pi)
    &=\tfrac12\mathbf d_{\bm\pi}^\top(\mathbf K^\varepsilon)^{-1}\mathbf d_{\bm\pi},\\
    (\delta\bm\theta)^\star_{\bm\pi}
    &=\mathbf R^{-1}\mathbf J^\top(\mathbf K^\varepsilon)^{-1}\mathbf d_{\bm\pi}
    \end{aligned},
    \label{eq:batch_atg_value}
\end{equation}
where $\mathbf I$ is the identity matrix and all unqualified vector norms are Euclidean. The first line supplies the action used for plan selection in Eq.~\ref{eq:progress_constrained_batch_atg}.
For a fixed component along an eigenvector of $\mathbf K$, a smaller eigenvalue increases action cost and reduces the predicted response.
These expressions are exact for the fixed linear model; training fits fixed targets (Eq.~\ref{eq:shared_targets}). Appendix~\ref{app:convergence} bounds actual progress under smoothness and a positive-definite metric, accounting for fitting and finite-step errors (\cref{tab:execution_progress_diagnostics}).

\paragraph{Field form of the batch response.}
Fix the planning batch with $\mathbf a>0$, a feasible coupling $\bm\pi\in\mathcal C(\mathbf a,\mathbf b)$, $\varepsilon>0$, and the positive-semidefinite mobility $\mathbf K$. To compute the feature gradient, write the stacked features as $\mathbf u\in\mathbb R^{Nd}$, with current value $\mathbf u_{\bm\theta}$:
\[
 \mathcal A_\varepsilon(\bm\pi;\mathbf u)
 =\tfrac12\mathbf d_{\bm\pi}(\mathbf u)^\top
 (\mathbf K^\varepsilon)^{-1}\mathbf d_{\bm\pi}(\mathbf u).
\]
where $\mathbf d_{\bm\pi}(\mathbf u)\coloneqq\rvec(\mathbf D^{-1}\bm\pi\mathbf Y)-\mathbf u$.
Holding $\bm\pi$, $\mathbf Y$, and $\mathbf K$ fixed, we evaluate the gradient at $\mathbf u=\mathbf u_{\bm\theta}$:
\[
 \left.\nabla_{\mathbf u}\mathcal A_\varepsilon(\bm\pi;\mathbf u)\right|_{\mathbf u=\mathbf u_{\bm\theta}}
 =-(\mathbf K^\varepsilon)^{-1}\mathbf d_{\bm\pi}.
\]
Let $\mathbf H^\varepsilon=\mathbf K(\mathbf K^\varepsilon)^{-1}$, with $d\times d$ blocks $\mathbf H^\varepsilon_{in}$. For planning features $\mathbf u_i$ and $\mathbf v_j$, define
\[
 \mathbf V_{p,i}^\varepsilon=\sum_{n,j}\frac{\pi_{nj}}{a_n}\mathbf H^\varepsilon_{in}(\mathbf v_j-\mathbf u_i),\;
 \mathbf V_{q,i}^\varepsilon=\sum_n\mathbf H^\varepsilon_{in}(\mathbf u_n-\mathbf u_i).
\]
Stacking the field components gives $\mathbf V_p^\varepsilon,\mathbf V_q^\varepsilon\in\mathbb R^{Nd}$. The row constraint $\sum_j\pi_{nj}=a_n$ gives, for each $i$,
\[
 \mathbf V_{p,i}^\varepsilon-\mathbf V_{q,i}^\varepsilon
 =\sum_n\mathbf H^\varepsilon_{in}\left(\sum_j\frac{\pi_{nj}}{a_n}\mathbf v_j-\mathbf u_n\right)
 =[\mathbf H^\varepsilon\mathbf d_{\bm\pi}]_i,
\]
where $[\cdot]_i$ extracts the $i$-th $d$-dimensional block. Combining the two identities gives
\begin{equation}
 \mathbf V^\varepsilon
 =-\mathbf K\nabla_{\mathbf u}\mathcal A_\varepsilon
 =\mathbf V_p^\varepsilon-\mathbf V_q^\varepsilon.
 \label{eq:action_interaction_field}
\end{equation}
where $\mathbf V^\varepsilon=\mathbf J(\delta\bm\theta)^\star_{\bm\pi}$ is the optimal linearized batch response. The matrix-valued blocks $\mathbf H^\varepsilon_{in}$ can contain signed entries, so the two terms describe the aggregate batch response.
Substituting an estimated mobility in Eq.~\ref{eq:action_interaction_field} gives an approximation to the optimal linearized response. If a numerical plan has row residual $e_n=\sum_j\pi_{nj}/a_n-1$, the two-term difference additionally contains $-\sum_n\mathbf H^\varepsilon_{in}e_n\mathbf u_i$; zero row residual recovers the two-term decomposition.

\subsection{Selecting Movements by Action}
\label{sec:target_selection}
We minimize regularized action while retaining a prescribed fraction of the reference plan's predicted progress.

\paragraph{Predicted progress.}
For a fixed planning batch $\mathcal B$ with reference features $\mathbf Y$, let $\mathcal F$ denote the supplied differentiable energy on the batch.
The gradient $\mathbf g_k=\nabla_{\mathbf u}\mathcal F(\mathbf u_{\bm\theta})\in\mathbb R^{Nd}$ defines first-order progress
\begin{equation}
    \mathcal P_k(\bm\pi)\coloneqq-\mathbf g_k^\top \mathbf d_{\bm\pi} .
    \label{eq:distributional_progress}
\end{equation}
For a Sinkhorn energy, we evaluate the reference field (Appendix~\ref{app:sinkhorn_reference}) on the feature distributions $\psi_\#q$ and $\psi_\#p$. The progress $\mathcal P_k(\bm\pi)$ is the predicted energy decrease along $\mathbf d_{\bm\pi}$.

\paragraph{Plan selection.}
Let $\bm\pi_{\mathrm{ref}}$ be the entropy-regularized reference transport plan with cost $\|\psi(\mathbf x_i)-\psi(\mathbf y_j)\|_2^2$ on the same supports.
We add entropy and persistence penalties to the action in Eq.~\ref{eq:batch_atg_value}, giving the regularized objective $\mathcal J$.
\begin{equation}
    \begin{aligned}
    \min_{\bm\pi\in\mathcal C(\mathbf a,\mathbf b)}&
    \mathcal A_\varepsilon(\bm\pi)+\rho\operatorname{KL}(\bm\pi\|\mathbf a\mathbf b^\top)+\lambda\operatorname{KL}(\bm\pi\|\widehat{\bm\pi}_{k-1})\\
    \text{s.t.}\;&
    \mathcal P_k(\bm\pi)
    \geq\gamma\mathcal P_k(\bm\pi_{\mathrm{ref}}),
    \; 0<\gamma\leq1 .
    \end{aligned}
    \label{eq:progress_constrained_batch_atg}
\end{equation}
where $\operatorname{KL}$ denotes generalized Kullback--Leibler divergence. $\bm\pi_k^\star$ denotes an optimal plan and $\widehat{\bm\pi}_k$ the numerical plan. The solver optimizes $\mathcal J_{\widehat{\mathbf K}}$ with mobility estimate $\widehat{\mathbf K}$.
The weights $\rho,\lambda\geq0$ control entropy and persistence; $\gamma$ specifies the retained fraction of the reference plan's first-order progress. Persistence requires matching latent and data indices across successive plans. When $\mathcal P_k(\bm\pi_{\mathrm{ref}})\leq0$, the solver returns $\bm\pi_{\mathrm{ref}}$. 
PROJECT applies Sinkhorn scaling and a scalar dual update, returning $\bm\pi_{\mathrm{ref}}$ if infeasible (Appendix~\ref{app:regression_details}).

\subsection{TDAC Training in Feature Space}
\label{sec:executor}

\paragraph{Mobility estimation.}
For plan selection, we estimate batch mobility from random parameter directions and their linearized feature responses:
\begin{equation}
    \bm\xi_s\sim\mathcal N(\mathbf 0,\mathbf I),\qquad
    \mathbf r_s=\mathbf J\mathbf R^{-1/2}\bm\xi_s,
    \label{eq:active_probe_mobility}
\end{equation}
where $s=1,\ldots,n_{\mathrm{probe}}$ indexes independent random vectors of the parameter dimension. Each $\mathbf r_s\in\mathbb R^{Nd}$ is the linearized feature response to the parameter direction $\mathbf R^{-1/2}\bm\xi_s$, computed using a Jacobian--vector product.

Conditioned on the current batch, parameters, and metric, the sample covariance of feature responses is unbiased for $\mathbf K$.
One set of random directions estimates mobility for plan selection. An independent set estimates the action of the selected plan.
For scalability, we approximate the response covariance with a randomized low-rank matrix,
\begin{equation}
    \widehat{\mathbf K}=\mathbf B\mathbf B^\top+\sigma^2\mathbf I,
    \label{eq:low_rank_batch_mobility}
\end{equation}
where $\mathbf B\in\mathbb R^{Nd\times r}$, $r$ is the retained rank, and $\sigma^2\geq0$ is the isotropic regularization term.
Rank truncation and isotropic regularization can introduce bias into the covariance estimate.
We apply the Woodbury identity~\citep{hager1989updating} to evaluate the action in Eq.~\ref{eq:batch_atg_value} without forming an $Nd\times Nd$ matrix.

\paragraph{Feature regression.}
The chosen plan defines weighted feature means in training, from which we construct fixed regression targets.
For $L$ frozen feature maps $\phi_\ell$, with feature tensors flattened when taking norms, let
\begin{equation}
    \mathbf u_i^\ell=\phi_\ell(\mathbf x_i).
    \label{eq:shared_plan_features}
\end{equation}
The fixed targets and regression loss are
\begin{gather}
    \mathbf t_i^\ell\coloneqq\frac1{a_i}\sum_j\widehat\pi_{k,ij}\,\phi_\ell(\mathbf y_j),
    \qquad \widehat\pi_{k,ij}:=(\widehat{\bm\pi}_k)_{ij},
    \label{eq:shared_targets}\\
    \widetilde{\mathbf u}_i^\ell
    =\operatorname{sg}\!\left[
    \mathbf u_i^\ell+\eta\frac{\mathbf t_i^\ell-\mathbf u_i^\ell}{r_\ell}
    \right],\quad
    \cL=\sum_{\ell,i}\frac{w_\ell}{N}
    \left\|\mathbf u_i^\ell-\widetilde{\mathbf u}_i^\ell\right\|_2^2.
    \notag
\end{gather}
where $\phi_\ell:\mathcal X\to\mathbb R^{d_\ell}$ is frozen, $\eta>0$ is the target step size, and $w_\ell\geq0$ are loss weights. The sum spans $L$ levels and $N$ samples.
$\mathbf u_i^\ell,\mathbf t_i^\ell,\widetilde{\mathbf u}_i^\ell\in\mathbb R^{d_\ell}$ are the prediction, weighted mean, and fixed target, respectively.
$\operatorname{sg}$ denotes stop-gradient; $r_\ell>0$ is a normalization scale fixed during regression. Planning evaluates raw displacement $\mathbf d_{\bm\pi}$; Appendix~\ref{app:convergence} analyzes actual parameter updates.

During regression, the plan and targets are fixed; gradients pass through $\phi_\ell$ to the generator
(see Appendix~\ref{app:regression_details}).


\algrenewcommand\algorithmicindent{0.55em}

\section{Implementation}
\label{sec:implementation}

We perform image generation on ImageNet~\citep{deng2009imagenet} at $256\times256$. Planning and guidance details are in Appendix~\ref{app:planning_implementation_details}.

\paragraph{Tokenizer.}
Stable Diffusion's variational autoencoder (SD-VAE)~\citep{rombach2022highresolution} encodes $256\times256$ images as $32\times32\times4$ latents.

\paragraph{Architecture.}
DiT~\citep{peebles2023scalable} B/2, L/2, XL/2 use $2\times2$ patches and class and guidance-scale conditioning~\citep{deng2026generative,han2026wflow}.

\paragraph{Feature Extractor.}
A frozen latent masked autoencoder (latent-MAE) provides planning features $\psi$ and regression features $\{\phi_\ell\}$; \cref{tab:tdac_planning_configuration} specifies the metric, layers, dimensions, and normalization.

\paragraph{Training.}
A class-restricted plan fixes targets for the ordered batch across workers; \cref{tab:tdac_planning_configuration} gives grouping and mobility structure. AdamW trains the generator; sampling uses EMA parameters (Algorithms~\ref{alg:tdaction_training} and~\ref{alg:batch_atg_plan}).

\paragraph{Classifier-Free Guidance (CFG)~\citep{ho2022cfg}.}
For class $c$ and guidance scale $s=w+1\geq1$, let $q$, $p(\cdot\mid c)$, and $p(\cdot\mid\varnothing)$ denote the generated, conditional data, and unconditional data feature distributions, respectively.
The guided field is
\begin{equation}
\begin{aligned}
\widetilde V^{\rho,s}(\mathbf u_i)={}&
s\left[T^\rho_{q,p(\cdot\mid c)}(\mathbf u_i)-T^\rho_{q,q}(\mathbf u_i)\right]\\
&-(s-1)\left[T^\rho_{q,p(\cdot\mid\varnothing)}(\mathbf u_i)-T^\rho_{q,q}(\mathbf u_i)\right].
\end{aligned}
\label{eq:tdac_cfg_velocity}
\end{equation}

\section{Experiments}
\label{sec:experiments}

\paragraph{Experiment setup.} 
For $256\times256$ ImageNet~\citep{deng2009imagenet}, we use DiT-style~\citep{peebles2023scalable} generators and a pretrained latent-MAE encoder as in Drifting~\citep{deng2026generative}. Details are in Appendix~\ref{app:implementation}.

\subsection{Toy Experiments}
\label{sec:toy_experiments}

\paragraph{Training-time evolution of generated samples.}
\Cref{fig:tdac_evolution} shows 2D sample evolution with a small multilayer perceptron generator.
Generated samples move from a compact initial cloud toward the targets as the shared drift-field loss decreases.
The pinwheel arms become distinct, while the grid remains less sharply resolved.
The preconditioned TDAC variant illustrates sample evolution; Appendix~\ref{app:toy_evolution_settings} specifies its update and settings.

\begin{figure}[t]
\centering
\includegraphics[width=\columnwidth]{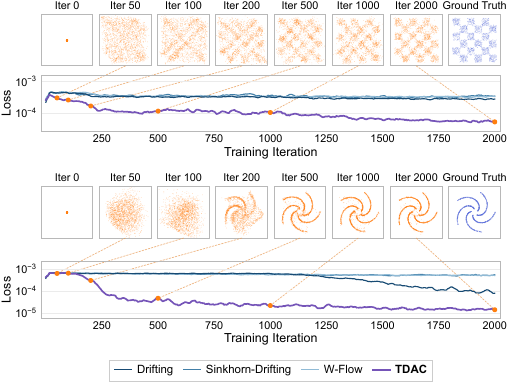}
\caption{\textbf{Evolution of samples.} Checkerboard (top), pinwheel (bottom): 2,048 generated (orange) and 2,048 target (blue) points.}
\label{fig:tdac_evolution}
\end{figure}

\paragraph{Action-aware classifier-free guidance.} 
We study action-aware guidance with a two-dimensional finite generator in \cref{fig:tdaction_cfg}.
The study separates the KL-CFG reference target from the action-feasible endpoint under the common generator geometry and budget.
Dashed contours show the KL-CFG targets. The reference uses true generator mobility, while TDAC uses a perturbed estimate under the same guidance and action budget. Drifting, Sinkhorn-Drifting, and W-Flow follow distribution-first dynamics.

\begin{figure}[H]
\centering
\includegraphics[width=0.96\columnwidth]{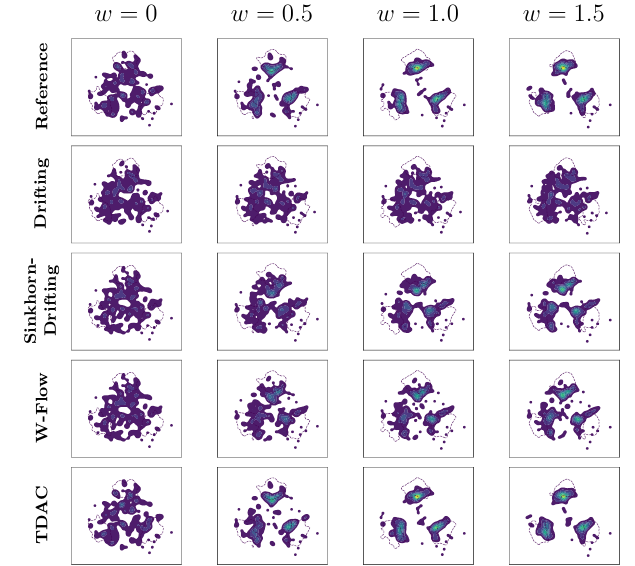}
\caption{\textbf{Action-aware guidance.} Dashed: KL-CFG targets. TDAC's mobility estimate has $2^\circ$ direction and $1\%$ scale errors.}
\label{fig:tdaction_cfg}
\end{figure}

In \cref{fig:tdaction_cfg_metrics}, TDAC matches the action-optimal reference better than W-Flow, particularly at stronger guidance.

\begin{figure}[!htbp]
\centering
\includegraphics[width=\columnwidth]{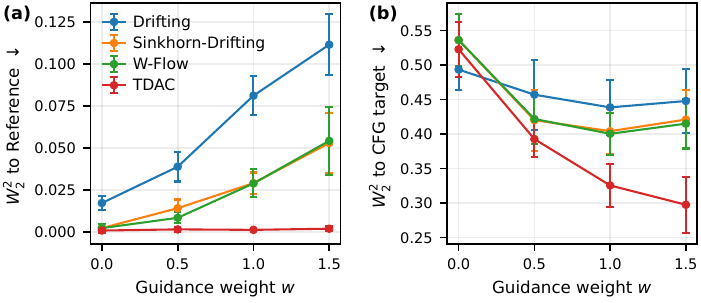}
\caption{\textbf{Guidance errors.} $W_2^2$ distances to \textbf{(a)} Reference and \textbf{(b)} \mbox{KL-CFG} target; mean and standard deviation over five seeds.}
\label{fig:tdaction_cfg_metrics}
\end{figure}


\subsection{ImageNet Experiments}
\label{sec:setting}

\paragraph{Experimental setup.} 
Section~\ref{sec:implementation} gives the components; \cref{tab:reference_configs,tab:tdac_planning_configuration} give training and planning settings.

\paragraph{Results.}
We compare TDAC with one-step generators and multi-step diffusion and flow models in \cref{tab:imagenet_results}.
TDAC achieves lower FID and higher Inception Score (IS) than W-Flow~\citep{han2026wflow} at B/2, L/2, and XL/2, and outperforms Drifting~\citep{deng2026generative} in both FID and IS at B/2 and L/2.
TDAC also exhibits favorable scaling with generator size: increasing the model from B/2 to L/2 and XL/2 further improves FID, while inference remains one step.

\begin{figure}[!htbp]
\centering
\includegraphics[width=\columnwidth]{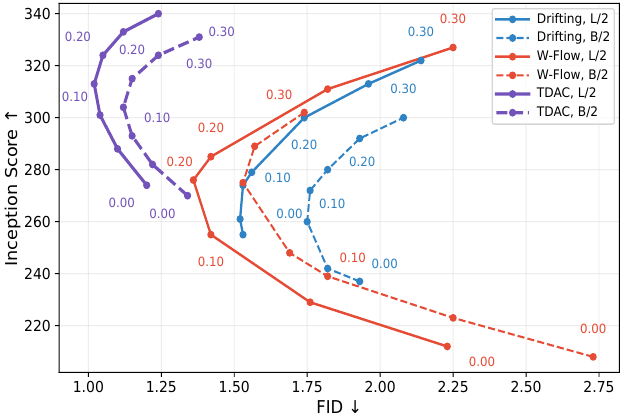}
\caption{\textbf{Classifier-free guidance.} FID and Inception Score (IS) curves as CFG scale varies for B/2 and L/2 generators.}
\label{fig:tdaction_imagenet_cfg}
\end{figure}

The smaller TDAC B/2 generator achieves lower FID than the L/2 generators of Drifting and W-Flow.
At XL/2 scale, TDAC achieves the lowest FID among the compared one-step methods, including iMeanFlow.
The comparison also includes multi-step generation: TDAC XL/2 achieves lower FID than LightningDiT and RAE+DiT$^{\mathrm{DH}}$ with one generator evaluation.
TDAC achieves low FID with a single generator forward pass at inference.

\begin{table}[!t]
\centering
\caption{\textbf{Class-conditional ImageNet 256$\times$256 results.} Baseline results follow the cited sources, with settings varying across methods. TDAC covers B/2, L/2, and XL/2. ``\#Params'' counts generator and decoder parameters when available.}
\label{tab:imagenet_results}
{\footnotesize
\setlength{\tabcolsep}{1.5pt}
\renewcommand{\arraystretch}{1.02}
\begin{tabular*}{\columnwidth}{@{\extracolsep{\fill}}>{\raggedright\arraybackslash}p{0.43\columnwidth} c c r r@{}}
\toprule
Method & \#Params & NFE & FID$\downarrow$ & IS$\uparrow$ \\
\midrule
\rowcolor[gray]{0.9} \multicolumn{5}{l}{\textit{Multi-step Diffusion/Flows}} \\
\headspace \tdafaded{ADM-G} \tdafaded{\citep{dhariwal2021diffusion}} & \tdafaded{ 554M} & \tdafaded{ 250$\times$2} & \tdafaded{4.59} & \tdafaded{186.7} \\
\headspace \tdafaded{DiT-XL/2} \tdafaded{\citep{peebles2023scalable}} & \tdafaded{ 675M+49M} & \tdafaded{ 250$\times$2} & \tdafaded{2.27} & \tdafaded{278.2} \\
\headspace \tdafaded{SiT-XL/2} \tdafaded{\citep{ma2024sit}} & \tdafaded{ 675M+49M} & \tdafaded{ 250$\times$2} & \tdafaded{2.06} & \tdafaded{270.3} \\
\headspace \tdafaded{SiT-XL/2+REPA} \tdafaded{\citep{yu2024representation}} & \tdafaded{ 675M+49M} & \tdafaded{ 250$\times$2} & \tdafaded{1.42} & \tdafaded{305.7} \\
\headspace \tdafaded{LightningDiT-XL/2} \tdafaded{\citep{yao2025reconstruction}} & \tdafaded{ 675M+70M} & \tdafaded{ 250$\times$2} & \tdafaded{1.35} & \tdafaded{295.3} \\
\headspace \tdafaded{RAE+DiT$^{\mathrm{DH}}$-XL/2} \tdafaded{\citep{zheng2025diffusion}} & \tdafaded{ 839M+415M} & \tdafaded{ 50$\times$2} & \tdafaded{1.13} & \tdafaded{262.6} \\
\midrule
\rowcolor[gray]{0.9} \multicolumn{5}{l}{\textit{Masking \& Autoregressive}} \\
\headspace \tdafaded{MaskGIT} \tdafaded{\citep{chang2022maskgit}} & \tdafaded{ 227M+31M} & \tdafaded{ 8} & \tdafaded{6.18} & \tdafaded{182.1} \\
\headspace \tdafaded{VAR-$d30$} \tdafaded{\citep{tian2024visual}} & \tdafaded{ 2B+61M} & \tdafaded{ 10$\times$2} & \tdafaded{1.92} & \tdafaded{323.1} \\
\headspace \tdafaded{MAR-H} \tdafaded{\citep{li2024autoregressive}} & \tdafaded{ 943M+41M} & \tdafaded{ 256$\times$2} & \tdafaded{1.55} & \tdafaded{303.7} \\
\midrule
\rowcolor[gray]{0.9} \multicolumn{5}{l}{\textit{GANs}} \\
\headspace \tdamid{BigGAN} \tdamid{\citep{brock2018large}} & \tdamid{ 112M} & \tdamid{ 1} & \tdamid{6.95} & \tdamid{152.8} \\
\headspace \tdamid{GigaGAN} \tdamid{\citep{kang2023scaling}} & \tdamid{ 569M} & \tdamid{ 1} & \tdamid{3.45} & \tdamid{225.5} \\
\headspace \tdamid{StyleGAN-XL} \tdamid{\citep{sauer2022stylegan}} & \tdamid{ 166M} & \tdamid{ 1} & \tdamid{2.30} & \tdamid{265.1} \\
\midrule
\rowcolor[gray]{0.9} \multicolumn{5}{l}{\textit{Pixel-space Models}} \\
\headspace \tdafaded{SiD2} \tdafaded{\citep{hoogeboom2025simpler}} & \tdafaded{ 397M} & \tdafaded{ 512} & \tdafaded{1.38} & \tdafaded{--} \\
\headspace \tdafaded{JiT-G/16} \tdafaded{\citep{li2026back}} & \tdafaded{ 2B} & \tdafaded{ 100$\times$2} & \tdafaded{1.82} & \tdafaded{292.6} \\
\headspace \tdafaded{PixelFlow-XL/4} \tdafaded{\citep{chen2025pixelflow}} & \tdafaded{ 677M} & \tdafaded{ 120$\times$2} & \tdafaded{1.98} & \tdafaded{282.1} \\
\headspace \tdafaded{PixelDiT-XL/16} \tdafaded{\citep{yu2026pixeldit}} & \tdafaded{ 797M} & \tdafaded{ 100$\times$2} & \tdafaded{1.61} & \tdafaded{292.7} \\
\headspace \tdafaded{DeCo-XL/16} \tdafaded{\citep{ma2026deco}} & \tdafaded{ 682M} & \tdafaded{ 250$\times$2} & \tdafaded{1.62} & \tdafaded{301.0} \\
\midrule
\rowcolor[gray]{0.9} \multicolumn{5}{l}{\textit{Single-step Diffusion/Flows}} \\
\headspace \tdamid{iCT-XL/2} \tdamid{\citep{song2023improved}} & \tdamid{ 675M+49M} & \tdamid{ 1} & \tdamid{34.24} & \tdamid{--} \\
\headspace \tdamid{Shortcut-XL/2} \tdamid{\citep{frans2024one}} & \tdamid{ 675M+49M} & \tdamid{ 1} & \tdamid{10.60} & \tdamid{--} \\
\headspace \tdamid{MeanFlow-XL/2} \tdamid{\citep{geng2025mean}} & \tdamid{ 676M+49M} & \tdamid{ 1} & \tdamid{3.43} & \tdamid{--} \\
\headspace \tdamid{TiM-XL/2} \tdamid{\citep{wang2025transition}} & \tdamid{ 664M+49M} & \tdamid{ 1} & \tdamid{3.26} & \tdamid{210.3} \\
\headspace \tdamid{$\alpha$-Flow-XL/2+} \tdamid{\citep{zhang2025alphaflow}} & \tdamid{ 676M+49M} & \tdamid{ 1} & \tdamid{2.58} & \tdamid{--} \\
\headspace \tdamid{AdvFlow-XL/2} \tdamid{\citep{lin2025adversarial}} & \tdamid{ 673M+49M} & \tdamid{ 1} & \tdamid{2.38} & \tdamid{284.2} \\
\headspace \tdamid{iMeanFlow-XL/2} \tdamid{\citep{geng2025improved}} & \tdamid{ 610M+49M} & \tdamid{ 1} & \tdamid{1.72} & \tdamid{282.0} \\
\midrule
\rowcolor[gray]{0.9} \multicolumn{5}{l}{\textit{Drifting Models}} \\
\headspace \tdadark{Drifting, B/2} \tdadark{\citep{deng2026generative}} & \tdadark{ 133M+49M} & \tdadark{ 1} & \tdadark{1.75} & \tdadark{263.2} \\
\headspace \tdadark{Drifting, L/2} \tdadark{\citep{deng2026generative}} & \tdadark{ 463M+49M} & \tdadark{ 1} & \tdadark{1.54} & \tdadark{258.9} \\
\midrule
\rowcolor[gray]{0.9} \multicolumn{5}{l}{\textit{Sinkhorn Drift Models}} \\
\headspace \tdadark{Sinkhorn drift, B/2} \tdadark{\citep{turan2026generativedrifting}} & \tdadark{ 133M+49M} & \tdadark{ 1} & \tdadark{8.81} & \tdadark{135.5} \\
\midrule
\rowcolor[gray]{0.9} \multicolumn{5}{l}{\textit{W-Flow Models}} \\
\headspace W-Flow, B/2 \citep{han2026wflow} &  133M+49M &  1 & 1.52 & 271.8 \\
\headspace W-Flow, L/2 \citep{han2026wflow} &  463M+49M &  1 & 1.35 & 272.5 \\
\headspace W-Flow, XL/2 \citep{han2026wflow} &  679M+49M &  1 & 1.29 & 265.4 \\
\midrule
\rowcolor[gray]{0.9} \multicolumn{5}{l}{\textit{TDAC Models}} \\
\headspace \textbf{TDAC, B/2} & { 133M+49M} & { 1} & \textbf{1.33} & \textbf{283.7} \\
\headspace \textbf{TDAC, L/2} & { 463M+49M} & { 1} & \textbf{1.10} & \textbf{286.1} \\
\headspace \textbf{TDAC, XL/2} & { 679M+49M} & { 1} & \textbf{1.05} & \textbf{278.7} \\
\bottomrule
\end{tabular*}
}
\end{table}

\begin{table*}[!t]
\centering
\def\TDACShadeRow{\noalign{\begingroup\color[gray]{0.9}%
\hrule width\linewidth height\dimexpr\ht\csname @arstrutbox\endcsname+\dp\csname @arstrutbox\endcsname\relax
\kern-\dimexpr\ht\csname @arstrutbox\endcsname+\dp\csname @arstrutbox\endcsname\relax
\endgroup}}
\caption{\textbf{\boldmath Ablation study on ImageNet $256\times256$ generation.} FID uses 50K generated images and 1-NFE. (a) Divergence choice. (b) Classifier-free guidance and entropy regularization. (c) Generator geometry. Further details are given in Appendix~\ref{app:ablation_details}.}
\label{tab:tdac_main_ablation}
\begin{minipage}[t]{0.325\textwidth}
\vspace{0pt}
\centering
\scriptsize
\setlength{\tabcolsep}{1.8pt}
\renewcommand{\arraystretch}{1.12}
\begin{tabular*}{\linewidth}{@{\extracolsep{\fill}}lccc@{}}
\toprule
\textbf{Divergence} & \textbf{Drifting}~\citep{deng2026generative} & \textbf{W-Flow}~\citep{han2026wflow} & \textbf{TDAC} \\
\midrule
Drift field & 8.46 & -- & -- \\
MMD & -- & 10.40 & 8.92 \\
KL & -- & 10.17 & 7.15 \\
Sinkhorn & -- & 7.29 & 6.84 \\
\TDACShadeRow
Sliced Wasserstein & -- & -- & \textbf{6.72} \\
\bottomrule
\end{tabular*}
\end{minipage}%
\hfill
\begin{minipage}[t]{0.325\textwidth}
\vspace{0pt}
\centering
\scriptsize
\setlength{\tabcolsep}{2pt}
\renewcommand{\arraystretch}{1.12}
\begin{tabular*}{\linewidth}{@{\extracolsep{\fill}}ccccc@{}}
\toprule
$\bm\rho$ & \textbf{CFG} & \textbf{Drifting}~\citep{deng2026generative} & \textbf{W-Flow}~\citep{han2026wflow} & \textbf{TDAC} \\
\midrule
0.050 & dist. & 8.46 & 7.29 & 6.95 \\
\TDACShadeRow
0.050 & velo. & 8.12 & 7.08 & \textbf{6.84} \\
0.025 & velo. & 8.28 & 7.11 & 6.88 \\
0.075 & velo. & 8.39 & 7.26 & 6.92 \\
0.100 & velo. & 8.55 & 7.45 & 6.98 \\
\bottomrule
\end{tabular*}
\end{minipage}%
\hfill
\begin{minipage}[t]{0.325\textwidth}
\vspace{0pt}
\centering
\scriptsize
\setlength{\tabcolsep}{1.5pt}
\renewcommand{\arraystretch}{1.12}
\begin{tabular*}{\linewidth}{@{\extracolsep{\fill}}lcc@{}}
\toprule
\textbf{Mobility structure $\mathbf K$} & \textbf{Rel.~Error} & \textbf{FID} \\
\midrule
Identity Metric ($\mathbf I$) & 8.64\% & 7.38 \\
Isotropic Floor ($\kappa\mathbf I$) & 4.82\% & 7.18 \\
Block-Diagonal (Per-sample) & 2.15\% & 7.04 \\
Diagonal (Feature-wise) & 3.42\% & 7.10 \\
\TDACShadeRow
\textbf{Full-batch mobility} & \textbf{0.21\%} & \textbf{6.84} \\
\bottomrule
\end{tabular*}
\end{minipage}
\end{table*}

\paragraph{Classifier-free guidance.}
\Cref{fig:tdaction_imagenet_cfg} compares FID--IS trade-offs across guidance scales for B/2 and L/2 generators.
The Drifting and W-Flow curves are digitized from Figure~4 of W-Flow~\citep{han2026wflow}.
Across these scales, stronger guidance raises IS, while FID initially falls and then rises at higher scales.
The TDAC curves achieve lower FID than the baselines at comparable IS, with the L/2 curve lying left of B/2 over their overlapping IS range.
Appendix~\ref{app:guidance_evaluation} specifies paired evaluation at fixed guidance scales (\cref{tab:guidance_evaluation}).

\paragraph{Convergence speed.}
The progress constraint retains predicted distributional progress while the action objective accounts for the cost of fitting selected movements.
TDAC reaches lower FID than W-Flow's final models in fewer epochs at both scales, while using a smaller batch and more updates per epoch (\cref{tab:fid_training_epochs}).
At the reported batch sizes, TDAC has lower step time than both baselines for B/2 and L/2. Hardware, updates, and total costs for training comparisons are in \cref{tab:training_cost_records,tab:training_cost_breakdown} (Appendix~\ref{app:training_cost_details}).

\begin{table}[!htbp]
  \centering
  \small
  \caption{\textbf{ImageNet FID at different training epochs.} For each model, FID is followed by training time in seconds per step.}
  \label{tab:fid_training_epochs}
  \setlength{\tabcolsep}{3pt}
  \begin{tabular*}{\columnwidth}{@{\extracolsep{\fill}}lcccccc@{}}
    \toprule
    Method & Epoch & Batch $\downarrow$ & B/2 $\downarrow$ & s/step $\downarrow$ & L/2 $\downarrow$ & s/step $\downarrow$ \\
    \midrule
    Drifting~\citep{deng2026generative}
      & 1280 & 8192 & 1.75 & 21.5 & 1.54 & 36.8 \\
    \arrayrulecolor{black!25}\midrule\arrayrulecolor{black}
    \multirow{3}{*}{W-Flow~\citep{han2026wflow}}
      & 384  & 8192 & 1.75 & 28.4 & 1.55 & 48.6 \\
      & 640  & 8192 & 1.63 & 28.4 & 1.46 & 48.6 \\
      & 1280 & 8192 & 1.52 & 28.4 & 1.35 & 48.6 \\
    \arrayrulecolor{black!25}\midrule\arrayrulecolor{black}
    \multirow{3}{*}{TDAC (ours)}
      & 384  & 2048 & 1.50 & 5.8 & 1.32 & 10.2 \\
      & 640  & 2048 & 1.41 & 5.8 & 1.21 & 10.2 \\
      & 1280 & \textbf{2048} & \textbf{1.33} & \textbf{5.8} & \textbf{1.10} & \textbf{10.2} \\
    \bottomrule
  \end{tabular*}
\end{table}

\paragraph{Generated samples.} 
Using the same 56K-step EMA checkpoint, seeds, and classes, \cref{fig:tdaction_fixed_checkpoint_samples} compares $s=1.00$ and $1.25$ in the passive pairwise implementation.
Stronger guidance improves class-specific details in several paired examples while retaining similar image content.

\begin{figure}[!htbp]
  \centering
  \begin{minipage}[t]{0.48\columnwidth}
    \centering
    \includegraphics[width=\linewidth]{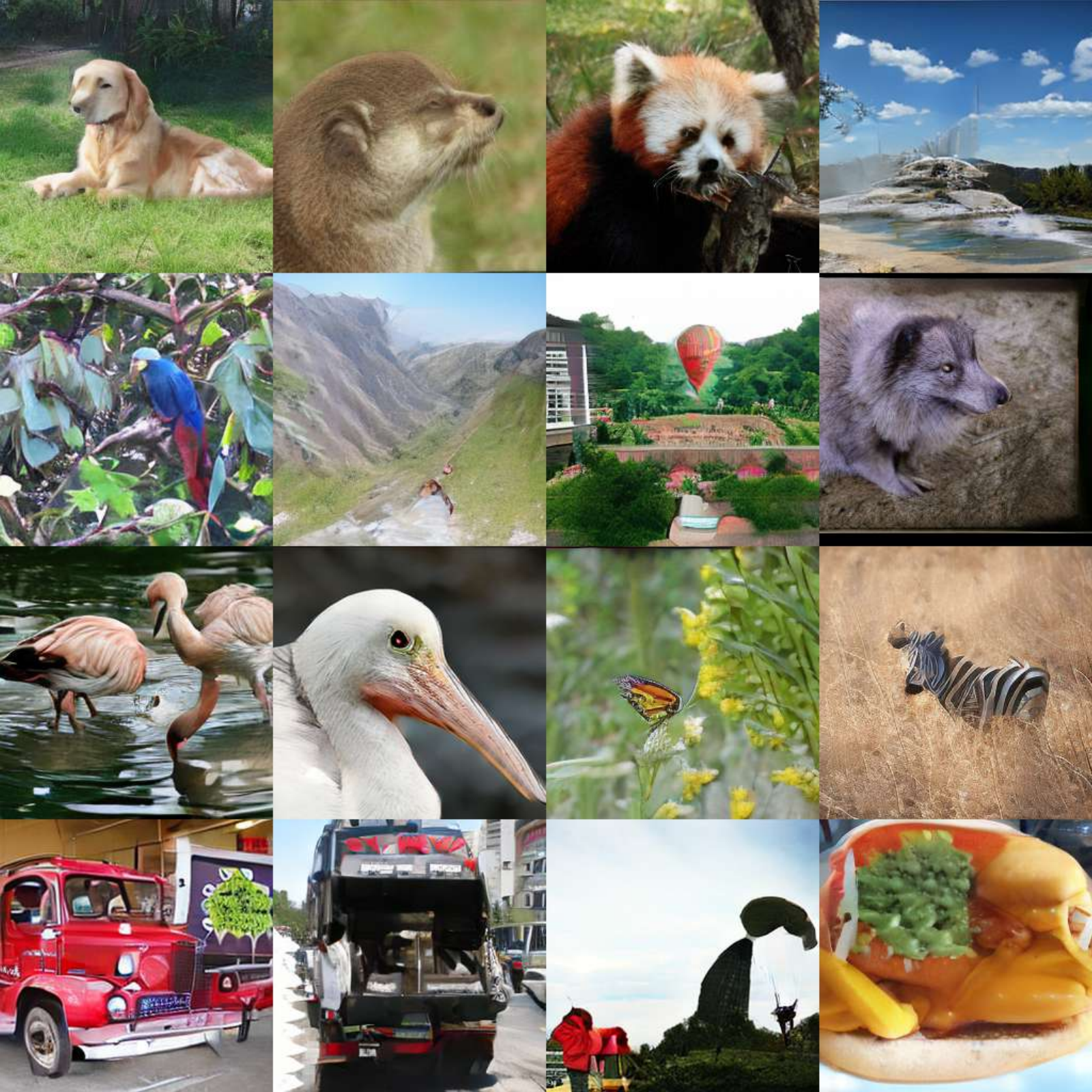}
    Intrinsic, $s=1.00$
  \end{minipage}\hfill
  \begin{minipage}[t]{0.48\columnwidth}
    \centering
    \includegraphics[width=\linewidth]{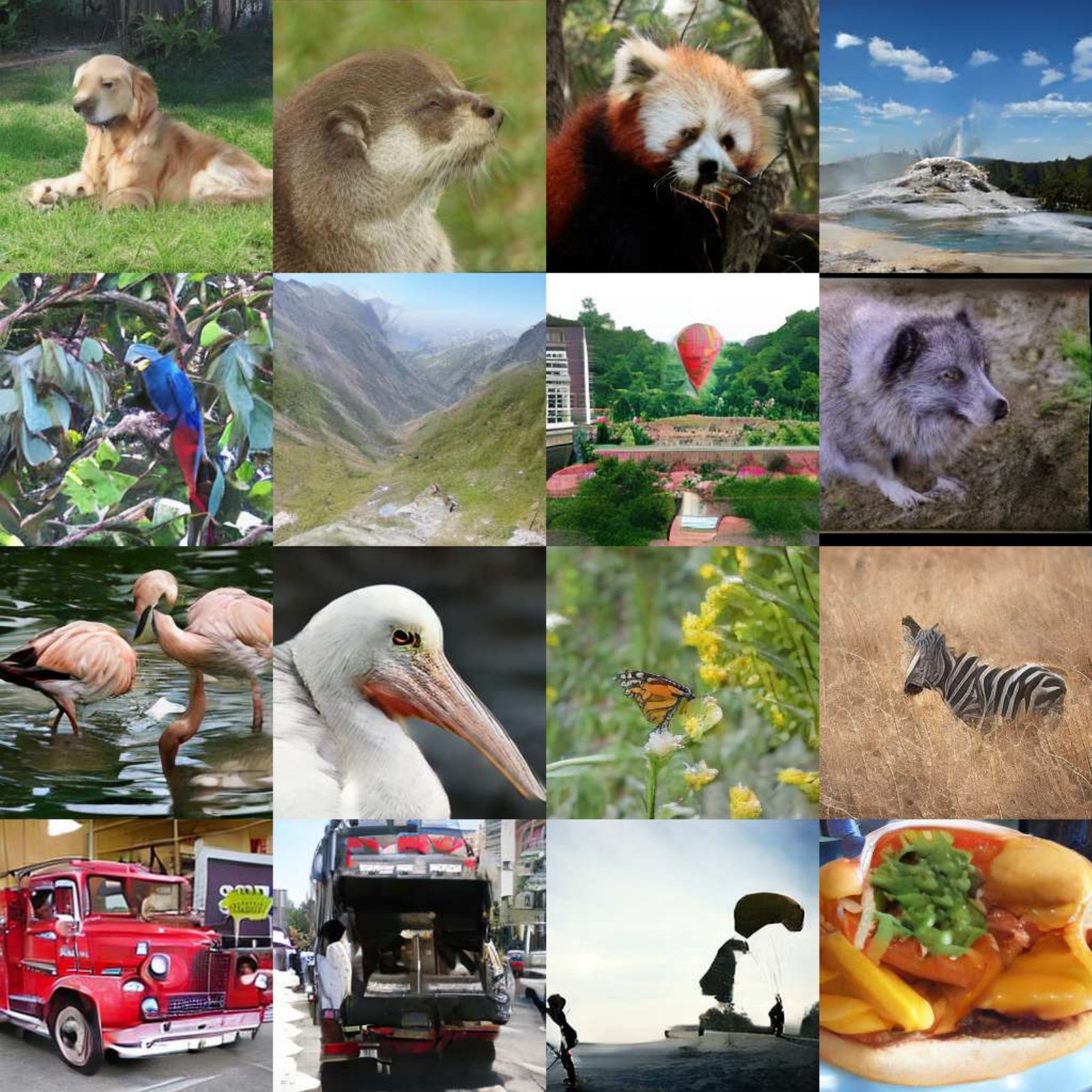}
    Selected guidance, $s=1.25$
  \end{minipage}
  \caption{\textbf{Effect of guidance on ImageNet samples.} Uncurated 56K-step EMA checkpoint samples: $s=1.00$ (left), $1.25$ (right).}
  \label{fig:tdaction_fixed_checkpoint_samples}
\end{figure}


\subsection{Ablation Study}
\label{sec:tdaction_ablation}

\paragraph{Ablation setting.}
The ImageNet $256\times256$ ablation setup follows Drifting~\citep{deng2026generative} and W-Flow~\citep{han2026wflow}, using a DiT-B/2 generator, SD-VAE latents, and a frozen pretrained latent-MAE encoder.
FID is evaluated on 50K generated images with 1-NFE.
Appendix~\ref{app:ablation_details} details settings, controls, and errors (\cref{tab:matched_plan_controls,tab:mobility_action_errors}); \cref{tab:ablation_action_rank} compares rank and $\varepsilon$.

\paragraph{Divergence choice.}
For the three divergences shared with W-Flow, \cref{tab:tdac_main_ablation}(a) shows the lowest FID with Sinkhorn for both methods.
TDAC has lower FID under each divergence, suggesting that action-based target selection can improve training across different distributional objectives.

\paragraph{Classifier-free guidance.}
In \cref{tab:tdac_main_ablation}(b), dist. and velo. denote distribution guidance and velocity guidance.
With Sinkhorn divergence and $\rho=0.05$, the results favor velocity guidance over distribution guidance for both W-Flow and TDAC.
The shared improvement supports combining velocity guidance with action-based target selection.

\paragraph{Entropy regularization.}
TDAC shows little FID variation across the listed entropy coefficients under velocity guidance~(\cref{tab:tdac_main_ablation}(b)).
We use $\rho=0.05$, which gives the lowest FID, in the remaining comparisons.

\paragraph{Generator geometry.}
Full-batch mobility gives the lowest FID, and per-sample blocks outperform isotropic and feature-wise diagonal approximations.
The further improvement from per-sample blocks to full-batch mobility supports retaining cross-sample responses when selecting targets by joint fitting cost~(\cref{tab:tdac_main_ablation}(c)).

\section{Conclusion}
\label{sec:conclusion}

We introduced TDAC for one-step generation, selecting sample movements by joint fitting cost. Closed-form ATG measures parameter-update cost and feature error; a constraint retains predicted progress. Random parameter directions enable action evaluation; fixed-target regression trains the generator. Execution-gap analysis relates planned to actual first-order progress. On ImageNet $256\times256$, TDAC achieves an FID below 1.1 without distillation using one generator evaluation. TDAC evaluates action per training update, leaving planning across multiple updates and alternative feature encoders for study in future work.

\section*{LLM Usage Statement}
Large language models assisted with writing, literature searches, and mathematical derivations and proofs.
The authors retain full responsibility for the scientific content, results, references, and final manuscript.
\FloatBarrier
{
    \small
    \bibliographystyle{ieeenat_fullname}
    \bibliography{main}
}

\clearpage
\setcounter{page}{1}
\onecolumn
\raggedbottom
\suppressfloats[t]
\begin{center}
\Large
\textbf{\thetitle}\\
Supplementary Material\\

\end{center}
\appendix

\section{Additional Implementation and Evaluation}
\subsection{Generator Training Setup}
\label{app:implementation}

\Cref{tab:reference_configs} gives the ImageNet-256 settings for Drifting, W-Flow, and the fixed-control ActionFuse runs. For ActionFuse, steps count updates after resuming, epochs specify the reference budget, and $s$ is the guidance scale in Eq.~\ref{eq:tdac_cfg_velocity}. The DiT maps noise, a class label, and a guidance scale to SD-VAE latents, which we decode for visualization and evaluation. A frozen latent MAE provides planning features $\psi$ and regression features $\{\phi_\ell\}_{\ell=1}^{L}$ and passes gradients to the generator. At each support refresh, we draw latent inputs $\mathbf Z$ and data samples $\mathcal Y=\{\mathbf y_j\}_{j=1}^{M}$ with $\mathbf y_j\sim p$, then cache the data features $\mathbf Y$. Algorithm~\ref{alg:tdaction_training} recomputes the generated features $\mathbf U$ at every optimizer step (Eq.~\ref{eq:batch_feature_map}). The plan sets $\pi_{ij}=0$ across classes and requires equal generated and data mass per class. Shared parameters can still couple responses across classes.

\begin{table}[H]
\centering
\caption{\textbf{ImageNet $256\times256$ experiments.} Training configurations for Drifting, W-Flow, and TDAC (planning settings in \cref{tab:tdac_planning_configuration}).}
\label{tab:reference_configs}
{\scriptsize
\setlength{\tabcolsep}{2pt}
\renewcommand{\arraystretch}{1.02}
\resizebox{\textwidth}{!}{%
\begin{tabular}{@{}l|cc|ccc|ccc@{}}
\toprule
& \multicolumn{2}{c|}{\textbf{Drifting}} & \multicolumn{3}{c|}{\textbf{W-Flow}}
& \multicolumn{3}{c}{\textbf{TDAC (ActionFuse)}} \\
\textbf{Setting} & \textbf{B/2} & \textbf{L/2}
& \textbf{B/2} & \textbf{L/2} & \textbf{XL/2}
& \textbf{B/2} & \textbf{L/2} & \textbf{XL/2} \\
\midrule
\rowcolor[gray]{0.9}\multicolumn{9}{l}{\textit{\textbf{Generator Architecture}}} \\
arch & DiT-B/2 & DiT-L/2 & DiT-B/2 & DiT-L/2 & DiT-XL/2 & DiT-B/2 & DiT-L/2 & DiT-XL/2 \\
input size & \multicolumn{8}{c}{$32\times32\times4$} \\
patch size & \multicolumn{8}{c}{$2\times2$} \\
hidden dim & 768 & 1024 & 768 & 1024 & 1152 & 768 & 1024 & 1152 \\
depth & 12 & 24 & 12 & 24 & 28 & 12 & 24 & 28 \\
register tokens & \multicolumn{8}{c}{16} \\
style embedding tokens & \multicolumn{5}{c}{32} & \multicolumn{3}{c}{32} \\
\midrule
\rowcolor[gray]{0.9}\multicolumn{9}{l}{\textit{\textbf{Feature Encoder}}} \\
architecture & \multicolumn{5}{c}{ResNet} & \multicolumn{3}{c}{ResNet} \\
self-supervised learning (SSL) & \multicolumn{8}{c}{latent-MAE} \\
ResNet: input size & \multicolumn{5}{c}{$32\times32\times4$} & \multicolumn{3}{c}{$32\times32\times4$} \\
ResNet: conv$_1$ stride & \multicolumn{5}{c}{1} & \multicolumn{3}{c}{1} \\
ResNet: base width & \multicolumn{8}{c}{640} \\
ResNet: block type & \multicolumn{5}{c}{bottleneck} & \multicolumn{3}{c}{bottleneck} \\
ResNet: blocks / stage & \multicolumn{5}{c}{$[3,4,6,3]$} & \multicolumn{3}{c}{$[3,4,6,3]$} \\
ResNet: size / stage & \multicolumn{5}{c}{$[32^2,16^2,8^2,4^2]$} & \multicolumn{3}{c}{$[32^2,16^2,8^2,4^2]$} \\
MAE: masking ratio & \multicolumn{5}{c}{50\%} & \multicolumn{3}{c}{50\%} \\
MAE: pre-train epochs & \multicolumn{5}{c}{1280} & \multicolumn{3}{c}{1280} \\
classification finetune & \multicolumn{5}{c}{3k steps} & \multicolumn{3}{c}{3k steps} \\
\midrule
\rowcolor[gray]{0.9}\multicolumn{9}{l}{\textit{\textbf{Generator Optimizer}}} \\
optimizer & \multicolumn{8}{c}{AdamW ($\beta_1=0.9$, $\beta_2=0.95$)} \\
learning rate & $4\!\times\!10^{-4}$ & $4\!\times\!10^{-4}$ & $4\!\times\!10^{-4}$ & $4\!\times\!10^{-4}$ & $3\!\times\!10^{-4}$ & $1.5\!\times\!10^{-4}$ & $3\!\times\!10^{-4}$ & $1.5\!\times\!10^{-4}$ \\
weight decay & 0.0 & 0.01 & 0.0 & 0.01 & 0.01 & 0.0 & 0.01 & 0.01 \\
warmup steps & \multicolumn{5}{c}{10k} & \multicolumn{3}{c}{10k} \\
gradient clip & \multicolumn{8}{c}{2.0} \\
training steps & 200k & 200k & 200k & 200k & 200k & 20k & 8k & 20k \\
training epochs (reference configuration) & \multicolumn{5}{c}{1280} & \multicolumn{3}{c}{1280} \\
EMA decay & \multicolumn{2}{c|}{$\{0.999,0.9995,0.9998,0.9999\}$} & \multicolumn{3}{c|}{0.999} & \multicolumn{3}{c}{0.999} \\
generator initialization & \multicolumn{5}{c}{From scratch} & \multicolumn{3}{c}{Resume from checkpoint} \\
start $\to$ target update & $0\to200\mathrm{k}$ & $0\to200\mathrm{k}$ & $0\to200\mathrm{k}$ & $0\to200\mathrm{k}$ & $0\to200\mathrm{k}$ & $105\mathrm{k}\to125\mathrm{k}$ & $117\mathrm{k}\to125\mathrm{k}$ & $105\mathrm{k}\to125\mathrm{k}$ \\
\midrule
\rowcolor[gray]{0.9}\multicolumn{9}{l}{\textit{\textbf{Training Loss Computation}}} \\
class labels $N_{\rm c}$ & 128 & 128 & 128 & 128 & 128 & 32 & 32 & 32 \\
positive samples $N_{\rm pos}$ & 128 & 64 & 64 & 64 & 128 & 64 & 64 & 64 \\
generated samples $N_{\rm neg}$ & \multicolumn{8}{c}{64} \\
effective batch ($N_{\rm c}\!\times\!N_{\rm neg}$) & \multicolumn{5}{c}{8192} & \multicolumn{3}{c}{2048} \\
temperatures $\tau$ & \multicolumn{2}{c|}{$\{0.02,0.05,0.2\}$} & \multicolumn{3}{c|}{--} & \multicolumn{3}{c}{--} \\
iteration $L$ & -- & -- & 10 & 1 & 1 & 10 & 10 & 10 \\
regularization $\rho$ & -- & -- & 0.05 & 0.05 & 0.01 & 0.05 & 0.05 & 0.01 \\
step size $\eta$ & -- & -- & 1.0 & 1.0 & 1.0 & 1.0 & 1.0 & 1.0 \\
transport estimator & \multicolumn{5}{c}{--} & \multicolumn{3}{c}{Debiased OT} \\
transport radius & \multicolumn{5}{c}{--} & \multicolumn{3}{c}{0.05} \\
action control / cycle mass fraction & \multicolumn{5}{c}{--} & \multicolumn{3}{c}{0.25 / 0.25} \\
\midrule
\rowcolor[gray]{0.9}\multicolumn{9}{l}{\textit{\textbf{CFG Configuration}}} \\
train: CFG range & \multicolumn{2}{c|}{$\alpha\in[1,4]$} & \multicolumn{3}{c|}{$w\in[0,3]$} & \multicolumn{3}{c}{$s\in[1,4]$} \\
train: CFG point mass & $0$ & $50\%:\alpha=1$ & $0$ & $50\%:w=0$ & $20\%:w=0$ & $0$ & $50\%:s=1$ & $50\%:s=1$ \\
train: remaining CFG density & $\propto\alpha^{-5}$ & $\propto\alpha^{-3}$ & $\propto(w+1)^{-5}$ & $\propto(w+1)^{-3}$ & $\propto(w+1)^{-4}$ & $\propto s^{-5}$ & $\propto s^{-3}$ & $\propto s^{-3}$ \\
train: uncond samples $N_{\rm uncond}$ & \multicolumn{8}{c}{32} \\
inference: CFG search & \multicolumn{2}{c|}{$\alpha\in[1,3.5]$} & \multicolumn{3}{c|}{$w\in[0,2.5]$} & \multicolumn{3}{c}{Model-specific grid} \\
\bottomrule
\end{tabular}}
}
\end{table}

\subsection{ActionFuse Normalization}
\label{app:native_normalization}
Appendix~\ref{app:regression_details} gives the TDAC solver and fixed-target regression details.
The fixed-control ActionFuse variant first scales the features used to compute the transport maps. At feature level $\ell$, let $\mathbf F\in\mathbb R^{b\times n\times d_\ell}$ contain generated features. The axes index independent local transport problems, particles, and feature coordinates, respectively. Let $\mathbf H^{(a)}\in\mathbb R^{b\times m_a\times d_\ell}$, $a\in\{\mathrm{self},\mathrm{neg},\mathrm{pos}\}$, be the scale-reference supports, with nonnegative weights $\omega^{(a)}_{cj}$. The feature scale is
\begin{equation}
 s_\ell=\max\!\left\{\sqrt{\frac{\max\{0,\sum_{a,c,i,j}\omega^{(a)}_{cj}\|\mathbf F_{ci:}-\mathbf H^{(a)}_{cj:}\|_2^2\}}{d_\ell n[\sum_{a,c,j}\omega^{(a)}_{cj}+10^{-8}b\sum_a m_a]}},\;10^{-3}\right\},
\end{equation}
where the numerator is evaluated using centered moments to reduce cancellation, its nonnegative clamp protects the square root, and the additive term protects a zero total support weight. All supports, including the unconditional support, are divided by this same scalar before computing transport maps.

Let $\mathbf D^{(0)}$ be the uncorrected three-map displacement from Eq.~\ref{eq:tdac_cfg_velocity} in these scaled coordinates, and let $\mathbf D^{(u)}$ denote the displacement after the fixed-control four-cycle correction. The reference root mean square (RMS) is
\begin{equation}
 r_\ell=\sqrt{\max\!\left\{\frac{1}{bnd_\ell}\sum_{c,i,j}(D^{(0)}_{cij})^2,10^{-8}\right\}},
\end{equation}
where $r_\ell\geq10^{-4}$ by construction. Reduction covers the three axes of the local batch at each feature level. The normalized target displacement is $\mathbf D^{(u)}/r_\ell$. A zero uncorrected displacement gives $r_\ell=10^{-4}$; the normalized output is zero only if the corrected displacement is also zero.

The actual target and local loss are
\begin{equation}
 \widetilde{\mathbf F}=\operatorname{sg}[\mathbf F/s_\ell+\mathbf D^{(u)}/r_\ell],
 \qquad \mathcal L_c=\frac1{nd_\ell}\sum_{i,j}(F_{cij}/s_\ell-\widetilde F_{cij})^2,
\end{equation}
where the target, scale, transport plans, and RMS remain fixed during regression. Gradients propagate through the prediction $\mathbf F/s_\ell$. The target step size is one. Barycentric maps divide by each plan's row sum with a lower bound of $10^{-12}$.

\subsection{Computational Comparison}

ActionFuse runs use bfloat16, seed 42, and four workers. Each worker processes eight single-class batches, with gradient accumulation set to 16, 16, and 8 for B/2, L/2, and XL/2, respectively. The learning rate follows cosine decay from the parent checkpoint's update. The frozen latent-MAE encoder has width 640. Conditional negatives are resampled separately from the generated gradient batch. Every 1K updates, we evaluate 50K samples per CFG scale: $\{1.14,1.17,1.19,1.21,1.23\}$ for B/2 and $\{1.16,1.17,1.18,1.19\}$ for L/2 and XL/2.

\label{app:compute_illustration}

\Cref{fig:compute-illustration} compares how Drifting, W-Flow, and TDAC construct targets for one-step generators. Here $\mathbf x_i=f_{\bm\theta}(\mathbf z_i)$ and $\mathbf y_j\sim p$ are generated and data samples. Panels (a,b) are adapted from W-Flow~\citep{han2026wflow}.
In \cref{fig:compute-illustration}(a), Drifting normalizes the pairwise affinities $\mathbf K_{\mathrm{aff}}$ along rows and columns to form attraction and repulsion weights~\citep{deng2026generative}. The combined weights specify sample displacements and need not satisfy both transport marginals. \Cref{fig:compute-illustration}(b) shows W-Flow's Sinkhorn balancing, which enforces the prescribed marginals~\citep{han2026wflow}. The data and self couplings give the barycentric terms of the velocity field for descent of the distributional energy.
In \cref{fig:compute-illustration}(c), TDAC selects a coupling by minimizing the regularized action objective in Eq.~\ref{eq:progress_constrained_batch_atg}. Random parameter directions estimate mobility, and $\bm\pi_{\mathrm{ref}}$ sets the reference progress. On the fixed supports in Eqs.~\ref{eq:batch_feature_map}--\ref{eq:batch_displacement}, the action in Eq.~\ref{eq:batch_atg_value} weights $\mathbf d_{\bm\pi}$ by the regularized inverse of $\mathbf K=\mathbf J\mathbf R^{-1}\mathbf J^\top$. The mobility $\mathbf K\in\mathbb R^{Nd\times Nd}$ contains cross-sample blocks from the shared generator parameters. At equal magnitude, displacements cost more along directions with weaker generator response.

For the coupling derivative, hold $\mathbf U$, $\mathbf Y$, $\mathbf J$, and $\mathbf R$ fixed. Let $\operatorname{mat}$ undo row stacking in Eq.~\ref{eq:batch_feature_map}. The chain rule gives
\begin{equation*}
    \mathbf W_{\bm\pi}=\operatorname{mat}\!\left((\mathbf K^\varepsilon)^{-1}\mathbf d_{\bm\pi}\right)
        \in\mathbb R^{N\times d},\qquad
    \nabla_{\bm\pi} \mathcal A_\varepsilon(\bm\pi)=\mathbf D^{-1}\mathbf W_{\bm\pi}\mathbf Y^\top.
\end{equation*}
The solver includes entropy and, on aligned supports, persistence. For positive reference progress, the coupling must satisfy the transport marginals and $\mathcal P_k(\bm\pi)\geq\gamma\mathcal P_k(\bm\pi_{\mathrm{ref}})$.

The selected coupling defines fixed targets at each feature level in Eq.~\ref{eq:shared_targets}. The progress constraint applies to the planned displacement. Equation~\ref{eq:action_interaction_field} gives the ideal local response with the coupling, reference features, and mobility fixed. The projection residual (Eq.~\ref{eq:projection_gap}) measures the part of the requested displacement outside the shared tangent space. The execution gap (Eq.~\ref{eq:app_action_execution_gap}) measures the actual AdamW update's excess quadratic cost.

\begin{figure}[H]
    \centering
    \includegraphics[width=0.95\textwidth,trim=2.0bp 3.0bp 2.0bp 4.0bp,clip]{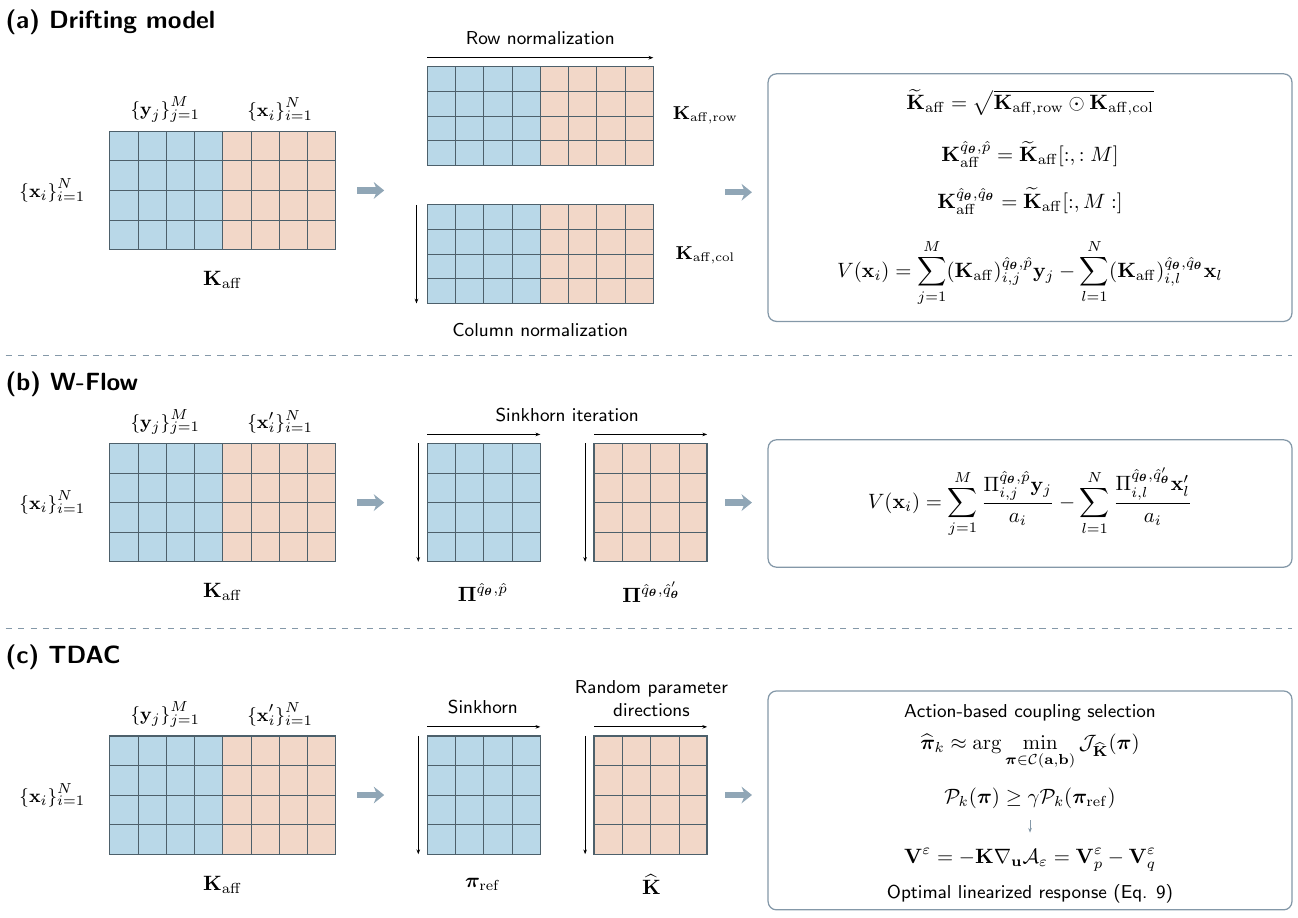}
    \caption{\textbf{From kernel normalization to generator-aware action selection.}
    (a) Drifting uses kernel normalization.
    (b) W-Flow uses Sinkhorn balancing.
    (c) TDAC selects couplings by local generator action under a progress constraint.}
    \label{fig:compute-illustration}
\end{figure}

In Algorithms~\ref{alg:tdaction_training} and~\ref{alg:batch_atg_plan}, workers gather ordered features and probe responses for one coupling, preserving cross-sample mobility across workers and microbatches. Each worker constructs fixed targets from its coupling rows and contributes to the reduced regression gradient. Feature extraction, Jacobian--vector products, barycenters, and backpropagation can run in chunks.

The local execution gap compares the quadratic cost of the actual optimizer increment $\delta\bm\theta$ with the minimum action for the same plan:
\[
  \Delta_{\bm\pi}
  =
  \frac12\delta\bm\theta^\top \mathbf R\delta\bm\theta
  +\frac{1}{2\varepsilon}\left\|\mathbf J\delta\bm\theta-\mathbf d_{\bm\pi}\right\|_2^2
  -\mathcal A_\varepsilon(\bm\pi).
\]
We measure action on the estimation and validation directions, planned and post-update progress, projection and marginal residuals, target-realization cosine, plan churn, feasibility, step time, and allocated and reserved memory.

We select guidance on a separate sample set and evaluate the frozen EMA checkpoint on 50K class-balanced samples. The run record includes the checkpoint, configuration, seeds, real-data statistics, sample count, and evaluator version.

Matched ablations fix the backbone, latent representation, conditioning, optimizer schedule, EMA rule, batch, data order, and evaluation statistics. Action-to-Go target selection (Eq.~\ref{eq:progress_constrained_batch_atg}) adds mobility estimation, plan optimization, and worker communication. Inference evaluates the EMA generator once. We record support order, estimation and validation directions, normalization, and whole-step resource use.

\subsection{Computing the Transport Plan}
\label{app:regression_details}

At each update $k$, we recompute the generated features $\mathbf U_k$ and planning gradient $\mathbf g_k$. We reuse $\widehat{\mathbf K}$ between refreshes while the planning samples and their order stay fixed. A support change requires a new estimate.

Algorithm~\ref{alg:batch_atg_plan} uses the objective
\[
    \mathcal J(\bm\pi)=\mathcal A_\varepsilon(\bm\pi)
    +\rho\operatorname{KL}(\bm\pi\|\mathbf a\mathbf b^\top)
    +\lambda\operatorname{KL}(\bm\pi\|\widehat{\bm\pi}_{k-1}).
\]
with $\mathbf K$ replaced by $\widehat{\mathbf K}$ in the action term, giving $\mathcal J_{\widehat{\mathbf K}}$. For a nonnegative reference matrix $\mathbf r\in\mathbb R_+^{N\times M}$, $\operatorname{KL}(\bm\pi\|\mathbf r)=\sum_{ij}[\pi_{ij}\log(\pi_{ij}/r_{ij})-\pi_{ij}+r_{ij}]$. The log term is zero when $\pi_{ij}=0$, and the divergence is infinite if $\pi_{ij}>0=r_{ij}$. We omit zero-weight penalties, even for infinite KL.

Let $\mathcal Q=\{\bm\pi\in\mathcal C(\mathbf a,\mathbf b):\mathcal P_k(\bm\pi)\geq\gamma\mathcal P_k(\bm\pi_{\mathrm{ref}})\}$ be the feasible set in Eq.~\ref{eq:progress_constrained_batch_atg}. The projection of candidate weights $\mathbf W$ is $\textsc{Project}(\mathbf W)=\arg\min_{\bm\pi\in\mathcal Q}\operatorname{KL}(\bm\pi\|\mathbf W)$. We approximate it with Sinkhorn scaling for the marginals and a scalar dual update for progress. The operators $\odot$ and $\exp$ act elementwise.

At each refresh, we draw $n_{\mathrm{probe}}$ independent Gaussian directions $\bm\xi_s\sim\mathcal N(\mathbf 0,\mathbf I_P)$, where $P$ is the number of generator parameters. With the fixed metric $\mathbf R\succ0$ in Eq.~\ref{eq:batch_mobility}, the responses and Monte Carlo estimate are
\[
  \mathbf r_s=\mathbf J\mathbf R^{-1/2}\bm\xi_s,
  \qquad
  \widehat{\mathbf K}^{\mathrm{MC}}=\frac1{n_{\mathrm{probe}}}\sum_{s=1}^{n_{\mathrm{probe}}}\mathbf r_s\mathbf r_s^\top .
\]
We use these directions to estimate mobility for plan selection and reserve independent directions for validation. For $N$ planning samples with feature dimension $d$, we store the estimate as
\[
  \widehat{\mathbf K}=\mathbf B\mathbf B^\top+\sigma^2\mathbf I,
  \qquad \mathbf B\in\mathbb R^{Nd\times r}.
\]
where $r$ is the retained rank. The floor $\sigma^2$ regularizes the mobility estimate, and $\varepsilon$ controls feature fitting. Define $\widehat{\mathbf K}^{\varepsilon}\coloneqq\widehat{\mathbf K}+\varepsilon\mathbf I$ and $\tau_\varepsilon=\sigma^2+\varepsilon$. The Woodbury identity gives
\[
  (\widehat{\mathbf K}^{\varepsilon})^{-1}\mathbf d_{\bm\pi}
  =
  \tau_\varepsilon^{-1}\mathbf d_{\bm\pi}
  -
  \tau_\varepsilon^{-2}\mathbf B
  \bigl(\mathbf I+\tau_\varepsilon^{-1}\mathbf B^\top \mathbf B\bigr)^{-1}\mathbf B^\top \mathbf d_{\bm\pi} .
\]
This requires $\mathcal O(Ndr)$ storage instead of $\mathcal O((Nd)^2)$. The probe cost per update depends on the refresh interval.

Algorithm~\ref{alg:batch_atg_plan} uses uniform weights $\mathbf a=N^{-1}\mathbf1_N$ and $\mathbf b=M^{-1}\mathbf1_M$ in Eq.~\ref{eq:coupling_polytope}. It first computes $\bm\pi_{\mathrm{ref}}$ by entropy-regularized transport with squared Euclidean cost. For positive reference progress, it holds the generator fixed and takes $n_{\mathrm{plan}}$ projected KL mirror steps of size $\eta_{\mathrm{plan}}>0$. The persistence term applies only when the previous plan uses the same ordered latent and data samples. The first training update omits it. Let $\bm\pi_k^t$ be the plan after $t$ inner steps and $\bm\pi_k^\star$ an exact minimizer of Eq.~\ref{eq:progress_constrained_batch_atg}. The solver returns $\widehat{\bm\pi}_k=\bm\pi_k^{n_{\mathrm{plan}}}$ if feasible, and $\bm\pi_{\mathrm{ref}}$ if reference progress is nonpositive or the final iterate is infeasible. \Cref{tab:tdac_planning_configuration} gives the solver settings.

The selected plan provides coupling weights for fixed regression targets. Equations~\ref{eq:shared_plan_features} and~\ref{eq:shared_targets} give the features, weighted means $\mathbf t_i^\ell$, targets, and loss. The regression feature dimension $d_\ell$ can differ from $d$, and the normalization scale $r_\ell>0$ stays fixed during regression. Gradients pass through the frozen feature maps to the generator with the plan, mobility estimate, and targets held fixed. Appendix~\ref{app:native_normalization} describes the ActionFuse normalization. The progress constraint uses the unnormalized displacement $\mathbf d_{\bm\pi}$, and the execution bound uses the actual parameter increment.

\section{Additional Experiments}
\label{app:additional_experiments}
\subsection{Additional Results for Toy Experiments}
\label{app:additional_toy_results}

\paragraph{Sample-evolution settings.}
\label{app:toy_evolution_settings}
The sample-evolution study in \cref{fig:tdac_evolution} uses a dense $6\times6$ checkerboard and a five-arm pinwheel. The generator maps 32-dimensional Gaussian noise to two-dimensional samples. It has three hidden layers of width 256 with sigmoid linear unit activations. Each snapshot shows 2,048 generated points in orange or target points in blue. We plot the mean squared norm of the Drifting field at temperature $0.045$ on a log scale with a centered five-point moving average.

The curve labeled TDAC uses a Drifting update with output-covariance preconditioning in place of the ATG solver in Algorithm~\ref{alg:batch_atg_plan}. The pinwheel also uses a tangential correction. Per plotted iteration, the variant takes four updates on the checkerboard and three on the pinwheel, with learning rates of $0.007$ and $0.006$, respectively, and cosine decay. Each baseline takes one update per plotted iteration at a constant learning rate of $0.001$. The plots index sample evolution by outer iteration.

In \cref{fig:tdaction_cfg_metrics}(a), TDAC has lower error to the action-optimal reference than W-Flow across the guidance range, with larger differences at stronger guidance. At the strongest guidance setting, TDAC also has lower error to the KL-CFG target than W-Flow in panel (b).

\label{sec:mobility_validation}

For mobility estimation, we test the identifiable active-probe model $\mathbf r=\mathbf B_{\mathrm{resp}}\bm\xi+\bm\zeta$, with $\mathbf B_{\mathrm{resp}}\in\mathbb R^{16\times6}$, $\bm\xi\sim\mathcal N(\mathbf 0,\mathbf I_6)$, Gaussian observation noise $\bm\zeta\in\mathbb R^{16}$, and ridge $0.15$.
For each probe count, we use 20 independent seeds, estimate $\mathbf B_{\mathrm{resp}}\mathbf B_{\mathrm{resp}}^\top$ after subtracting the known noise covariance, and measure mobility recovery and true-action plan regret.

\begin{table}[H]
\centering
\caption{\textbf{Mobility estimation with active probes.} Each probe count is evaluated over 20 seeds.}
\label{tab:mobility_identification}
\begin{tabular}{ccccc}
\toprule
Probes & Frob.$\downarrow$ & Angle$\downarrow$ & Inverse$\downarrow$ & Regret$\downarrow$\\
\midrule
16  & 0.484 & $3.96^\circ$ & 0.284 & $3.55\!\times\!10^{-3}$\\
32  & 0.314 & $2.23^\circ$ & 0.171 & $1.62\!\times\!10^{-3}$\\
64  & 0.207 & $1.52^\circ$ & 0.107 & $8.06\!\times\!10^{-4}$\\
128 & 0.142 & $1.07^\circ$ & 0.078 & $3.55\!\times\!10^{-4}$\\
256 & 0.103 & $0.75^\circ$ & 0.053 & $1.71\!\times\!10^{-4}$\\
512 & \textbf{0.072} & $\boldsymbol{0.53^\circ}$ & \textbf{0.033} & $\boldsymbol{7.77\!\times\!10^{-5}}$\\
\bottomrule
\end{tabular}
\end{table}

\Cref{tab:mobility_identification} shows improved mobility estimation and lower action regret as the number of random directions increases.
The two-point example in Proposition~\ref{prop:anisotropic} shows that Euclidean coupling can incur more action than minimum-action coupling.

\begin{table}[H]
\centering
\caption{\textbf{Two-dimensional local-action study for TDAC.}
Fr\'echet values are two-dimensional Gaussian-moment proxies.}
\label{tab:local_action}
\begin{tabular}{lccc}
\toprule
Metric & W-Flow & TDAC & Relative change \\
\midrule
Two-dimensional Fr\'echet proxy $\downarrow$ & $0.0469\pm0.0233$ & $\boldsymbol{0.0417\pm0.0189}$ & $-10.98\%$ \\
Sliced Wasserstein $\downarrow$ & $0.3406\pm0.0148$ & $\boldsymbol{0.3351\pm0.0126}$ & $-1.62\%$ \\
Optimizer-time action $\downarrow$ & $70.1898\pm6.4149$ & $\boldsymbol{62.3337\pm5.7703}$ & $-11.19\%$ \\
Mode coverage $\uparrow$ & $\boldsymbol{8/8}$ & $\boldsymbol{8/8}$ & -- \\
\bottomrule
\end{tabular}
\end{table}

We extend the toy studies in Section~\ref{sec:toy_experiments} to a Laplacian Gibbs kernel,
$k(x,y)=\exp(-\lVert x-y\rVert/\tau)$.
The choice of 8-Gaussians and Checkerboard, the temperature sweep $\tau\in\{0.01,0.05,0.1\}$, and the paired presentation of final samples and empirical $W_2^2$ trajectories follow Appendix~J and Figure~7 of Sinkhorn-Drifting~\citep{he2026sinkhorndrifting}.
That study compares one-sided, two-sided, and Sinkhorn normalization and finds improved coverage and convergence from Sinkhorn balancing, particularly at small temperatures.
We compare Drifting, Sinkhorn-Drifting, W-Flow, and TDAC under a common finite-generator response schedule and output-action budget, testing how these distributional constructions interact with generator realizability.

The quantitative evaluation uses 32 particles, 80 updates, and seeds 42--46. We evaluate $W_2^2$ every five updates and use 20 Sinkhorn iterations per solve. The scatter panels use 96 particles with seed 42. \Cref{fig:tdaction_laplacian} shows that TDAC has the lowest final distributional error on both datasets at every tested kernel temperature.

\label{sec:trajectory_stress_test}
\Cref{fig:tdaction_action_to_go}(a) visualizes two two-dimensional transports that isolate the finite-generator realization gap.
Each setting fixes the source and target distributions, learner-response schedule, optimizer-time horizon, and cumulative output-action budget across Drifting, Sinkhorn-Drifting, W-Flow, and TDAC.
Cross-Swap compares a shorter endpoint match with a longer correspondence that requires less generator action. Ring-Cycle compares eight modes with gradually varying responses.

In Cross-Swap, the distribution-only baselines move samples toward the target modes, but their correspondences remain poorly aligned with the generator's latent-dependent responses.
The predictive TDAC variant selects persistent latent-to-target correspondences using forecast mobility, yielding final outputs concentrated near the target modes in both settings. At the same output-action budget, TDAC also reaches lower final distributional error in both settings~(\cref{fig:tdaction_action_to_go}(b)).

\begin{figure}[H]
    \centering
    \includegraphics[width=\textwidth,trim=1.2bp 2.5bp 1.2bp 1.2bp,clip]{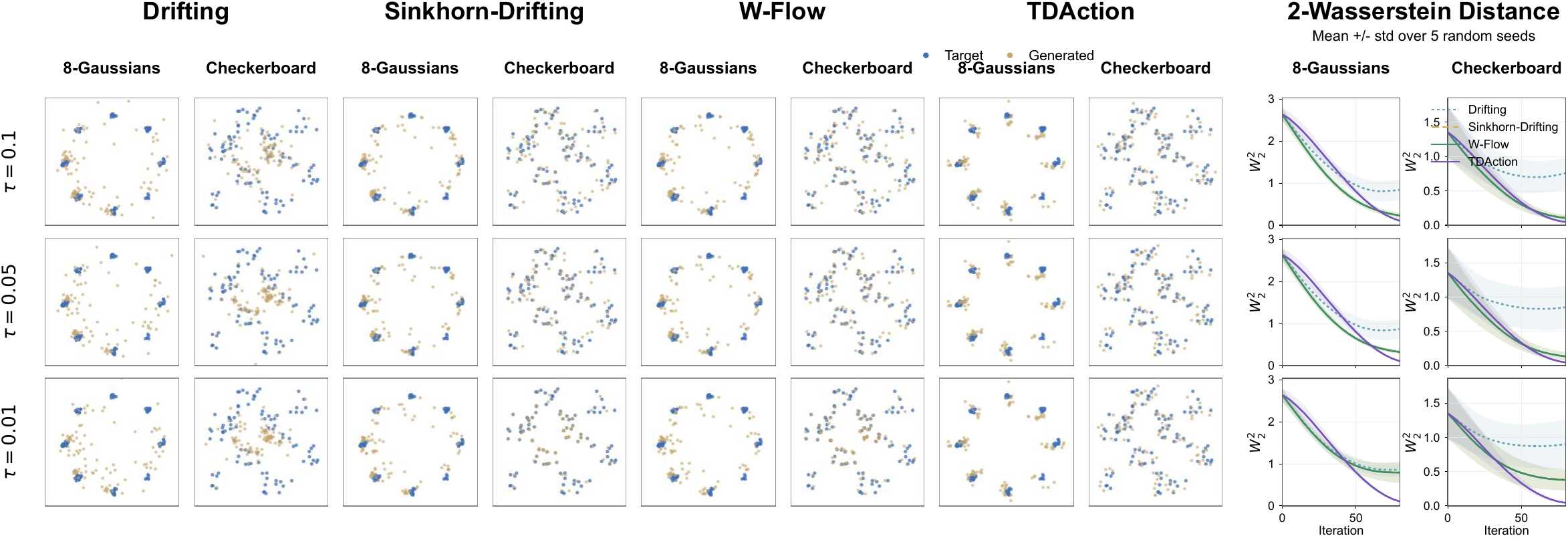}
    \caption{\textbf{Laplacian-kernel toy transports.}
    We compare four methods on 8-Gaussians and Checkerboard at $\tau\in\{0.01,0.05,0.1\}$.
    \emph{Left:} final generated samples (tan) and targets (blue).
    \emph{Right:} empirical $W_2^2$ trajectories. Lines and bands show the mean and standard deviation across five seeds.}
    \label{fig:tdaction_laplacian}
\end{figure}

\begin{figure}[H]
    \centering
    \includegraphics[width=\textwidth,trim=0.8bp 1.6bp 9.8bp 5.4bp,clip]{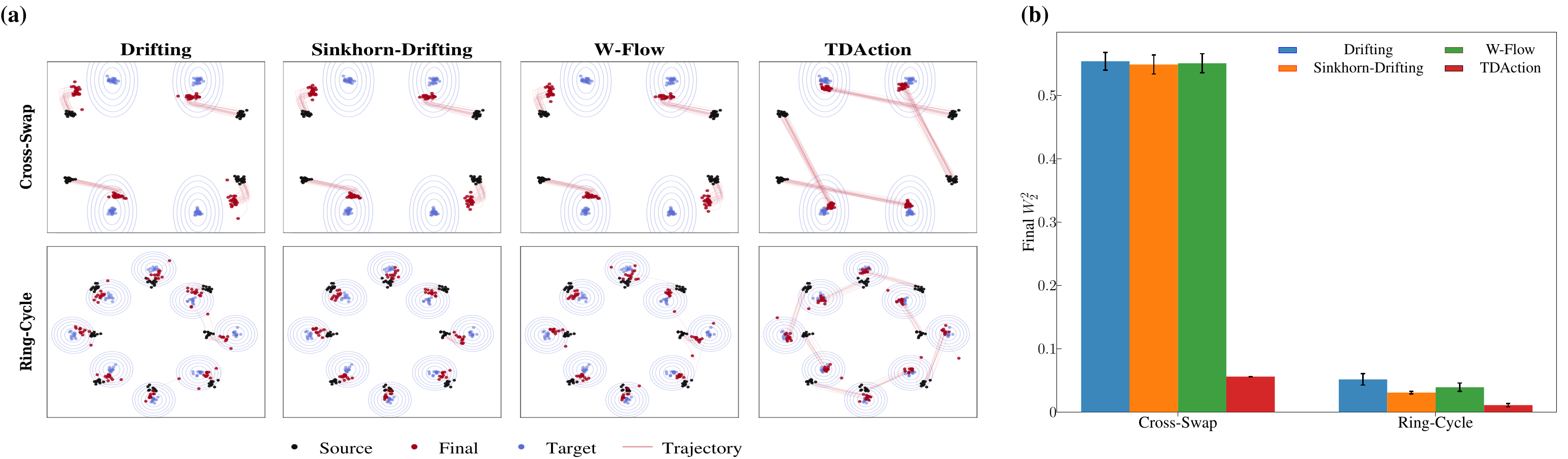}
    \caption{\textbf{Finite-generator trajectories in two toy transports.}
    (a) Cross-Swap and Ring-Cycle under a shared held-out response schedule. TDAC estimates Action-to-Go without the realized generator response.
    (b) Final empirical $W_2^2$ across five seeds.}
    \label{fig:tdaction_action_to_go}
\end{figure}

\subsubsection{Local-Action and Horizon Studies}
\label{app:local_action}

We compare W-Flow with the direct-update TDAC variant over 400 steps in \cref{tab:local_action}. The methods share the residual generator, initialization, AdamW optimizer, minibatch stream, and endpoint Sinkhorn oracle. We report the mean and standard deviation across five paired seeds.

TDAC uses less action in all five paired runs ($p=1.26\times10^{-5}$), and both methods cover all eight modes; endpoint Fr\'echet proxy ($-10.98\%$) and sliced Wasserstein ($-1.62\%$) have overlapping five-seed standard deviations.

\label{app:horizon_screen}

In a separate 180-step, three-seed horizon screen (\cref{tab:horizon_screen}), the intermediate and extended TDAC horizons reduce both the Fr\'echet proxy and action relative to W-Flow, whereas the single-step local variant increases both.

\begin{table}[H]
\centering
\caption{\textbf{Ablation on planning horizon.}
We evaluate each setting over three seeds on a two-dimensional task; action ratio averages per-seed ratios relative to W-Flow.}
\label{tab:horizon_screen}
\begin{tabular}{lccc}
\toprule
Method & Fr\'echet proxy $\downarrow$ & Action $\downarrow$ & Action ratio $\downarrow$ \\
\midrule
W-Flow & $0.1247$ & $17.6794$ & $1.0000$ \\
TDAC (local) & $0.1342$ & $27.8792$ & $1.7986$ \\
TDAC (intermediate) & $0.1172$ & $16.9900$ & $0.8666$ \\
TDAC (extended) & $\boldsymbol{0.1168}$ & $\boldsymbol{16.9607}$ & $\boldsymbol{0.8660}$ \\
\bottomrule
\end{tabular}
\end{table}

\subsection{Additional Results for ImageNet Experiments}
\label{app:additional_imagenet_results}

\subsubsection{Additional Ablation Studies}
\label{app:ablation_details}

The ImageNet-256 ablation protocol follows the short-budget W-Flow study~\citep{han2026wflow}: a DiT-B/2 generator, SD-VAE latents, a frozen latent-MAE encoder of base width 256 pretrained for 192 epochs, 30K optimizer updates, and a global generated batch of 4096.
The optimizer is AdamW with learning rate $2\times10^{-4}$, 5K warmup updates, and EMA decay $0.999$; the reference entropic transport uses $\rho=0.05$ and ten Sinkhorn iterations.
The 30K-update budget is described as 100 epochs in W-Flow and is kept distinct from W-Flow's 200K-update main results.
Paired runs use seeds 42, 43, and 44, with matching model initialization, data order, latent stream, conditioning, and evaluation statistics.
We evaluate each EMA checkpoint on 50K samples per guidance scale $s\in\{1.00,1.50,2.00,2.50,3.00,3.50\}$, where $s=w+1$ and $s=1.00$ denotes no guidance. FID and IS are paired at the same selected scale.
Additional matched controls and evaluation tables are given in Appendix~\ref{app:additional_imagenet_results}.

\paragraph{Matched Plan-Selection Comparisons.}
\label{app:matched_plan_controls}
\Cref{tab:matched_plan_controls} compares the reference plan with TDAC and tests the progress constraint, persistence, and the retained fraction $\gamma$. The reference-plan control returns $\bm\pi_{\mathrm{ref}}$ with the same generator, initialization, features, regression, conditioning, and training budget. Paired controls also match support reuse and refresh intervals. Removing the progress constraint deletes the half-space in Eq.~\ref{eq:progress_constrained_batch_atg}. Setting $\gamma=0$ retains a nonnegative-progress constraint. Removing persistence sets its weight to zero.

\begin{table}[H]
\centering\small
\caption{\textbf{Matched plan-selection ablations.} B/2 30K-update (50K-sample) setting; single-GPU A100 ($N=32$).}
\label{tab:matched_plan_controls}
\begin{tabular}{@{}lccccc@{}}
\toprule
Plan selection & FID $\downarrow$ & IS $\uparrow$ & Action ratio $\downarrow$ & Actual progress $\uparrow$ & s/step $\downarrow$ \\
\midrule
Reference plan only & 7.29 & 180.2 & 1.000 & 0.2275 & \textbf{0.198} \\
TDAC without progress constraint & 7.12 & 182.4 & \textbf{0.952} & \textbf{0.2331} & 0.282 \\
TDAC without persistence & 6.98 & 184.8 & \textbf{0.952} & 0.2329 & 0.283 \\
TDAC & \textbf{6.84} & \textbf{188.4} & 0.953 & 0.2328 & 0.284 \\
\midrule
\multicolumn{6}{l}{\textit{Retained progress fraction}} \\
TDAC, $\gamma=0.70$ & 6.91 & 186.1 & 0.953 & 0.2328 & 0.284 \\
TDAC, $\gamma=0.85$ (default) & \textbf{6.84} & \textbf{188.4} & 0.953 & 0.2328 & 0.284 \\
TDAC, $\gamma=0.95$ & 6.86 & 187.9 & 0.953 & 0.2328 & 0.284 \\
\bottomrule
\end{tabular}
\end{table}
TDAC lowers 30K validation action by $4.7\%$ while improving progress, FID, and IS; its planning overhead at $N=32$ amortizes below $4\%$ at full batch size (\cref{tab:training_cost_breakdown}). Dropping persistence or the progress safeguard relaxes the single-step problem but worsens full-run FID and IS, while $\gamma\in\{0.70,0.85,0.95\}$ yields close FIDs as the unconstrained plan retains $96.1\%$ of planned progress. Action ratio uses 16 probes (\cref{tab:execution_progress_diagnostics}).

\paragraph{Relative error.}
\Cref{tab:mobility_action_errors} compares errors in mobility and in the action of a fixed plan. For a small model with an explicit Jacobian, relative mobility error is $\|\widehat{\mathbf K}-\mathbf K\|_F/\|\mathbf K\|_F$. At a fixed plan with positive exact action, relative action error is $|\mathcal A_{\varepsilon,\widehat{\mathbf K}}(\bm\pi)-\mathcal A_{\varepsilon,\mathbf K}(\bm\pi)|/\mathcal A_{\varepsilon,\mathbf K}(\bm\pi)$. We report absolute error separately for zero denominators.

\begin{table}[H]
\centering\small
\caption{\textbf{Mobility and action estimation error.} Both errors use the same frozen local model and evaluation plans.}
\label{tab:mobility_action_errors}
\begin{tabular}{@{}lcc@{}}
\toprule
Mobility approximation & Relative mobility error & Relative action error \\
\midrule
Identity & 97.64\% & 8.64\% \\
Isotropic & 91.34\% & 4.82\% \\
Per-sample blocks & 89.40\% & 2.15\% \\
Feature-wise diagonal & 90.34\% & 3.42\% \\
Full-batch estimate & \textbf{63.97\%} & \textbf{0.21\%} \\
\bottomrule
\end{tabular}
\end{table}
The full-batch estimate attains the lowest action error ($0.21\%$) in \cref{tab:mobility_action_errors} at $63.97\%$ Frobenius mobility error.

\paragraph{Comparison settings.}
In \cref{tab:tdac_main_ablation}(a), MMD denotes squared MMD, Sliced Wasserstein denotes $SW_2^2$, and dashes mark unavailable results. Drifting uses a drift field with normalized kernels. In panel (b), we fix $\rho=0.05$ to compare guidance methods and use velocity guidance to compare entropy weights. In panel (c), identity mobility is the isotropic case with $\kappa=1$. \Cref{tab:tdac_planning_configuration} gives the planner settings, and \cref{tab:tdaction_ablation} summarizes the earlier diagnostic comparisons.

\begin{table}[!htbp]
\centering
\caption{\textbf{ImageNet $256\times256$ ablations.} 30 A100 steps (time, action efficiency), 50K full-run samples (FID). Bold per divergence in (c).}
\label{tab:tdaction_ablation}
\setlength{\tabcolsep}{4.5pt}
\begin{tabular}{@{}llccc@{}}
\toprule
\textbf{Sub-study} & \textbf{Variant / Configuration} & \textbf{Step (ms)$\downarrow$} & \textbf{Action Eff.$\uparrow$} & \textbf{FID-50K$\downarrow$} \\
\midrule
\multicolumn{5}{@{}l}{\textit{(a) Divergence and Action Selection}} \\
& Drifting Model~\citep{deng2026generative} & 138.6 & +0.255 & 8.46 \\
& MMD Gradient Flow & 153.1 & +0.147 & 10.40 \\
& KL Gradient Flow & 132.5 & \textbf{+0.306} & 10.17 \\
& Sinkhorn W-Flow~\citep{han2026wflow} & 136.2 & +0.296 & 7.29 \\
& \textbf{TDAC Action-to-Go} & \textbf{126.4} & +0.304 & \textbf{6.84} \\
\midrule
\multicolumn{5}{@{}l}{\textit{(b) Guidance and Direction Preservation}} \\
& Distribution Guidance ($s=1.50$) & 153.0 & +0.337 & 7.29 \\
& Velocity Guidance ($s=1.50$) & 174.7 & +0.363 & 7.08 \\
& \textbf{Tangent-Preserving Action Guidance} ($s=1.50, \lambda_{\mathrm{guide}}=0.92$) & 144.3 & \textbf{+0.402} & \textbf{6.72} \\
& Tangent Guidance ($s=1.25, \lambda_{\mathrm{guide}}=0.92$) & \textbf{137.7} & +0.352 & 6.81 \\
& Tangent Guidance ($s=2.00, \lambda_{\mathrm{guide}}=0.92$) & 140.9 & +0.240 & 6.89 \\
\midrule
\multicolumn{5}{@{}l}{\textit{(c) Guidance across Divergences}} \\
& KL Divergence + Distribution Guidance & \textbf{174.4} & +0.304 & 10.17 \\
& KL Divergence + Velocity Guidance~\citep{han2026wflow} & 177.0 & +0.384 & 7.68 \\
& \textbf{KL Divergence + Tangent Action Guidance} & 179.3 & \textbf{+0.458} & 7.15 \\
& Sinkhorn Divergence + Distribution Guidance & 153.0 & +0.337 & 7.29 \\
& Sinkhorn Divergence + Velocity Guidance~\citep{han2026wflow} & 174.7 & +0.363 & 7.08 \\
& \textbf{Sinkhorn Divergence + Tangent Action Guidance} & \textbf{144.3} & +0.402 & \textbf{6.72} \\
\midrule
\multicolumn{5}{@{}l}{\textit{(d) Self-Matching Bias and Transport Perturbation}} \\
& Single-Batch Unmasked & 165.7 & +0.271 & 17.57 \\
& Diagonal Mask Heuristic~\citep{deng2026generative} & \textbf{134.1} & +0.270 & 7.45 \\
& Two-Batch Resampling~\citep{han2026wflow} & 215.0 & +0.231 & 7.08 \\
& \textbf{4-Cycle Perturbation $\bm\Delta$ + ActionFuse} & 156.3 & \textbf{+0.365} & \textbf{6.84} \\
\midrule
\multicolumn{5}{@{}l}{\textit{(e) Action-Control Comparison (DiT-XL/2, Paired at 100k)}} \\
& Control Disabled (\textit{off}: $u_{\mathrm{ctrl}}=0.00$) & 252.7 & +6.895 & 1.88 \\
& \textbf{Control Enabled (\textit{on}: $u_{\mathrm{ctrl}}=0.25$)} & \textbf{226.4} & \textbf{+7.051} & \textbf{1.86} \\
\bottomrule
\end{tabular}
\end{table}

In \cref{tab:tdaction_ablation}, $u_{\mathrm{ctrl}}$ is total control strength (split equally across three components) and $\lambda_{\mathrm{guide}}$ is the tangent-projection coefficient in~(b)--(c) (cf.\ main ablations in \cref{tab:tdac_main_ablation}). Timing and action efficiency ($\Delta D/\mathcal A$) use a unified single-GPU A100 30-step benchmark with fixed latents (panel~(e) uses XL/2 at 100K, whose discrepancy scale differs from B/2); ActionFuse is the training implementation.
For a nonnegative coupling, the four-cycle perturbation on rows $i,j$ and distinct columns $r,s$ is
\[
\bm\Delta=\omega(\mathbf e_i-\mathbf e_j)(\mathbf e_s-\mathbf e_r)^\top,\qquad
\omega=\zeta\min(\pi_{ir},\pi_{js}),\quad 0\leq\zeta\leq1,
\]
where the coordinate vectors $\mathbf e_i,\mathbf e_j,\mathbf e_r,\mathbf e_s$ have the corresponding row or column dimension, and the scalar entries $\pi_{ir},\pi_{js}$ of $\bm\pi$ are not bolded. The perturbation has zero row and column sums and preserves nonnegativity of $\bm\pi+\bm\Delta$.

The shared Euclidean plan minimizes the pairwise entropic transport cost on the same canonical supports and supplies targets through the same regression update as TDAC.
The shared Euclidean plan changes the full coupling-selection objective.
Isotropic ATG instead retains the barycentric action objective and the same progress, entropy, and persistence terms. Soft Euclidean transport includes a cost for the conditional endpoint variance.
For the fixed-state isotropic comparison, $\kappa=\operatorname{tr}(\widehat{\mathbf K})/(Nd)$ matches the average mobility scale, with the same $\varepsilon$ and regularizer weights.
The per-sample control removes the cross-sample blocks $\mathbf J_i\mathbf R^{-1}\mathbf J_j^\top$ from the same mobility estimate, retaining diagonal blocks and the isotropic regularization term without an additional truncation or rescaling.
An additional spectrum control uses an orthogonal $\mathbf Q$ within each conditioning class. Conjugation preserves the full spectrum and isotropic regularization, including any cross-class responses in the original estimate.
The randomization is fixed on each canonical support and checked to change the mobility rather than merely relabel eigenpairs.

The no-progress variant deletes the half-space in Eq.~\ref{eq:progress_constrained_batch_atg}, leaving the other terms unchanged.
Both branches retain the same Euclidean fallback when the reference plan has nonpositive planned progress; fixed-state comparisons of the constraint use states with positive reference-plan progress and record fallback frequency separately.
At $\gamma=0$, progress must be $\geq 0$.

We compare three mobility estimates with the same ATG planner and regression update: the isotropic term alone, passive optimizer responses, and random parameter directions. The rank-$4$ sweep tests $n_{\mathrm{probe}}=4,8,16$, refreshes mobility every optimizer update, and takes $n_{\mathrm{probe}}=8$ as its reference. The probe-budget sweep in \cref{tab:ablation_action_rank} uses rank $8$ and refreshes mobility every four updates. For the passive control, we specify the response history on the same latent inputs, along with centering, rank, isotropic regularization, and the refresh rule. Full-training runs start from paired initializations. Short diagnostic runs restore the same model, EMA, optimizer, data order, and latent cache.

\Cref{tab:matched_plan_controls} compares the selected plan's action with the reference plan's action on the same support. Both actions are computed from Eq.~\ref{eq:batch_atg_value} with the same $\varepsilon$ and a mobility estimate from 16 independent validation directions. The ratio uses the positive reference action as its denominator, with near-zero values reported separately. \Cref{tab:execution_progress_diagnostics} shows planned first-order progress, first-order progress from the actual feature change, and the measured change in minibatch energy.

Each ablation varies one setting while keeping the validation probe count fixed. Whole-update time covers probing, plan optimization, communication, backpropagation, and the optimizer step. We report the median and 90th percentile of this time, with whole-run GPU-hours and peak memory reported separately. Total costs are compared at the same FID threshold.

\subsubsection{Planned and Actual Progress}
\label{app:execution_diagnostics}
\Cref{tab:execution_progress_diagnostics} compares the progress predicted by the plan with the progress of the AdamW update at each model scale. Appendix~\ref{app:convergence} relates the two on a fixed batch. Under the smoothness and positive-definite metric assumptions of Proposition~\ref{prop:executed_descent},
\[
 E(\bm\theta_{k+1})-E(\bm\theta_k)
 \leq-\mathcal P_k(\widehat{\bm\pi}_k)
 +\|\mathbf g_k\|_2\sqrt{2\varepsilon S_k}
 +\frac{L_E}{c_R}S_k,
 \qquad S_k=\mathcal A_\varepsilon(\widehat{\bm\pi}_k)+\Delta_k.
\]
where $\Delta_k$ is the excess quadratic cost of the actual update over the optimal local update. All quantities share the same batch, planning features, parameter subset, and metric. The measurements use FP32 batches ($N=32$) and the Sinkhorn objective. The actual AdamW increment includes the learning rate, optimizer state, and weight decay. Descent follows when planned progress exceeds the two error terms.

\begin{table}[H]
\centering\small
\caption{\textbf{Planned and actual progress.} Batch and objective are fixed within each row. $\Delta_k$ is per sample.}
\label{tab:execution_progress_diagnostics}
\begin{tabular}{@{}lccccc@{}}
\toprule
Model & $-\mathbf g_k^\top\mathbf d_{\widehat{\bm\pi}_k}$ & $-\mathbf g_k^\top\mathbf J\delta\bm\theta_k$ & $-\mathbf g_k^\top\Delta\mathbf u_k$ & $E_k-E_{k+1}$ & $\Delta_k$ \\
\midrule
TDAC B/2 & 0.4395 & 0.1898 & 0.0746 & 0.0399 & 6.3886 \\
TDAC L/2 & 0.3851 & 0.2019 & 0.0720 & 0.0357 & 7.6511 \\
TDAC XL/2 & 0.3507 & 0.1637 & 0.0562 & 0.0280 & 5.4554 \\
\bottomrule
\end{tabular}
\end{table}
In \cref{tab:execution_progress_diagnostics}, $\Delta\mathbf u_k=\mathbf u_{\bm\theta_{k+1}}-\mathbf u_{\bm\theta_k}$ and $E_k=E(\bm\theta_k)$. Both linearized and actual feature changes achieve positive progress across scales (curvature lowers realized progress below the linearized estimate, and both stay below planned progress), and energy decreases after each update. The per-sample gap uses estimated mobility; Appendix~\ref{app:action_convergence} gives the exact nonnegative identity.

\subsubsection{Planning Configuration and Guidance Interface}
\label{app:planning_implementation_details}
\Cref{tab:tdac_planning_configuration} gives the planning settings for each generator size. \Cref{tab:reference_configs} lists the fixed-control ActionFuse settings. Class restrictions determine the allowed transport pairs. The planner uses Sinkhorn scaling and a scalar dual update to approximate the projection onto the marginal and progress constraints (Appendix~\ref{app:regression_details}). The solver returns the reference plan if the final iterate is infeasible.
Refresh intervals count optimizer updates. Replacing or reordering the planning samples requires a new mobility estimate. Persistence requires matching support indices. On a fixed support, the current generated features and planning gradient are recomputed at each update; the stored mobility is reused only within its specified refresh interval. These choices determine both estimation cost and the samples entering plan selection.

\paragraph{Guidance and training.}
Both reference fields in Eq.~\ref{eq:tdac_cfg_velocity} start from the same generated feature distribution $q$. The guided field specifies a reference sample displacement. Equation~\ref{eq:batch_displacement} defines the planning displacement. In the ActionFuse implementation, Appendix~\ref{app:native_normalization} gives its correction and normalization before regression. For full TDAC, \cref{tab:tdac_planning_configuration} separates the reference displacement, reference coupling, planning energy, and class mask. The planning gradient is $\mathbf g_k=\nabla_{\mathbf u}\mathcal F(\mathbf u_{\bm\theta})$, supplied by the planning energy $\mathcal F$. After plan selection, Eq.~\ref{eq:shared_targets} constructs fixed targets at each regression layer. Gradients pass through the frozen feature maps to the generator while the coupling and targets remain fixed.

\Cref{tab:ablation_action_rank} compares solver and mobility settings; action relative error measures rank truncation vs.\ the full-rank solve (top), Monte Carlo probe error vs.\ a 64-probe reference (middle), and refresh-interval error vs.\ the reference mobility (bottom). The low-rank solver matches dense FID at far lower latency, with small FID sensitivity across probe and refresh settings.

\begin{table}[H]
\centering\small
\caption{\textbf{TDAC planning settings.} The ablation models are trained from scratch.}
\label{tab:tdac_planning_configuration}
\setlength{\tabcolsep}{3pt}
\renewcommand{\arraystretch}{0.94}
\begin{tabular}{@{}lccc@{}}
\toprule
Setting & B/2 & L/2 & XL/2 \\
\midrule
Parameters and metric $\mathbf R$ & All DiT ($\mathbf R=\mathbf I$) & All DiT ($\mathbf R=\mathbf I$) & All DiT ($\mathbf R=\mathbf I$) \\
Classes per batch & 32 & 32 & 32 \\
Samples per class (each) & 4 & 4 & 4 \\
Mobility structure & Low-rank + diag & Low-rank + diag & Low-rank + diag \\
Planning features & Final block ($d=256$) & Final block ($d=512$) & Final block ($d=768$) \\
Regression features & Blks 1--4 ($w_\ell=0.25$, LN) & Blks 1--4 ($w_\ell=0.25$, LN) & Blks 1--4 ($w_\ell=0.25$, LN) \\
Estimation probes $n_{\mathrm{probe}}$ & 16 & 16 & 16 \\
Validation probes & 16 & 16 & 16 \\
Retained rank $r$ & 8 & 8 & 8 \\
Isotropic floor $\sigma^2$ & 0.30 & 0.30 & 0.30 \\
Support refresh (updates) & 1 & 1 & 1 \\
Mobility refresh (updates) & 4 & 4 & 4 \\
Planning iterations $n_{\mathrm{plan}}$ & 25 & 25 & 25 \\
Mirror-step size $\eta_{\mathrm{plan}}$ & 0.05 & 0.05 & 0.05 \\
Marginal projections & 10 & 10 & 10 \\
Progress dual update & Dual step $0.01$, tol $10^{-4}$ & Dual step $0.01$, tol $10^{-4}$ & Dual step $0.01$, tol $10^{-4}$ \\
Feasibility tol. (both constraints) & $10^{-4}$ & $10^{-4}$ & $10^{-4}$ \\
Action relaxation $\varepsilon$ & 0.784 & 0.784 & 0.784 \\
Entropy / persistence $\rho,\lambda$ & $0.05,\, 0.92$ & $0.05,\, 0.92$ & $0.05,\, 0.92$ \\
Progress fraction $\gamma$ & 0.85 & 0.85 & 0.85 \\
Reference plan and class mask & Entropic OT, within-class & Entropic OT, within-class & Entropic OT, within-class \\
Energy $\mathcal F$ and gradient $\mathbf g_k$ & Sinkhorn $\mathcal W_\varepsilon$, $\nabla_{\mathbf u}\mathcal W_\varepsilon$ & Sinkhorn $\mathcal W_\varepsilon$, $\nabla_{\mathbf u}\mathcal W_\varepsilon$ & Sinkhorn $\mathcal W_\varepsilon$, $\nabla_{\mathbf u}\mathcal W_\varepsilon$ \\
Regression target step $\eta$ & 1.0 & 1.0 & 1.0 \\
Initialization & Scratch & Scratch & Scratch \\
\bottomrule
\end{tabular}
\end{table}

\begin{table}[H]
\centering\small
\caption{\textbf{Action solver and mobility estimation.} We use DiT-B/2 with $N=128$ and $d=256$.}
\label{tab:ablation_action_rank}
\setlength{\tabcolsep}{3.5pt}
\renewcommand{\arraystretch}{0.94}
\begin{tabular}{@{}lccccc@{}}
\toprule
\textbf{Solver Formulation} & \textbf{Rank $r$} & \textbf{Ridge $\varepsilon$} & \textbf{Solver Latency ($\mu$s)} & \textbf{Action Rel. Err.} & \textbf{FID-50K$\downarrow$} \\
\midrule
Dense Linear Solve & Full ($Nd$) & 0.784 & 8920.0 & \textbf{0.00\%} & 6.84 \\
Isotropic Floor ($\kappa\mathbf I$) & 0 & 0.784 & \textbf{12.1} & 4.82\% & 7.18 \\
\midrule
Low-Rank Woodbury ATG & 4 & 0.784 & 134.8 & 0.56\% & 6.91 \\
\textbf{Low-Rank Woodbury ATG (TDAC)} & 8 & 0.784 & 237.7 & 0.21\% & 6.84 \\
Low-Rank Woodbury ATG & 16 & 0.784 & 412.5 & \textbf{0.00\%} & \textbf{6.83} \\
\midrule
\textit{Spectral Ridge Sensitivity ($\varepsilon$)} & & & & & \\
Low Damping & 8 & 0.100 & 237.7 & 0.84\% & 7.12 \\
Default regularization & 8 & 0.784 & 237.7 & 0.21\% & \textbf{6.84} \\
Overdamped & 8 & 5.000 & 237.7 & \textbf{0.05\%} & 7.25 \\
\midrule
\multicolumn{6}{l}{\textit{Probe-budget sensitivity ($r=8$, four updates between mobility refreshes)}} \\
Setting & $n_{\mathrm{probe}}$ & Rank / refresh & Latency ($\mu$s) & Action Rel. Err. & FID-50K$\downarrow$ \\
Low probe budget & 4 & 8 / 4 & 328.2 & 0.62\% & 6.92 \\
Medium probe budget & 8 & 8 / 4 & \textbf{239.2} & 0.58\% & 6.86 \\
\textbf{Standard probe budget (TDAC)} & 16 & 8 / 4 & 273.0 & 0.58\% & \textbf{6.84} \\
Extended probe budget & 32 & 8 / 4 & 247.9 & \textbf{0.56\%} & \textbf{6.84} \\
\midrule
\multicolumn{6}{l}{\textit{Refresh sensitivity ($n_{\mathrm{probe}}=16$, $r=8$; error vs.\ reference mobility)}} \\
Setting & Support interval & Mobility interval & Latency ($\mu$s) & Action Rel. Err. & FID-50K$\downarrow$ \\
Frequent refresh & 1 update & 1 update & 265.2 & \textbf{0.11\%} & 6.85 \\
\textbf{Standard refresh (TDAC)} & 1 update & 4 updates & 256.2 & \textbf{0.11\%} & \textbf{6.84} \\
Infrequent refresh & 1 update & 8 updates & \textbf{254.1} & \textbf{0.11\%} & 6.85 \\
Relaxed refresh & 1 update & 16 updates & 257.4 & 0.14\% & 6.86 \\
\bottomrule
\end{tabular}
\end{table}

\subsubsection{Guidance Evaluation}
\label{app:guidance_evaluation}
\Cref{tab:guidance_evaluation} sweeps guidance across all three model scales at a unified 125K-step intermediate ActionFuse checkpoint (matching the 125K target in \cref{tab:reference_configs}); full-schedule checkpoints in \cref{tab:imagenet_results} reach 1.33/283.7, 1.10/286.1, and 1.05/278.7 FID/IS. Lowest FID occurs at $s=1.20$--$1.22$ (B/2), $s=1.22$--$1.24$ (L/2), and $s=1.16$ (XL/2); IS peaks at $s=2.50$ (B/2, XL/2) and $s=3.00$ (L/2). \Cref{fig:tdaction_imagenet_cfg} plots the FID--IS trade-off with baselines digitized from the reference figure.

\begin{table}[H]
\centering\small
\caption{\textbf{Guidance sweep at unified 125K checkpoints.} Paired 50K samples; full-schedule results are in \cref{tab:imagenet_results}.}
\label{tab:guidance_evaluation}
\setlength{\tabcolsep}{4.2pt}
\begin{tabular}{@{}p{0.13\linewidth}*{9}{c}@{}}
\toprule
& \multicolumn{3}{c}{\textbf{TDAC B/2} (125K)} & \multicolumn{3}{c}{\textbf{TDAC L/2} (125K)} & \multicolumn{3}{c}{\textbf{TDAC XL/2} (125K)} \\
\cmidrule(lr){2-4} \cmidrule(lr){5-7} \cmidrule(lr){8-10}
Regime & $s$ & FID $\downarrow$ & IS $\uparrow$ & $s$ & FID $\downarrow$ & IS $\uparrow$ & $s$ & FID $\downarrow$ & IS $\uparrow$ \\
\midrule
No guidance & 1.00 & 4.17 & 170.7 & 1.00 & 3.55 & 181.3 & 1.00 & 3.06 & 184.4 \\
\midrule
& 1.18 & 2.03 & 233.1 & 1.16 & 1.86 & 230.0 & 1.10 & 1.86 & 222.3 \\
& 1.20 & \textbf{2.00} & 239.4 & 1.18 & 1.80 & 235.6 & 1.14 & 1.67 & 241.6 \\
& 1.22 & \textbf{2.00} & 245.5 & 1.20 & 1.76 & 241.6 & 1.16 & \textbf{1.65} & 250.0 \\
Fine sweep & 1.24 & 2.03 & 251.7 & 1.22 & \textbf{1.74} & 246.8 & 1.18 & 1.67 & 257.5 \\
& 1.26 & 2.09 & 257.5 & 1.24 & \textbf{1.74} & 251.6 & 1.20 & 1.74 & 264.7 \\
& 1.28 & 2.18 & 263.1 & 1.26 & 1.77 & 256.7 & 1.25 & 2.02 & 282.2 \\
& 1.30 & 2.28 & 267.7 & 1.28 & 1.80 & 261.3 & 1.28 & 2.24 & 290.9 \\
\midrule
& 1.50 & 3.83 & 310.0 & 1.50 & 2.91 & 304.0 & 1.50 & 4.94 & 342.6 \\
& 2.00 & 8.44 & 355.3 & 2.00 & 6.93 & 359.1 & 2.00 & 12.01 & 402.6 \\
Coarse Grid & 2.50 & 11.44 & \textbf{364.3} & 2.50 & 9.98 & 377.4 & 2.50 & 16.20 & \textbf{413.7} \\
& 3.00 & 12.38 & 363.8 & 3.00 & 11.87 & \textbf{381.0} & 3.00 & 18.50 & 413.2 \\
& 3.50 & 12.62 & 363.1 & 3.50 & 12.65 & 379.8 & 3.50 & 19.35 & 410.5 \\
\bottomrule
\end{tabular}
\end{table}

\subsubsection{Batch Size and Training Cost}
\label{app:training_cost_details}
Feature regression and optimization take over 90\% of measured update time and peak memory (\cref{tab:training_cost_breakdown}, uncompiled A100 bfloat16 profile at 512 samples/GPU, 16 microbatches of 32, full multi-layer MAE). Mobility estimation (4 probes/GPU, 4-update refresh) and plan optimization take under 4\% combined.

The global generated batch counts samples contributing to one optimizer update, even when processed in microbatches. Per-class supports count generated, data, and independent generated-reference samples. Planning uses the $N$ generated and $M$ data points that define $\mathbf J$, $\mathbf K$, and the coupling. In \cref{tab:reference_configs}, W-Flow and ActionFuse both generate 64 samples per class group. Their 128 and 32 groups give global batches of 8192 and 2048. The full TDAC configuration lists 32 groups of four generated planning points (\cref{tab:tdac_planning_configuration}).

In Eq.~\ref{eq:batch_atg_value}, a displacement component along a mobility eigenvector with eigenvalue $\nu$ has action weight $1/(\nu+\varepsilon)$. TDAC assigns higher cost to directions with weaker generator responses while retaining a fraction of positive reference progress; Appendix~\ref{app:action_convergence} links this batch-level plan to realized optimizer progress.

Generated-sample epochs equal updates times the global batch divided by dataset size. At 100K updates, a batch of 8192 processes four times as many generated samples as 2048; \cref{tab:training_cost_records} reports whole-run wall-clock and GPU-hour totals at each method's training batch size, while \cref{tab:training_cost_breakdown} profiles single-device component compute at 512 samples per GPU. Step times come from \cref{tab:fid_training_epochs}, with XL/2 from \cref{tab:training_cost_breakdown}.

\begin{table}[H]
\centering\small
\caption{\textbf{Training resources and total cost.} Each run uses four A100 GPUs for 100K updates.}
\label{tab:training_cost_records}
\begin{tabular}{@{}lcccccc@{}}
\toprule
Model & GPUs & Global batch & Updates & s/step & Elapsed (h) & GPU-hours \\
\midrule
Drifting B/2 & 4 & 8192 & 100K & 21.5 & 597.2 & 2389 \\
Drifting L/2 & 4 & 8192 & 100K & 36.8 & 1022.2 & 4089 \\
W-Flow B/2 & 4 & 8192 & 100K & 28.4 & 788.9 & 3156 \\
W-Flow L/2 & 4 & 8192 & 100K & 48.6 & 1350.0 & 5400 \\
TDAC B/2 & 4 & 2048 & 100K & \textbf{5.8} & \textbf{161.1} & \textbf{644} \\
TDAC L/2 & 4 & 2048 & 100K & \textbf{10.2} & \textbf{283.3} & \textbf{1133} \\
TDAC XL/2 & 4 & 2048 & 100K & 11.0 & 305.6 & 1222 \\
\bottomrule
\end{tabular}
\end{table}

\begin{table}[H]
\centering\small
\caption{\textbf{TDAC per-device training cost and peak memory.} A100 80GB bfloat16 profile at 512 samples/GPU (4-update refresh).}
\label{tab:training_cost_breakdown}
\begin{tabular}{@{}lccc@{}}
\toprule
Component & B/2 & L/2 & XL/2 \\
\midrule
Mobility estimation (s/update) & 0.098 & 0.151 & 0.178 \\
Plan optimization (s/update) & 0.061 & 0.063 & 0.061 \\
Feature regression and optimizer (s/update) & 6.503 & 9.365 & 10.773 \\
Whole profiled device update (s/update) & 6.662 & 9.579 & 11.012 \\
\midrule
\multicolumn{4}{l}{\textit{Peak allocated memory (GiB)}} \\
Mobility estimation & 6.48 & 9.86 & 13.96 \\
Plan optimization & 5.28 & 8.06 & 11.45 \\
Feature regression and optimizer & 18.79 & 31.49 & 38.65 \\
Whole update & 18.79 & 31.49 & 38.65 \\
\midrule
100K-update training total (GPU-hours, \cref{tab:training_cost_records}) & 644 & 1133 & 1222 \\
\bottomrule
\end{tabular}
\end{table}

\subsubsection{Pixel-Space Generation}
\label{app:pixel_space}

\Cref{tab:pixel_space_ablation} shows that TDAC achieves the lowest 50K FID across every tested latent- and pixel-space setting in one generator evaluation, reaching 2.64 FID in pixel space with combined MAE-640 and ConvNeXt-V2 features~\citep{deng2026generative}. Switching to the wider MAE with longer pretraining and classification fine-tuning lowers FID more in pixel space than in latent space for all three methods. TDAC's pixel-to-latent FID gap narrows from 12.98 to 3.57 points. Adding ConvNeXt-V2 features further lowers pixel-space FID for all three methods. TDAC's FID lead over both baselines narrows with MAE-640 in both spaces and shrinks further with the combined features in pixel space. Only TDAC attains a lower best FID in pixel space than in latent space.

The feature encoder defines the space in which TDAC evaluates displacements. The Jacobian $\mathbf J$ of these features with respect to the generator parameters determines the batch mobility,
\begin{equation}
    \mathbf K = \mathbf J \mathbf R^{-1} \mathbf J^\top \in \mathbb{R}^{Nd \times Nd},
\end{equation}
where $(\mathbf K^\varepsilon)^{-1}$ assigns higher action cost to displacements along directions with lower mobility (Appendix~\ref{app:training_cost_details}).

\begin{table}[H]
\centering\small
\caption{\textbf{Latent- and pixel-space generation on ImageNet $256\times256$.} Panel (b) includes reproduced W-Flow results.}
\label{tab:pixel_space_ablation}
\begin{tabular}{@{}llccc@{}}
\toprule
\textbf{Feature Encoder $\phi$} & \textbf{Pretraining \& Scale} & \textbf{Drifting~\citep{deng2026generative}} & \textbf{W-Flow~\citep{han2026wflow}} & \textbf{TDAC} \\
\midrule
\multicolumn{5}{@{}l}{\textit{(a) Latent Space Generation (DiT-B/2, 100-Epoch 50K FID$\downarrow$)}} \\
MAE & Width 256, Epoch 192 & 8.46 & 7.29 & \textbf{6.84} \\
MAE + Cls Ft. & Width 640, Epoch 1280 & 3.36 & 3.12 & \textbf{2.88} \\
\midrule
\multicolumn{5}{@{}l}{\textit{(b) Pixel Space Generation (DiT-B/16, No VAE, 100-Epoch 50K FID$\downarrow$)}} \\
MAE & Width 256, Epoch 192 & 32.11 & 28.45 & \textbf{19.82} \\
MAE + Cls Ft. & Width 640, Epoch 1280 & 9.35 & 8.16 & \textbf{6.45} \\
+ MAE w/ ConvNeXt-V2 & Width 640 + ConvNeXt-V2 & 3.70 & 3.32 & \textbf{2.64} \\
\bottomrule
\end{tabular}
\end{table}

\subsubsection{Earlier Passive-Response Evaluation}
\label{app:historical_passive}

\label{app:evidence_status}

\label{sec:training_geometry_audit}
We reconstruct rank-four passive mobility at each ImageNet-256 B/2 checkpoint to track conditioning, subspace rotation, and plan churn. \Cref{tab:training_mobility_audit} shows mild anisotropy and large subspace rotations. Churn is averaged over 20 fixed, seeded canonical-feature probe sets. The angle is the mean principal angle to the preceding checkpoint.

\begin{table}[!htbp]
\centering\small
\caption{\textbf{Passive mobility across ImageNet-256 B/2 checkpoints.} Windows count mobility estimation intervals.}
\label{tab:training_mobility_audit}
\begin{tabular}{cccccc}
\toprule
Step & Windows & $\kappa(\bm\Lambda)$ & $\kappa(\bar{\mathbf A}^{-1})$ & Angle & Plan churn \\
\midrule
20K & 156 & 1.12 & 1.259 & -- & -- \\
40K & 312 & 1.63 & 1.282 & $87.0^\circ$ & $6.70\!\times\!10^{-4}$ \\
54K & 421 & 1.17 & 1.262 & $88.1^\circ$ & $3.47\!\times\!10^{-4}$ \\
56K & 437 & 1.32 & 1.273 & $87.9^\circ$ & $8.52\!\times\!10^{-2}$ \\
\bottomrule
\end{tabular}
\end{table}

At 56K, the passive plan lowers the entropic objective by $1.08\pm0.19\%$ relative to the Euclidean plan under the same passive cost, with an unregularized action ratio of $1.0054\pm0.0017$. The mean principal angle stays near $88^\circ$ from 54K to 56K while plan churn rises by more than two orders of magnitude (\cref{tab:training_mobility_audit}), showing that similar subspace angles can accompany distinct transport-plan shifts. The regularized inverse remains consistently conditioned across the four checkpoints ($\kappa(\bar{\mathbf A}^{-1})\in[1.259, 1.282]$).

Same-latent optimizer responses yield a per-sample low-rank mobility estimate with isotropic regularization and a ridge. We scale its regularized inverse to the feature dimension by trace, blend it with the identity, and form the pairwise cost using this per-sample half-quadratic form with the horizon scale absorbed into the temperature.

\label{app:checkpoint_diagnostic}

The 10K sweep selects $s=1.25$ by FID, although IS is higher at $s=1.30$ (\cref{tab:tdaction_cfg_sweep}). On the independent 50K evaluation of the same frozen 56K-step EMA checkpoint, $s=1.25$ lowers FID from 3.56 to 2.46 and raises IS relative to $s=1.00$ (\cref{tab:tdaction_fixed_checkpoint_50k}). The two 50K settings share the same generator, so this comparison measures the effect of guidance at inference. The sweep uses seed 0 and the independent evaluation uses seed 2027. Sampling is class-balanced and the 50K runs share ImageNet-256 statistics. The scale is $s=w+1$. IS reports mean and standard deviation across splits.

\begin{table}[H]
\centering\small
\caption{\textbf{Held-out CFG evaluation.} We report FID and IS on 50K samples from the 56K-step EMA checkpoint.}
\label{tab:tdaction_fixed_checkpoint_50k}
\begin{tabular}{lcccc}
\toprule
Setting & Code $s$ & Actual $w$ & FID$\downarrow$ & IS$\uparrow$ \\
\midrule
Intrinsic & 1.00 & 0.00 & 3.56 & $197.8\pm4.9$ \\
Selected CFG & 1.25 & 0.25 & \textbf{2.46} & $\boldsymbol{251.2\pm4.9}$ \\
\bottomrule
\end{tabular}
\end{table}

\begin{table}[H]
\centering\small
\caption{\textbf{Effect of CFG scale.} Each scale is evaluated on 10K samples from the 56K-step EMA checkpoint.}
\label{tab:tdaction_cfg_sweep}
\begin{tabular}{cccc}
\toprule
Code $s$ & $w$ & FID-10K $\downarrow$ & IS $\uparrow$ \\
\midrule
1.00 & 0.00 & 6.09 & $155.8\pm7.0$ \\
1.05 & 0.05 & 5.71 & $161.7\pm6.7$ \\
1.10 & 0.10 & 5.38 & $170.5\pm6.0$ \\
1.14 & 0.14 & 5.21 & $177.0\pm6.2$ \\
1.19 & 0.19 & 5.02 & $185.8\pm5.9$ \\
1.25 & 0.25 & \textbf{4.99} & $191.7\pm4.9$ \\
1.30 & 0.30 & 5.05 & $\boldsymbol{199.3\pm5.1}$ \\
\bottomrule
\end{tabular}
\end{table}

\subsubsection{Additional Generated Samples}
\label{app:generated_samples}

\Cref{fig:imagenet_more_samples_l2,fig:imagenet_more_samples_xl2_b2} show uncurated one-step ImageNet samples at the same guidance scale, $s=1.15$. The L/2 bird grids vary in pose and background, and the XL/2 flamingo grid shows both individual birds and groups. In the B/2 valley grid, the terrain and framing vary across samples. Across the three model sizes, the grids show how recognizable class features can coexist with variation in the subject, its surroundings, and the composition of the image.

\Cref{fig:imagenet_more_samples_xl2_largecfg} shows a separate set of XL/2 samples at $s=3.00$, where $s=w+1$. The streetcar grid retains the shape of the vehicle across front and side views, and the sports car grid varies in color and background. The king penguin grid also varies in how many birds appear and how they fill the frame. Within the $s=3.00$ grids, framing changes both how an individual subject is seen and how multiple subjects occupy the scene. The XL/2 sweep in \cref{tab:guidance_evaluation} provides a distribution-level view. At the sweep checkpoint, FID rises from 1.65 at $s=1.16$ to 18.50 at $s=3.00$. This measurement complements the qualitative examples of within-class variation.

\Cref{fig:matched_imagenet_samples} compares samples from Drifting, W-Flow, and TDAC after 30K updates on five shared ImageNet classes. In the bald eagle row, TDAC generates both close-up views and birds in flight. Its house finches appear against foliage and open backgrounds. TDAC has the lowest 50K FID in this comparison (\cref{tab:tdaction_ablation}).

\begin{figure}[H]
    \centering
    \includegraphics[width=\linewidth,height=\dimexpr\textheight-15pt\relax,trim=2.7bp 8.0bp 2.7bp 2.7bp,clip,keepaspectratio]{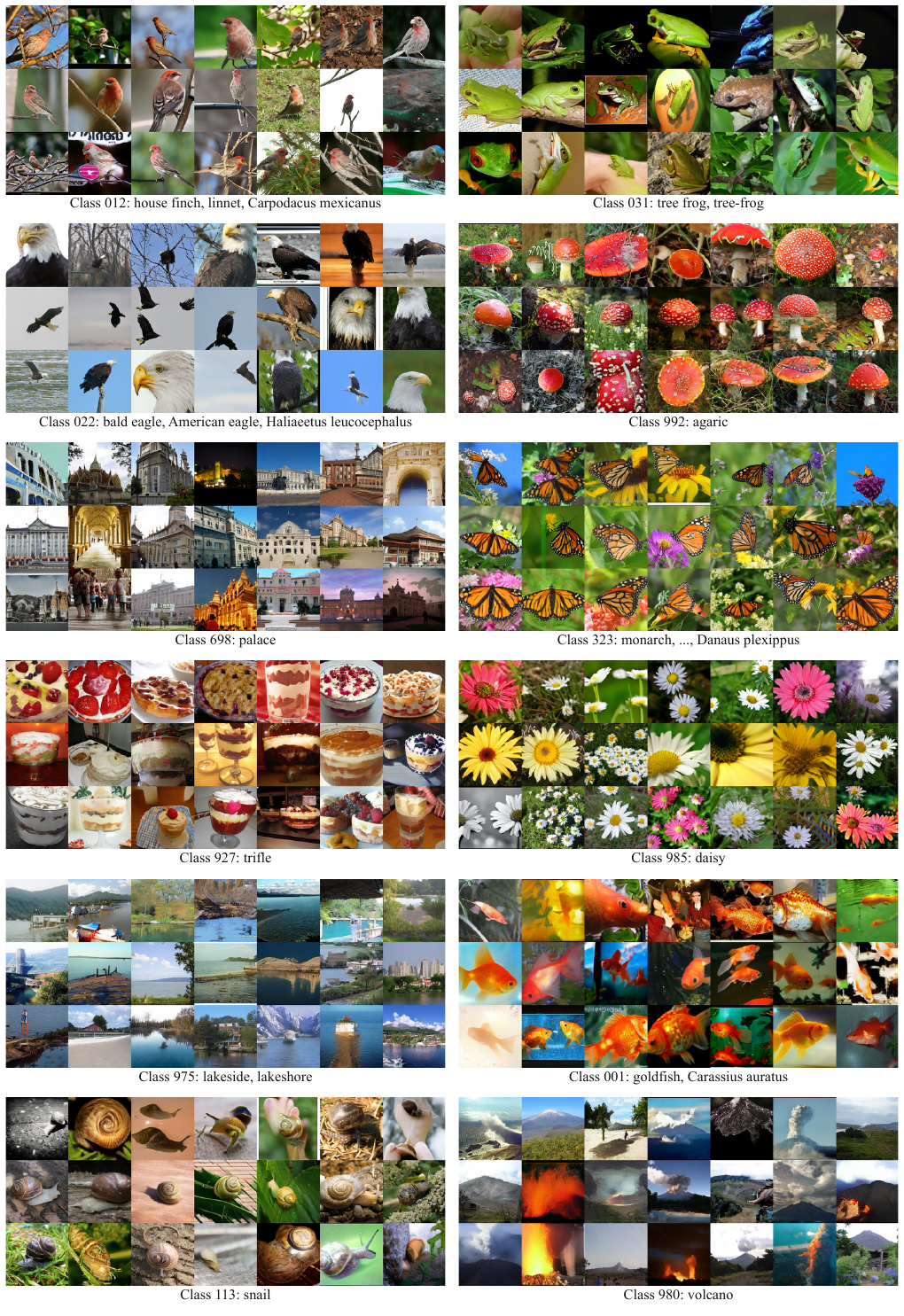}
    \captionsetup{skip=4pt}
    \caption{\textbf{Uncurated samples generated by TDAC L/2 on ImageNet $256\times 256$ with guidance scale $s=1.15$ ($w=0.15$).}}
    \label{fig:imagenet_more_samples_l2}
\end{figure}

\begin{figure}[H]
    \centering
    \includegraphics[width=\linewidth,height=\dimexpr\textheight-15pt\relax,trim=2.2bp 7.0bp 2.2bp 2.7bp,clip,keepaspectratio]{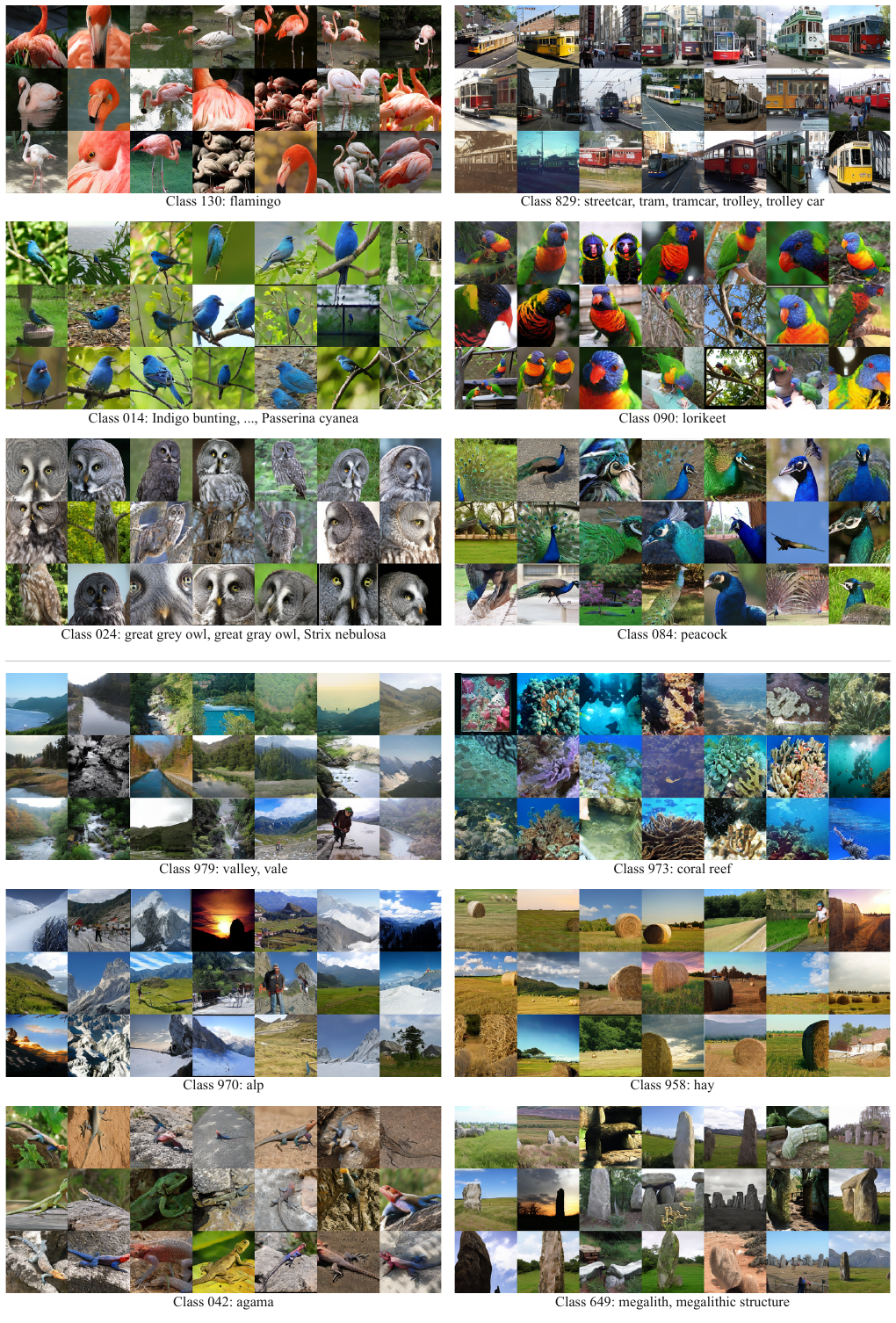}
    \captionsetup{skip=4pt}
    \caption{\textbf{Uncurated TDAC XL/2 \emph{(Top)} and B/2 \emph{(Bottom)} samples on ImageNet $256\times 256$ with guidance scale $s=1.15$ ($w=0.15$).}}
    \label{fig:imagenet_more_samples_xl2_b2}
\end{figure}

\begin{figure}[H]
    \centering
    \includegraphics[width=\linewidth,height=\dimexpr\textheight-15pt\relax,trim=4.1bp 9.9bp 4.1bp 4.1bp,clip,keepaspectratio]{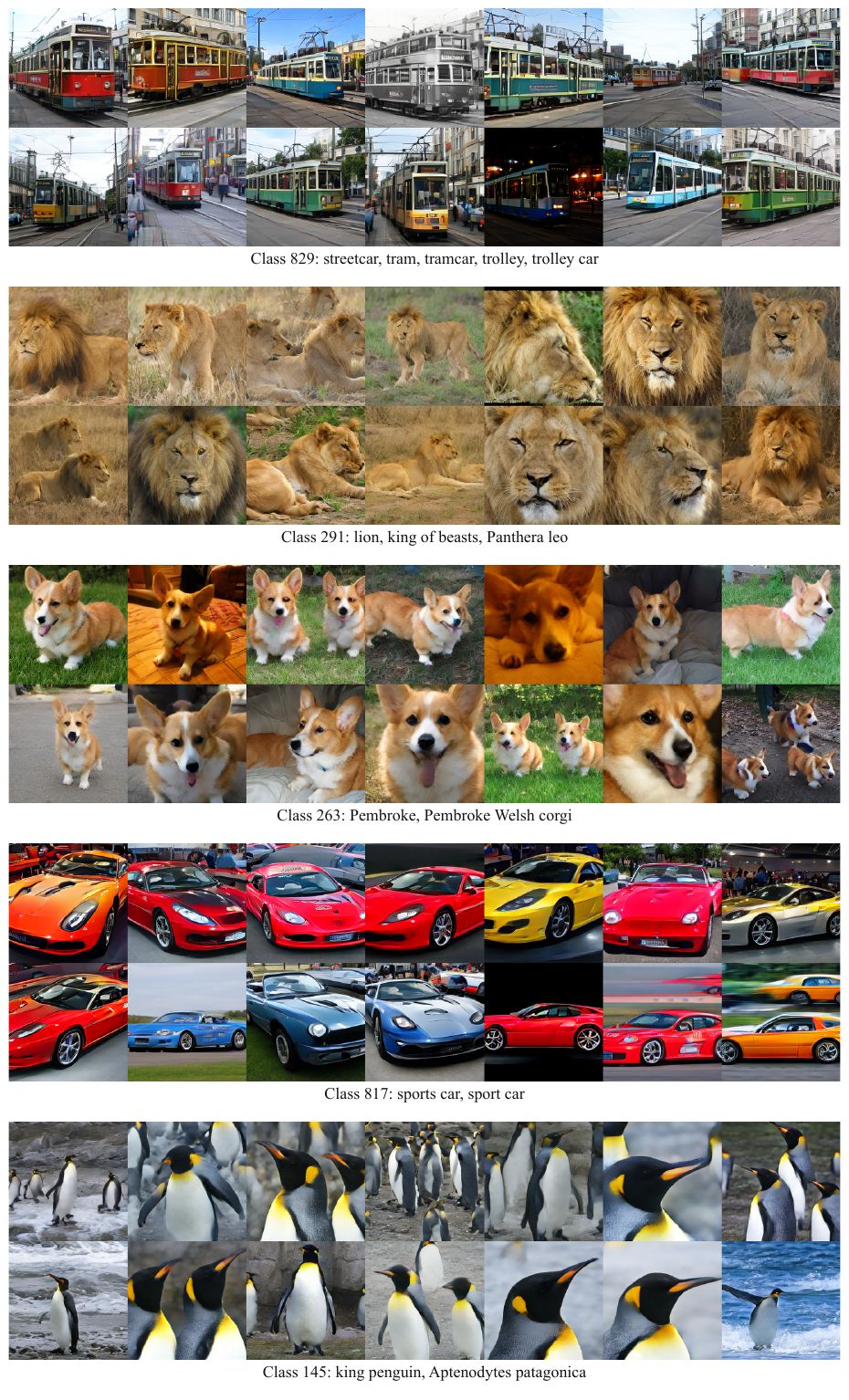}
    \captionsetup{skip=4pt}
    \caption{\textbf{Uncurated samples generated by TDAC XL/2 on ImageNet $256\times 256$ with guidance scale $s=3.00$ ($w=2.00$).}}
    \label{fig:imagenet_more_samples_xl2_largecfg}
\end{figure}

\begin{figure}[H]
    \centering
    \includegraphics[width=\linewidth,trim=5.5bp 6.0bp 5.5bp 7.0bp,clip]{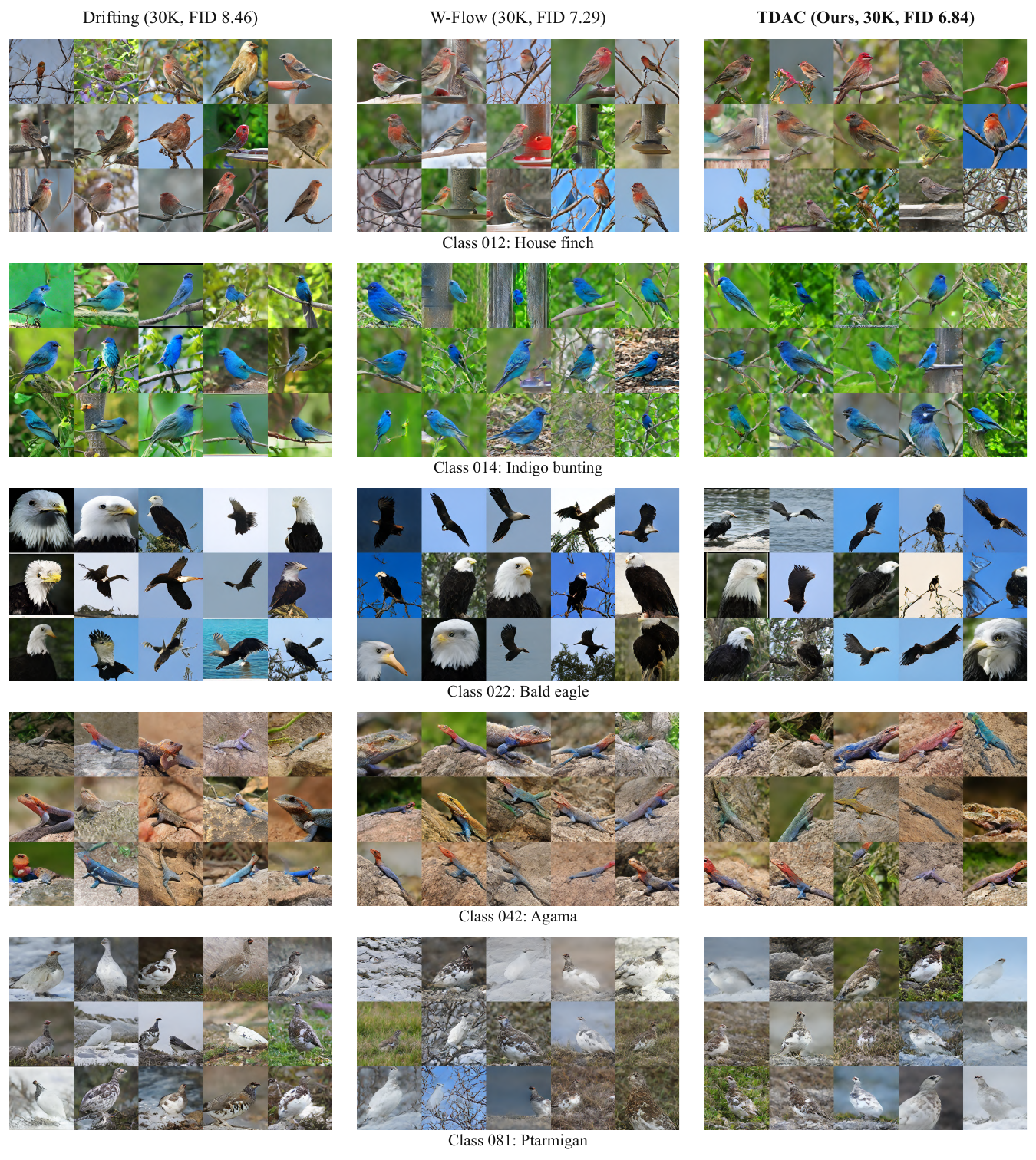}
    \captionsetup{skip=4pt}
    \caption{\textbf{Side-by-side comparison on ImageNet $256\times256$.} Drifting (left), W-Flow (middle), and TDAC (right) after 30K updates.}
    \label{fig:matched_imagenet_samples}
\end{figure}

Following the imbalanced-target experiment of W-Flow~\citep{han2026wflow}, we train Drifting, W-Flow, and TDAC for 2K updates from the same DiT-B/2 checkpoint on an ImageNet mixture of 95\% golden retrievers and 5\% flamingos. \Cref{fig:imagenet_mode_coverage} shows uncurated samples from the three methods. In the 1,000-sample evaluation of each method, the feature-similarity criterion assigns 0.0\% of Drifting samples, 3.7\% of W-Flow samples, and 4.1\% of TDAC samples to the flamingo mode. TDAC's measured share is closest to the 5\% target.

\begin{figure}[H]
    \centering
    \includegraphics[width=\linewidth,trim=9.4bp 5.6bp 9.4bp 7.5bp,clip]{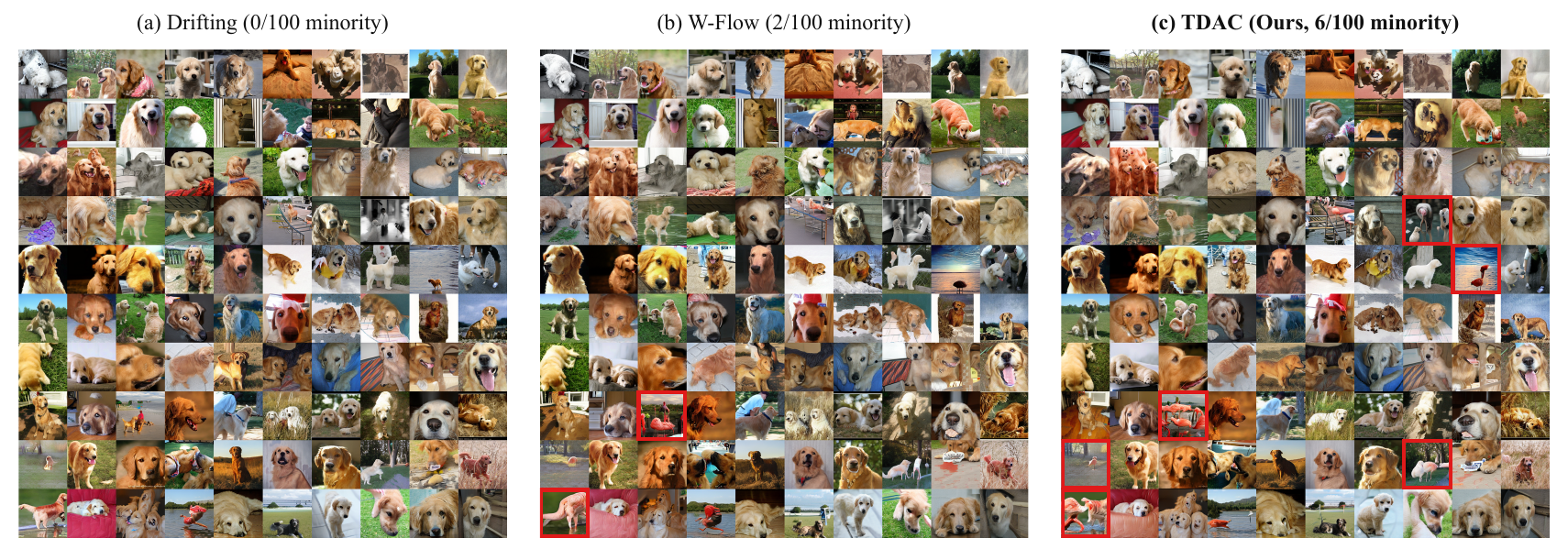}
    \captionsetup{skip=4pt}
    \caption{\textbf{Uncurated samples from the mode coverage experiment.} Red boxes mark samples assigned to the flamingo mode: 0, 2, and 6 out of 100 for Drifting, W-Flow, and TDAC, respectively.}
    \label{fig:imagenet_mode_coverage}
\end{figure}

\section{Proofs of the Action-to-Go Formulation}
\label{app:theory}

We fix the planning batch, generator parameters, and control metric $\mathbf R$. Following the action formulations of Benamou--Brenier~\citep{benamou2000computational} and Action Matching~\citep{neklyudov2023action}, we first express the batch cost as a parameter-path problem.
\label{app:batch_atg}
The matrix $\mathbf K^\varepsilon$ remains positive definite even when $\mathbf J$ is rank deficient.
Let $N$ and $M$ denote the generated and data sample counts, $d$ the planning-feature dimension, and $P$ the parameter dimension. Let $\mathbf R\succ0$, $\mathbf J\in\R^{Nd\times P}$, and $\mathbf K=\mathbf J\mathbf R^{-1}\mathbf J^\top$.

Fix $\mathbf R\succ0$, $\varepsilon>0$, and $\mathbf J$ on the unit interval. For finite-energy, absolutely continuous parameter paths with $\bm\theta_0=\bm\theta$ and $\bm\theta_1=\bm\theta+\delta\bm\theta$, convexity of the quadratic form gives
\begin{equation*}
 \int_0^1\dot{\bm\theta}_\tau^\top \mathbf R\dot{\bm\theta}_\tau\,\mathrm d\tau
 \geq\left(\int_0^1\dot{\bm\theta}_\tau\,\mathrm d\tau\right)^\top
 \mathbf R\left(\int_0^1\dot{\bm\theta}_\tau\,\mathrm d\tau\right)
 =\delta\bm\theta^\top \mathbf R\delta\bm\theta.
\end{equation*}
The constant-speed path $\bm\theta_\tau=\bm\theta+\tau\delta\bm\theta$ attains equality. Since the terminal penalty depends only on the endpoint, Eq.~\ref{eq:batch_atg_problem} is equivalent to
\begin{equation}
     \mathcal A_\varepsilon(\bm\pi)=\inf_{\{\bm\theta_\tau:\bm\theta_0=\bm\theta\}} \bigg\{\frac12\int_0^1\dot{\bm\theta}_\tau^\top \mathbf R\dot{\bm\theta}_\tau\,\mathrm d\tau  +\frac1{2\varepsilon}\|\mathbf J(\bm\theta_1-\bm\theta)-\mathbf d_{\bm\pi}\|_2^2\bigg\},
    \label{eq:app_local_parameter_action}
\end{equation}
where the infimum is over finite-energy, absolutely continuous paths in parameter space, with $\mathbf R$ and $\mathbf J$ fixed. This action measures parameter motion. The Benamou--Brenier action measures sample-space motion.

\begin{proposition}[Exact Action-to-Go]
\label{prop:batch_atg}
For $\mathbf R\succ0$ and $\varepsilon>0$, Eq.~\ref{eq:batch_atg_value} gives the value of Eq.~\ref{eq:batch_atg_problem}, convex in $\bm\pi$ on $\mathcal C(\mathbf a,\mathbf b)$.
\end{proposition}

\begin{proof}[Proof of Proposition~\ref{prop:batch_atg}]
For the fixed displacement $\mathbf d_{\bm\pi}$, the first-order condition for Eq.~\ref{eq:batch_atg_problem} is
\begin{equation}
    \left(\mathbf R+\varepsilon^{-1}\mathbf J^\top \mathbf J\right)\delta\bm\theta
    =\varepsilon^{-1}\mathbf J^\top \mathbf d_{\bm\pi}.
    \label{eq:app_batch_normal_equations}
\end{equation}
The coefficient matrix is positive definite, so the solution is unique:
\begin{equation}
    \delta\bm\theta^\star
    =\mathbf R^{-1}\mathbf J^\top(\mathbf K^\varepsilon)^{-1}\mathbf d_{\bm\pi}.
    \label{eq:app_batch_control_solution}
\end{equation}
The identity $\mathbf J\delta\bm\theta^\star=\mathbf K(\mathbf K^\varepsilon)^{-1}\mathbf d_{\bm\pi}$ gives
\begin{equation}
    \mathbf J\delta\bm\theta^\star-\mathbf d_{\bm\pi}
    =-\varepsilon(\mathbf K^\varepsilon)^{-1}\mathbf d_{\bm\pi}.
    \label{eq:app_batch_residual}
\end{equation}
Substituting Eqs.~\ref{eq:app_batch_control_solution} and \ref{eq:app_batch_residual} into Eq.~\ref{eq:batch_atg_problem} gives
\begin{equation}
    \begin{aligned}
\mathcal A_\varepsilon(\bm\pi)
    &=\tfrac12\mathbf d_{\bm\pi}^\top(\mathbf K^\varepsilon)^{-1}\mathbf K(\mathbf K^\varepsilon)^{-1}\mathbf d_{\bm\pi}
    +\tfrac\varepsilon2\mathbf d_{\bm\pi}^\top(\mathbf K^\varepsilon)^{-2}\mathbf d_{\bm\pi}\\
    &=\tfrac12\mathbf d_{\bm\pi}^\top(\mathbf K^\varepsilon)^{-1}
      (\mathbf K+\varepsilon\mathbf I)(\mathbf K^\varepsilon)^{-1}\mathbf d_{\bm\pi}\\
    &=\tfrac12\mathbf d_{\bm\pi}^\top(\mathbf K^\varepsilon)^{-1}\mathbf d_{\bm\pi}.
\end{aligned}
    \label{eq:app_batch_value}
\end{equation}
The positive-definite quadratic form is convex in $\bm\pi$ because $\mathbf d_{\bm\pi}$ is affine. It vanishes exactly when $\mathbf d_{\bm\pi}=\mathbf 0$.
\end{proof}

\begin{proposition}[Progress-preserving local-action selection]
\label{prop:progress_action_selection}
Fix the batch, linearization, $\mathbf R\succ0$, $\varepsilon>0$, and $\rho,\lambda\geq0$.
If $\mathcal P_k(\bm\pi_{\mathrm{ref}})>0$, $0<\gamma\leq1$, and all positive-weight regularizers are finite at $\bm\pi_{\mathrm{ref}}$, an exact minimizer of Eq.~\ref{eq:progress_constrained_batch_atg} exists and satisfies
\begin{equation}
    \begin{aligned}
    \mathcal J(\bm\pi_k^\star)\leq\mathcal J(\bm\pi_{\mathrm{ref}}),\quad
    \mathcal P_k(\bm\pi_k^\star)\geq\gamma\mathcal P_k(\bm\pi_{\mathrm{ref}})>0.
    \end{aligned}
    \label{eq:action_selection_guarantee}
\end{equation}
\end{proposition}

For a feasible selected plan $\bm\pi$, let $\delta\bm\theta$ be the actual optimizer increment and $\Delta_{\bm\pi}\geq0$ its excess quadratic fitting cost, as defined in Eq.~\ref{eq:app_action_execution_gap}. The progress constraint gives
\begin{equation}
    \begin{aligned}
    -\mathbf g_k^\top \mathbf J\delta\bm\theta
    \geq\gamma\mathcal P_k(\bm\pi_{\mathrm{ref}})-\|\mathbf g_k\|_2\sqrt{2\varepsilon[\mathcal A_\varepsilon(\bm\pi)+\Delta_{\bm\pi}]}.
    \end{aligned}
    \label{eq:action_executed_progress}
\end{equation}

\begin{proof}[Proof of Proposition~\ref{prop:progress_action_selection} and Eq.~\ref{eq:action_executed_progress}]
The progress constraint is a closed half-space because $\mathbf d_{\bm\pi}$ is affine. Its intersection with the compact coupling set contains $\bm\pi_{\mathrm{ref}}$. The objective is lower semicontinuous and finite at the reference plan, with zero-weight terms omitted. Hence a minimizer exists, and its optimality and feasibility give Eq.~\ref{eq:action_selection_guarantee}.

Write $Q_{\bm\pi}$ for the quadratic objective minimized in Eq.~\ref{eq:batch_atg_problem}. For the actual increment in the same fixed local model, the nonnegative parameter-cost term gives
\begin{equation*}
 \|\mathbf J\delta\bm\theta-\mathbf d_{\bm\pi}\|_2^2
 \leq2\varepsilon Q_{\bm\pi}(\delta\bm\theta)
 =2\varepsilon[\mathcal A_\varepsilon(\bm\pi)+\Delta_{\bm\pi}].
\end{equation*}
Using $\mathcal P_k(\bm\pi)=-\mathbf g_k^\top \mathbf d_{\bm\pi}$ and Cauchy--Schwarz gives
\begin{equation*}
     -\mathbf g_k^\top \mathbf J\delta\bm\theta =\mathcal P_k(\bm\pi)-\mathbf g_k^\top[\mathbf J\delta\bm\theta-\mathbf d_{\bm\pi}] \geq\mathcal P_k(\bm\pi)-\|\mathbf g_k\|_2\sqrt{2\varepsilon[\mathcal A_\varepsilon(\bm\pi)+\Delta_{\bm\pi}]},
\end{equation*}
Feasibility gives $\mathcal P_k(\bm\pi)\geq\gamma\mathcal P_k(\bm\pi_{\mathrm{ref}})$. Applying this lower bound yields Eq.~\ref{eq:action_executed_progress}.
\end{proof}

\phantomsection
\paragraph{Sinkhorn reference field.}
\label{app:sinkhorn_reference}
For quadratic-cost Sinkhorn divergence, W-Flow~\citep{han2026wflow} supplies the reference field
\begin{equation}
    V^\rho_{q,p}(\mathbf x)
    =-\nabla_{\mathbf x}\frac{\delta S_\rho(\cdot,p)}{\delta q}(q)(\mathbf x)
    =T^\rho_{q,p}(\mathbf x)-T^\rho_{q,q}(\mathbf x).
    \label{eq:wflow_reference_field}
\end{equation}
where $\rho$ controls entropy regularization, $S_\rho$ is the Sinkhorn divergence, $q$ is the current generated distribution, and $T^\rho_{q,p}$ and $T^\rho_{q,q}$ denote the data and generated-reference barycentric maps, respectively. The field moves samples toward data barycenters and subtracts the generated self-transport contribution. It defines first-order progress for the Sinkhorn energy, with the self term estimated from an independent generated batch.

\paragraph{Decomposing candidate displacements.}
Let $q$ and $p$ be the empirical generated and data distributions on the planning supports with weights $\mathbf a$ and $\mathbf b$. First take $\psi$ to be the identity, so $\mathbf U_{i,:}^{\top}=\mathbf x_i$ and $\mathbf Y_{j,:}^{\top}=\mathbf y_j$. We compute the barycentric maps from exact balanced couplings using the same cost and entropy regularization as $\bm\pi_{\mathrm{ref}}$. Adding and subtracting the reference maps gives
\begin{equation}
    (\mathbf D^{-1}\bm\pi\mathbf Y-\mathbf U)_{i,:}^{\top} ={}\underbrace{V^\rho_{q,p}(\mathbf U_{i,:}^{\top})}_{\text{reference field in Eq.~\ref{eq:wflow_reference_field}}} +\underbrace{\left[(\mathbf D^{-1}\bm\pi\mathbf Y)_{i,:}^{\top} -T^\rho_{q,p}(\mathbf U_{i,:}^{\top})\right]}_{\text{coupling change}} +\underbrace{\left[T^\rho_{q,q}(\mathbf U_{i,:}^{\top}) -\mathbf U_{i,:}^{\top}\right]}_{\text{self-transport correction}}.
\label{eq:app_candidate_field_relation}
\end{equation}
At the reference coupling, the data-barycenter difference vanishes, giving
\begin{equation*}
 (\mathbf D^{-1}\bm\pi_{\mathrm{ref}}\mathbf Y-\mathbf U)_{i,:}^{\top}
 =V^\rho_{q,p}(\mathbf x_i)+T^\rho_{q,q}(\mathbf x_i)-\mathbf x_i.
\end{equation*}
Equality with the unscaled Sinkhorn field holds exactly when $T^\rho_{q,q}(\mathbf x_i)=\mathbf x_i$. More generally, equality with the field step in Eq.~\ref{eq:local_velocity_map} requires
\begin{equation*}
 \frac1{a_i}\sum_j\pi_{ij}\mathbf y_j
 =\mathbf x_i+h\bigl[T^\rho_{q,p}(\mathbf x_i)-T^\rho_{q,q}(\mathbf x_i)\bigr].
\end{equation*}
\begin{proof}
Substituting Eq.~\ref{eq:wflow_reference_field} cancels the two reference maps in the decomposition. At $\bm\pi=\bm\pi_{\mathrm{ref}}$, the candidate barycenter equals the data barycenter, so only the self-transport correction remains. Adding $\mathbf x_i$ to the equality with $hV^\rho_{q,p}(\mathbf x_i)$ gives the final condition.
\end{proof}
The identities also hold for samples and barycentric maps in the feature space of $\psi$, using the corresponding pushforward distributions. TDAC uses the coupling-defined displacement and incorporates the supplied energy in the progress constraint. Regression applies the target step and normalization in Eq.~\ref{eq:shared_targets}.

\paragraph{Minimum-action update.}
Fix the batch, $\mathbf J$, $\mathbf R\succ0$, and $\mathbf g_k$. Suppose $\mathcal P_k(\bm\pi_{\mathrm{ref}})>0$, $0<\gamma\leq1$, and $\mathbf g_k^\top \mathbf K\mathbf g_k>0$. To isolate the role of parameter action, consider
\begin{equation*}
\min_{\delta\bm\theta}\quad\tfrac12\delta\bm\theta^\top \mathbf R\delta\bm\theta,
 \qquad\text{s.t.}\quad -\mathbf g_k^\top \mathbf J\delta\bm\theta
 \geq\gamma\mathcal P_k(\bm\pi_{\mathrm{ref}}).
\end{equation*}
The unique minimizer and its linearized batch response satisfy
\begin{equation}
 \delta\bm\theta=-\frac{\gamma\mathcal P_k(\bm\pi_{\mathrm{ref}})}
 {\mathbf g_k^\top \mathbf K\mathbf g_k}\mathbf R^{-1}\mathbf J^\top \mathbf g_k,
 \qquad
 \mathbf J\delta\bm\theta=-\frac{\gamma\mathcal P_k(\bm\pi_{\mathrm{ref}})}
 {\mathbf g_k^\top \mathbf K\mathbf g_k}\mathbf K\mathbf g_k.
\label{eq:app_single_constraint_action_descent}
\end{equation}
\begin{proof}
Weighted Cauchy--Schwarz and $\mathbf K=\mathbf J\mathbf R^{-1}\mathbf J^\top$ give
\begin{equation*}
 (-\mathbf g_k^\top \mathbf J\delta\bm\theta)^2
 \leq(\mathbf g_k^\top \mathbf K\mathbf g_k)
       (\delta\bm\theta^\top \mathbf R\delta\bm\theta).
\end{equation*}
Every feasible increment has cost at least
$\gamma^2\mathcal P_k(\bm\pi_{\mathrm{ref}})^2/(2\mathbf g_k^\top \mathbf K\mathbf g_k)$.
The increment in Eq.~\ref{eq:app_single_constraint_action_descent} attains this cost with equality in the constraint. Positive-definite $\mathbf R$ ensures strict convexity and uniqueness. Applying $\mathbf J$ gives the batch response.
\end{proof}
Equation~\ref{eq:app_single_constraint_action_descent} gives the minimum-cost parameter update for a single batch progress constraint. Its feature response follows $-\mathbf K\mathbf g_k$, with cross-sample effects carried by $\mathbf K$. TDAC selects a coupling using Eq.~\ref{eq:progress_constrained_batch_atg} and fits its fixed targets. Equation~\ref{eq:action_executed_progress} bounds the first-order progress of the actual optimizer increment.

For a fixed current Jacobian, the tangent projection residual is
\begin{equation}
        \mathcal E(\bm\pi;\bm\theta,\mathbf Z) =\min_{\delta\bm\theta} \left\|\mathbf J\delta\bm\theta-\mathbf d_{\bm\pi}\right\|_2^2=\left\|(\mathbf I-\mathbf P)\mathbf d_{\bm\pi}\right\|_2^2, \qquad  \mathbf P=\mathbf J\mathbf J^\dagger ,
    \label{eq:projection_gap}
\end{equation}
where $\dagger$ denotes the Moore--Penrose pseudoinverse and $\mathbf P$ projects orthogonally onto $\operatorname{range}(\mathbf J)$. The residual measures target motion outside the shared tangent space, even for generators injective in their latent inputs.

\begin{proposition}[Realizability and isotropic recovery]
\label{prop:batch_realizability}
The projection gap in Eq.~\ref{eq:projection_gap} is the squared distance from $\mathbf d_{\bm\pi}$ to $\operatorname{range}(\mathbf J)$.
For $N=M$ with uniform weights $a_i=b_j=1/N$, if $\mathbf K=\kappa\mathbf I$, minimizing $\mathcal A_\varepsilon$ over deterministic balanced assignments is equivalent to minimizing the quadratic Euclidean assignment cost.
\end{proposition}

\begin{proof}[Proof of Proposition~\ref{prop:batch_realizability}]
The closest point in $\operatorname{range}(\mathbf J)$ is $\mathbf P\mathbf d_{\bm\pi}$. By orthogonality,
\begin{equation}
    \min_{\delta\bm\theta}\|\mathbf J\delta\bm\theta-\mathbf d_{\bm\pi}\|_2^2
    =\|(\mathbf I-\mathbf P)\mathbf d_{\bm\pi}\|_2^2.
    \label{eq:app_batch_projection}
\end{equation}
The first statement follows from Eq.~\ref{eq:app_batch_projection}.

Now suppose $N=M$, $a_i=b_j=1/N$, $\mathbf K=\kappa\mathbf I$, and the couplings are deterministic balanced assignments.
In this proof, $\mathbf u_i=\mathbf U_{i,:}^{\top}=\psi(\mathbf x_i)$ and $\mathbf v_j=\mathbf Y_{j,:}^{\top}$ denote the generated and data vectors in planning-feature space. For an assignment $\sigma_{\mathrm{perm}}$, the barycentric endpoint is $\mathbf v_{\sigma_{\mathrm{perm}}(i)}$, and Eq.~\ref{eq:batch_atg_value} becomes
\begin{equation}
    \mathcal A_\varepsilon(\bm\pi)
    =
    \frac{1}{2(\kappa+\varepsilon)}
    \sum_{i=1}^N\|\mathbf v_{\sigma_{\mathrm{perm}}(i)}-\mathbf u_i\|_2^2.
    \label{eq:app_batch_isotropic_value}
\end{equation}
The positive prefactor is independent of the assignment, so both objectives have the same minimizers.
\end{proof}

\begin{proposition}[Anisotropic shared-parameter separation]
\label{prop:anisotropic}
For every $C>0$, there exist finite two-point supports, a positive anisotropic batch mobility, and $\varepsilon>0$ such that the Euclidean deterministic assignment has local generator action more than $C$ times that of a deterministic cross assignment.
\end{proposition}

\begin{proof}[Proof of Proposition~\ref{prop:anisotropic}]
In this construction, $\mathbf u_i$ and $\mathbf v_j$ denote local planning-feature vectors. Let two equally weighted sources be $\mathbf u_1=(0,0)$ and $\mathbf u_2=(u,v)$, and let the targets be $\mathbf v_1=(u,0)$ and $\mathbf v_2=(0,v)$ with $0<u<v$.
The Euclidean assignment has batch displacement $\mathbf d_{\mathrm{Euc}}=(u,0,-u,0)^\top$, while the cross assignment has $\mathbf d_{\mathrm{cross}}=(0,v,0,-v)^\top$.
The Euclidean assignment is preferred because $\|\mathbf d_{\mathrm{Euc}}\|_2^2=2u^2<2v^2=\|\mathbf d_{\mathrm{cross}}\|_2^2$.

Choose the positive batch mobility $\mathbf K=\operatorname{diag}(\kappa,1,\kappa,1)$ with $\kappa>0$.
Eq.~\ref{eq:batch_atg_value} gives
\begin{equation}
    \mathcal A_\varepsilon(\bm\pi_{\mathrm{Euc}})
    =\frac{u^2}{\kappa+\varepsilon},
    \qquad
    \mathcal A_\varepsilon(\bm\pi_{\mathrm{cross}})
    =\frac{v^2}{1+\varepsilon}.
    \label{eq:app_batch_anisotropic_values}
\end{equation}
The ratio between the two action values in Eq.~\ref{eq:app_batch_anisotropic_values} is
\begin{equation}
    \frac{\mathcal A_\varepsilon(\bm\pi_{\mathrm{Euc}})}{\mathcal A_\varepsilon(\bm\pi_{\mathrm{cross}})}
    =
    \frac{u^2(1+\varepsilon)}{v^2(\kappa+\varepsilon)}.
    \label{eq:app_batch_anisotropic_ratio}
\end{equation}
For fixed $0<u<v$, the ratio in Eq.~\ref{eq:app_batch_anisotropic_ratio} tends to infinity as $\varepsilon$ and $\kappa$ tend to zero through positive values.
\end{proof}

For fractional couplings, $\mathbf d_{\bm\pi}$ interpolates between the two deterministic displacements. Equation~\ref{eq:progress_constrained_batch_atg} gives the regularized problem with entropy and persistence.

\begin{proposition}[Mobility error to action regret]
\label{prop:batch_regret}
Write $\mathcal J_{\mathbf K}$ for the objective in Eq.~\ref{eq:progress_constrained_batch_atg} evaluated with mobility $\mathbf K$. Both objectives use its feasible set $\mathcal Q$ and the same previous plan on aligned supports.
Let $\bm\pi^\star$ and $\bm\pi_{\widehat{\mathbf K}}^\star$ be exact minimizers of $\mathcal J_{\mathbf K}$ and $\mathcal J_{\widehat{\mathbf K}}$ on $\mathcal Q$.
If $\|\mathbf d_{\bm\pi}\|_2\leq D$ for $\bm\pi\in\mathcal Q$ and $\|\widehat{\mathbf K}-\mathbf K\|_{\mathrm{op}}\leq\delta<\varepsilon$, where $\|\cdot\|_{\mathrm{op}}$ is the spectral norm and $D$ is a scalar bound,
\begin{equation}
    \mathcal J_{\mathbf K}(\bm\pi_{\widehat{\mathbf K}}^\star)-\mathcal J_{\mathbf K}(\bm\pi^\star)
    \leq
    \frac{D^2\delta}{\varepsilon(\varepsilon-\delta)},
    \label{eq:batch_plan_regret}
\end{equation}
The bound concerns exact minimizers. Finite-iteration solver error is included in Eq.~\ref{eq:app_action_implemented_selection}.
\end{proposition}

\begin{proof}[Proof of Proposition~\ref{prop:batch_regret}]
Let
\begin{equation*}
    \Delta_{\mathbf K}
    =
    \sup_{\bm\pi\in\mathcal Q}
    |\mathcal A_{\varepsilon,\widehat{\mathbf K}}(\bm\pi)-\mathcal A_{\varepsilon,\mathbf K}(\bm\pi)|.
\end{equation*}
The resolvent identity gives
\begin{equation}
    (\widehat{\mathbf K}^{\varepsilon})^{-1}-(\mathbf K^\varepsilon)^{-1}
    =
    -(\widehat{\mathbf K}^{\varepsilon})^{-1}
    (\widehat{\mathbf K}-\mathbf K)
    (\mathbf K^\varepsilon)^{-1}.
    \label{eq:app_batch_resolvent_identity}
\end{equation}
If $\|\widehat{\mathbf K}-\mathbf K\|_{\mathrm{op}}\leq\delta<\varepsilon$, Eq.~\ref{eq:app_batch_resolvent_identity} gives
\begin{equation}
    \left\|
    (\widehat{\mathbf K}^{\varepsilon})^{-1}-(\mathbf K^\varepsilon)^{-1}
    \right\|_{\mathrm{op}}
    \leq\frac{\delta}{\varepsilon(\varepsilon-\delta)}.
    \label{eq:app_batch_resolvent_bound}
\end{equation}
From Eq.~\ref{eq:app_batch_resolvent_bound},
\begin{equation}
    \Delta_{\mathbf K}
    \leq
    \frac{D^2\delta}{2\varepsilon(\varepsilon-\delta)}.
    \label{eq:app_batch_uniform_value_error}
\end{equation}
The two objectives differ only in the action term.
Optimality of $\bm\pi_{\widehat{\mathbf K}}^\star$ under $\widehat{\mathbf K}$ and Eq.~\ref{eq:app_batch_uniform_value_error} yield
\[
    \mathcal J_{\mathbf K}(\bm\pi_{\widehat{\mathbf K}}^\star)-\mathcal J_{\mathbf K}(\bm\pi^\star) \leq \left[\mathcal J_{\mathbf K}(\bm\pi_{\widehat{\mathbf K}}^\star)-\mathcal J_{\widehat{\mathbf K}}(\bm\pi_{\widehat{\mathbf K}}^\star)\right] + \left[\mathcal J_{\widehat{\mathbf K}}(\bm\pi^\star)-\mathcal J_{\mathbf K}(\bm\pi^\star)\right] \leq 2\Delta_{\mathbf K},
\]
where each action difference contributes at most $\Delta_{\mathbf K}$. This proves Eq.~\ref{eq:batch_plan_regret}.
\end{proof}

For the random parameter directions in Eq.~\ref{eq:active_probe_mobility}, $\E[\mathbf r_s\mathbf r_s^\top]=\mathbf K$ at a fixed batch and linearization.
A bound on the total mobility error supplies $\delta$ in Proposition~\ref{prop:batch_regret}, including sampling, low-rank truncation, isotropic regularization, and estimate-reuse errors.
The bound separates error in estimating $\mathbf K$ from coupling optimization.

\section{Convergence of the Action Dynamics}
\label{app:convergence}

We compare planned progress with the progress of an actual update. The differentiable objective $E(\bm\theta)=\mathcal F(\mathbf u_{\bm\theta})$ is fixed across iterations by holding the latent inputs, reference samples, weights, and feature maps fixed. At step $k$, $\bm\theta_{k+1}=\bm\theta_k+\delta\bm\theta_k$ and $\mathbf g_k=\nabla_{\mathbf u}\mathcal F(\mathbf u_{\bm\theta_k})$, so $\nabla_{\bm\theta}E(\bm\theta_k)=\mathbf J^\top\mathbf g_k$. At step $k$, $\mathbf R_k$ denotes the control metric $\mathbf R$ used for that update. Each local comparison fixes the parameter subset and linearization point and uses the exact Jacobian and mobility. The increment includes the learning rate, optimizer history, and weight decay. Parameters outside the subset stay fixed.

\subsection{Local Realization and Execution Error}
\label{app:nonlinear_realization}

Let $\mathbf J(\bm\theta)$ be the feature Jacobian in Eq.~\ref{eq:batch_feature_map}.
\begin{proposition}[Local nonlinear realization error]
\label{prop:nonlinear_realization}
Suppose $\mathbf u_{\bm\theta}$ is continuously differentiable along the actual update segment and, for some $L_J\geq0$,
\begin{equation}
 \|\mathbf J(\bm\theta+s\delta\bm\theta)-\mathbf J(\bm\theta)\|_{\mathrm{op}}
 \leq L_Js\|\delta\bm\theta\|_2,\qquad 0\leq s\leq1.
 \label{eq:app_lipschitz_jacobian}
\end{equation}
The finite-batch linearization error satisfies
\begin{equation}
 \|\mathbf u_{\bm\theta+\delta\bm\theta}-\mathbf u_{\bm\theta}-\mathbf J(\bm\theta)\delta\bm\theta\|_2
 \leq\frac{L_J}{2}\|\delta\bm\theta\|_2^2.
 \label{eq:app_nonlinear_error}
\end{equation}
\end{proposition}
\begin{proof}
By the fundamental theorem of calculus, the linearization error is
\[
 \mathbf u_{\bm\theta+\delta\bm\theta}-\mathbf u_{\bm\theta}-\mathbf J(\bm\theta)\delta\bm\theta
 =\int_0^1[\mathbf J(\bm\theta+s\delta\bm\theta)-\mathbf J(\bm\theta)]\delta\bm\theta\,\mathrm ds.
\]
By Eq.~\ref{eq:app_lipschitz_jacobian}, the integral has norm at most $L_J\|\delta\bm\theta\|_2^2\int_0^1 s\,\mathrm ds$.
\end{proof}

\subsection{Execution Gap and Descent}
\label{app:action_convergence}

Fix a candidate coupling $\bm\pi$, $\varepsilon>0$, and the current $\mathbf J$ and $\mathbf R\succ0$. Let $\mathbf K=\mathbf J\mathbf R^{-1}\mathbf J^\top$ and $\mathbf H_{\mathbf R}=\mathbf R+\varepsilon^{-1}\mathbf J^\top\mathbf J$. For positive-definite $\mathbf H$, $\|\mathbf z\|_{\mathbf H}^2=\mathbf z^\top\mathbf H\mathbf z$.

\paragraph{Execution-gap identity.}
The objective minimized in Eq.~\ref{eq:batch_atg_problem} is the quadratic fitting cost $Q_{\bm\pi}$. Its minimum and excess cost satisfy
\begin{equation}
\begin{aligned}
Q_{\bm\pi}(\delta\bm\theta)
 &=\tfrac12(\delta\bm\theta)^\top\mathbf R\delta\bm\theta
   +\tfrac1{2\varepsilon}\|\mathbf J\delta\bm\theta-\mathbf d_{\bm\pi}\|_2^2,\\
 (\delta\bm\theta)^\star_{\bm\pi}
 &=\mathbf R^{-1}\mathbf J^\top(\mathbf K^\varepsilon)^{-1}\mathbf d_{\bm\pi},\qquad Q_{\bm\pi}(\delta\bm\theta)=\mathcal A_\varepsilon(\bm\pi)+\Delta_{\bm\pi},
\end{aligned}
\label{eq:app_action_execution_gap}
\end{equation}
where $\Delta_{\bm\pi}=\tfrac12\|\delta\bm\theta-(\delta\bm\theta)^\star_{\bm\pi}\|_{\mathbf H_{\mathbf R}}^2\geq0$ is the execution gap and $\mathbf H_{\mathbf R}\in\mathbb R^{P\times P}$ is positive definite.
\begin{proof}[Proof of Eq.~\ref{eq:app_action_execution_gap}]
Expanding the squared residual yields
\[
 Q_{\bm\pi}(\delta\bm\theta)
 =\tfrac12(\delta\bm\theta)^\top\mathbf H_{\mathbf R}\delta\bm\theta
  -\varepsilon^{-1}(\delta\bm\theta)^\top\mathbf J^\top\mathbf d_{\bm\pi}
  +\tfrac1{2\varepsilon}\|\mathbf d_{\bm\pi}\|_2^2.
\]
Since $\mathbf R\succ0$, the Hessian $\mathbf H_{\mathbf R}$ is positive definite. The unique minimizer solves
\[
 \nabla_{\delta\bm\theta}Q_{\bm\pi}
 =\mathbf H_{\mathbf R}\delta\bm\theta-\varepsilon^{-1}\mathbf J^\top\mathbf d_{\bm\pi}=\mathbf0.
\]
The identity
\[
 \mathbf H_{\mathbf R}\mathbf R^{-1}\mathbf J^\top
 =\mathbf J^\top+\varepsilon^{-1}\mathbf J^\top\mathbf K
 =\varepsilon^{-1}\mathbf J^\top(\mathbf K^\varepsilon)
\]
verifies the stated optimal increment. Its objective value is $\mathcal A_\varepsilon(\bm\pi)$ from Eq.~\ref{eq:batch_atg_value}.
Expanding about this minimizer gives
\begin{align*}
 Q_{\bm\pi}(\delta\bm\theta)
 &=Q_{\bm\pi}((\delta\bm\theta)^\star_{\bm\pi})
   +[\mathbf H_{\mathbf R}(\delta\bm\theta)^\star_{\bm\pi}-\varepsilon^{-1}\mathbf J^\top\mathbf d_{\bm\pi}]^\top
       [\delta\bm\theta-(\delta\bm\theta)^\star_{\bm\pi}]\\
 &\quad+\tfrac12\|\delta\bm\theta-(\delta\bm\theta)^\star_{\bm\pi}\|_{\mathbf H_{\mathbf R}}^2\\
 &=\mathcal A_\varepsilon(\bm\pi)+\Delta_{\bm\pi}.
\end{align*}
The linear term vanishes by the optimality condition, and the final quadratic term is nonnegative.
\end{proof}

\paragraph{Bounds on the actual response.}
Assume $\mathbf R\succeq c_R\mathbf I$ with $c_R>0$. Each term of $Q_{\bm\pi}$ is nonnegative, so
\[
 \tfrac{c_R}{2}\|\delta\bm\theta\|_2^2\leq Q_{\bm\pi}(\delta\bm\theta),\qquad
 \tfrac1{2\varepsilon}\|\mathbf J\delta\bm\theta-\mathbf d_{\bm\pi}\|_2^2\leq Q_{\bm\pi}(\delta\bm\theta).
\]
Solving these inequalities for the update and residual norms gives
\begin{equation}
    \|\delta\bm\theta\|_2^2\leq\frac{2}{c_R}Q_{\bm\pi}(\delta\bm\theta), \qquad  \|\mathbf J\delta\bm\theta-\mathbf d_{\bm\pi}\|_2 \leq\sqrt{2\varepsilon Q_{\bm\pi}(\delta\bm\theta)}.
\label{eq:app_executed_norm_bounds}
\end{equation}
The Hessian also satisfies $\mathbf H_{\mathbf R}\succeq\varepsilon^{-1}\mathbf J^\top\mathbf J$. Applying this matrix inequality to the difference from the optimal increment gives
\[
 \|\mathbf J[\delta\bm\theta-(\delta\bm\theta)^\star_{\bm\pi}]\|_2^2
 \leq\varepsilon\|\delta\bm\theta-(\delta\bm\theta)^\star_{\bm\pi}\|_{\mathbf H_{\mathbf R}}^2
 =2\varepsilon\Delta_{\bm\pi}.
\]
With $\mathbf V^\varepsilon=\mathbf J(\delta\bm\theta)^\star_{\bm\pi}$ from Eq.~\ref{eq:action_interaction_field}, this proves
\begin{equation}
 \|\mathbf J\delta\bm\theta-\mathbf V^\varepsilon\|_2\leq\sqrt{2\varepsilon\Delta_{\bm\pi}}.
 \label{eq:app_linear_response_gap}
\end{equation}
Under the assumptions of Proposition~\ref{prop:nonlinear_realization}, adding and subtracting $\mathbf J\delta\bm\theta$ and applying the triangle inequality yields
\begin{equation}
    \|\mathbf u_{\bm\theta+\delta\bm\theta}-\mathbf u_{\bm\theta}-\mathbf V^\varepsilon\|_2 \leq\|\mathbf u_{\bm\theta+\delta\bm\theta}-\mathbf u_{\bm\theta}-\mathbf J\delta\bm\theta\|_2 +\|\mathbf J\delta\bm\theta-\mathbf V^\varepsilon\|_2 \leq\frac{L_J}{2}\|\delta\bm\theta\|_2^2+\sqrt{2\varepsilon\Delta_{\bm\pi}}.
\label{eq:app_actual_response_gap}
\end{equation}
The terms bound the linearization error and the effect of deviating from the optimal local increment.

\begin{proposition}[Executed-update descent under a selected plan]
\label{prop:executed_descent}
Suppose $E$ has an $L_E$-Lipschitz gradient along the actual update segment and $\mathbf R_k\succeq c_R\mathbf I$ for $c_R>0$.
For the selected plan $\widehat{\bm\pi}_k$, let $S_k=Q_{\widehat{\bm\pi}_k}(\delta\bm\theta_k)=\mathcal A_\varepsilon(\widehat{\bm\pi}_k)+\Delta_k$. Here $\Delta_k$ is the execution gap in Eq.~\ref{eq:app_action_execution_gap} for $\widehat{\bm\pi}_k$ at the actual increment $\delta\bm\theta_k$.
The actual update satisfies
\begin{equation}
    E(\bm\theta_{k+1})\leq E(\bm\theta_k)-B_k, \qquad  B_k=\mathcal P_k(\widehat{\bm\pi}_k)-\|\mathbf g_k\|_2\sqrt{2\varepsilon S_k} -\frac{L_E}{c_R}S_k.
\label{eq:tdac-descent-certificate}
\end{equation}
For a feasible returned plan with $\mathcal P_k(\bm\pi_{\mathrm{ref}})>0$, the progress constraint gives
\begin{equation}
 E(\bm\theta_{k+1})-E(\bm\theta_k)
 \leq-\gamma\mathcal P_k(\bm\pi_{\mathrm{ref}})
      +\|\mathbf g_k\|_2\sqrt{2\varepsilon S_k}+\frac{L_E}{c_R}S_k.
 \label{eq:app_progress_constrained_descent}
\end{equation}
\end{proposition}
\begin{proof}
The descent lemma for the parameter-space objective gives
\[
 E(\bm\theta_k+\delta\bm\theta_k)
 \leq E(\bm\theta_k)+\nabla_{\bm\theta}E(\bm\theta_k)^\top\delta\bm\theta_k
       +\tfrac{L_E}{2}\|\delta\bm\theta_k\|_2^2.
\]
The chain rule and the definition of planned progress imply
\[
    \nabla_{\bm\theta}E(\bm\theta_k)^\top\delta\bm\theta_k =\mathbf g_k^\top\mathbf J\delta\bm\theta_k =\mathbf g_k^\top\mathbf d_{\widehat{\bm\pi}_k} +\mathbf g_k^\top(\mathbf J\delta\bm\theta_k-\mathbf d_{\widehat{\bm\pi}_k}) =-\mathcal P_k(\widehat{\bm\pi}_k) +\mathbf g_k^\top(\mathbf J\delta\bm\theta_k-\mathbf d_{\widehat{\bm\pi}_k}).
\]
Cauchy--Schwarz and Eq.~\ref{eq:app_executed_norm_bounds} give
\[
 \mathbf g_k^\top(\mathbf J\delta\bm\theta_k-\mathbf d_{\widehat{\bm\pi}_k})
 \leq\|\mathbf g_k\|_2\sqrt{2\varepsilon S_k},\qquad
 \tfrac{L_E}{2}\|\delta\bm\theta_k\|_2^2\leq\frac{L_E}{c_R}S_k.
\]
Substitution yields Eq.~\ref{eq:tdac-descent-certificate}, including fallback under the same assumptions. If the returned plan is feasible and reference progress is positive, the progress constraint yields Eq.~\ref{eq:app_progress_constrained_descent}.
\end{proof}

\paragraph{Comparison at a common generator state.}
Fix the objective, generator state, Jacobian, metric, $\varepsilon$, and smoothness bounds for TDAC and the reference method. For $m\in\{\mathrm{TD},\mathrm W\}$, write $Q_m=Q_{\bm\pi_m}(\delta\bm\theta_m)$. Subtracting their bounds from Eq.~\ref{eq:tdac-descent-certificate} gives
\begin{align*}
 B_{\mathrm{TD}}-B_{\mathrm W}
 &=\mathcal P_{\mathrm{TD}}-\mathcal P_{\mathrm W}
  +\|\mathbf g\|_2\sqrt{2\varepsilon}(\sqrt{Q_{\mathrm W}}-\sqrt{Q_{\mathrm{TD}}})
  +\frac{L_E}{c_R}(Q_{\mathrm W}-Q_{\mathrm{TD}})\\
 &\geq-(1-\gamma)\mathcal P_{\mathrm W}
  +\|\mathbf g\|_2\sqrt{2\varepsilon}(\sqrt{Q_{\mathrm W}}-\sqrt{Q_{\mathrm{TD}}})
  +\frac{L_E}{c_R}(Q_{\mathrm W}-Q_{\mathrm{TD}}),
\end{align*}
where the last line assumes $\mathcal P_{\mathrm{TD}}\geq\gamma\mathcal P_{\mathrm W}$. A reduction in executed fitting cost improves the descent bound when the reduction offsets the permitted progress loss.

\subsection{Conditional Stationarity and Convergence}
\label{app:conditional_convergence}

\begin{proposition}[Conditional stationarity and convergence]
\label{prop:tdac_training_convergence}
Let $E_\star=\inf_{\bm\theta}E(\bm\theta)>-\infty$. Assume the hypotheses of Proposition~\ref{prop:executed_descent} hold at every step with common constants $L_E$ and $c_R$. Suppose, for a constant $\kappa>0$,
\begin{equation}
 B_k\geq\kappa\|\nabla_{\bm\theta}E(\bm\theta_k)\|_2^2\qquad\text{for every }k.
 \label{eq:app_action_certificate_condition}
\end{equation}
For every integer $T\geq1$,
\begin{equation}
 \min_{0\leq k<T}\|\nabla_{\bm\theta}E(\bm\theta_k)\|_2^2
 \leq\frac{E(\bm\theta_0)-E_\star}{\kappa T},\qquad
 \lim_{k\to\infty}\|\nabla_{\bm\theta}E(\bm\theta_k)\|_2=0.
 \label{eq:app_conditional_stationarity}
\end{equation}
If the Polyak--\L{}ojasiewicz condition
\begin{equation}
 \|\nabla_{\bm\theta}E(\bm\theta_k)\|_2^2\geq2\mu[E(\bm\theta_k)-E_\star]
 \label{eq:app_pl_condition}
\end{equation}
holds along the iterates for a constant $\mu>0$ with $0<2\mu\kappa<1$, the objective values satisfy
\begin{equation}
\begin{aligned}
 E(\bm\theta_k)-E_\star
 \leq(1-2\mu\kappa)^k[E(\bm\theta_0)-E_\star].
\end{aligned}
\label{eq:tdac-conditional-linear-rate}
\end{equation}
\end{proposition}
\begin{proof}
Combining descent with Eq.~\ref{eq:app_action_certificate_condition} gives $E(\bm\theta_{k+1})\leq E(\bm\theta_k)-\kappa\|\nabla_{\bm\theta}E(\bm\theta_k)\|_2^2$.
Summing over $k=0,\ldots,T-1$ telescopes the objective differences:
\[
 \kappa\sum_{k=0}^{T-1}\|\nabla_{\bm\theta}E(\bm\theta_k)\|_2^2 \leq\sum_{k=0}^{T-1}[E(\bm\theta_k)-E(\bm\theta_{k+1})]  =E(\bm\theta_0)-E(\bm\theta_T) \leq E(\bm\theta_0)-E_\star.
\]
The minimum of the squared gradient norms is bounded by their average, so
\[
 \min_{0\leq k<T}\|\nabla_{\bm\theta}E(\bm\theta_k)\|_2^2
 \leq\frac1T\sum_{k=0}^{T-1}\|\nabla_{\bm\theta}E(\bm\theta_k)\|_2^2
 \leq\frac{E(\bm\theta_0)-E_\star}{\kappa T}.
\]
The bound holds for all $T$, so the squared gradient norms are summable and the gradient norms converge to zero.

With Eq.~\ref{eq:app_pl_condition}, sufficient decrease becomes
\[
 E(\bm\theta_{k+1})-E_\star
 \leq E(\bm\theta_k)-E_\star-2\mu\kappa[E(\bm\theta_k)-E_\star]
 =(1-2\mu\kappa)[E(\bm\theta_k)-E_\star].
\]
Iterating this contraction gives Eq.~\ref{eq:tdac-conditional-linear-rate}.
\end{proof}

\paragraph{A sufficient progress and cost condition.}
Fix $0<\gamma\leq1$. Suppose the selected plans satisfy the progress constraint and there are constants $c_p>0$ and $c_e\geq0$, uniform along the actual updates, such that
\begin{equation}
    \mathcal P_k(\bm\pi_{\mathrm{ref}}) \geq c_p\|\nabla_{\bm\theta}E(\bm\theta_k)\|_2^2, \qquad  \|\mathbf g_k\|_2\sqrt{2\varepsilon S_k}+\frac{L_E}{c_R}S_k \leq c_e\|\nabla_{\bm\theta}E(\bm\theta_k)\|_2^2.
\label{eq:app_progress_cost_conditions}
\end{equation}
The progress constraint and Eq.~\ref{eq:tdac-descent-certificate} imply
\[
 B_k\geq\gamma\mathcal P_k(\bm\pi_{\mathrm{ref}})
       -\|\mathbf g_k\|_2\sqrt{2\varepsilon S_k}-\frac{L_E}{c_R}S_k
 \geq(\gamma c_p-c_e)\|\nabla_{\bm\theta}E(\bm\theta_k)\|_2^2.
\]
When $\gamma c_p>c_e$, Eq.~\ref{eq:app_action_certificate_condition} holds with $\kappa=\gamma c_p-c_e>0$.

\subsection{Approximate Planning and Action Error}
\label{app:planning_error}

For the implemented plan, write $\mathcal J_{\mathbf K}=\mathcal A_{\varepsilon,\mathbf K}+\Omega$, where $\Omega(\bm\pi)$ is the sum of the entropy and persistence penalties in Eq.~\ref{eq:progress_constrained_batch_atg}.
Use the same feasible set $\mathcal Q$ as Proposition~\ref{prop:batch_regret}, with finite regularizers, $\|\mathbf d_{\bm\pi}\|_2\leq D$, and symmetric $\widehat{\mathbf K}$ satisfying $\|\widehat{\mathbf K}-\mathbf K\|_{\mathrm{op}}\leq\delta<\varepsilon$.
Suppose the numerical plan is feasible and satisfies
\[
 \widehat{\bm\pi}\in\mathcal Q,\qquad
 \mathcal J_{\widehat{\mathbf K}}(\widehat{\bm\pi})
 \leq\inf_{\bm\pi\in\mathcal Q}\mathcal J_{\widehat{\mathbf K}}(\bm\pi)+\eta_{\mathrm{sol}},
 \qquad\eta_{\mathrm{sol}}\geq0.
\]
Assume $\mathcal P_k(\bm\pi_{\mathrm{ref}})>0$ and $0<\gamma\leq1$, so the reference plan belongs to $\mathcal Q$. The selected plan satisfies
\begin{equation}
    \mathcal A_{\varepsilon,\mathbf K}(\widehat{\bm\pi}) \leq\mathcal A_{\varepsilon,\mathbf K}(\bm\pi_{\mathrm{ref}}) +\Omega(\bm\pi_{\mathrm{ref}})-\Omega(\widehat{\bm\pi})  \quad+\frac{D^2\delta}{\varepsilon(\varepsilon-\delta)}+\eta_{\mathrm{sol}}.
\label{eq:app_action_implemented_selection}
\end{equation}
\begin{proof}
Apply Eq.~\ref{eq:app_batch_uniform_value_error} at the numerical and reference plans. The regularizers are unchanged, and the numerical plan is within $\eta_{\mathrm{sol}}$ of the estimated optimum. Thus
\[
    \mathcal J_{\mathbf K}(\widehat{\bm\pi}) \leq\mathcal J_{\widehat{\mathbf K}}(\widehat{\bm\pi}) +\frac{D^2\delta}{2\varepsilon(\varepsilon-\delta)} \leq\mathcal J_{\widehat{\mathbf K}}(\bm\pi_{\mathrm{ref}}) +\eta_{\mathrm{sol}}+\frac{D^2\delta}{2\varepsilon(\varepsilon-\delta)} \leq\mathcal J_{\mathbf K}(\bm\pi_{\mathrm{ref}}) +\eta_{\mathrm{sol}}+\frac{D^2\delta}{\varepsilon(\varepsilon-\delta)}.
\]
Separating $\mathcal J=\mathcal A_\varepsilon+\Omega$ proves Eq.~\ref{eq:app_action_implemented_selection}.
\end{proof}
The executed fitting cost includes the execution gap in Eq.~\ref{eq:app_action_execution_gap}.

\subsection{Measured Cost and Update Complexity}
\label{app:update_complexity}

\begin{proposition}[Conditional update complexity]
\label{prop:conditional_optimizer_rate}
Assume the sufficient-decrease and Polyak--\L{}ojasiewicz conditions of Proposition~\ref{prop:tdac_training_convergence} for method $m$, with $\kappa=\kappa_m$ and $0<2\mu\kappa_m<1$. Write $e_k=E(\bm\theta_k)-E_\star$ and fix an accuracy $0<\epsilon_{\mathrm{acc}}<e_0$.
A sufficient number of updates to reach $e_k\leq\epsilon_{\mathrm{acc}}$ is
\begin{equation}
 N_m(\epsilon_{\mathrm{acc}})
 =\left\lceil\frac{\log(e_0/\epsilon_{\mathrm{acc}})}{-\log(1-2\mu\kappa_m)}\right\rceil.
 \label{eq:conditional_action_rate}
\end{equation}
For a common objective, initial gap, and Polyak--\L{}ojasiewicz constant, $\kappa_{\mathrm{TD}}>\kappa_{\mathrm W}$ gives a smaller unrounded sufficient update bound and $N_{\mathrm{TD}}(\epsilon_{\mathrm{acc}})\leq N_{\mathrm W}(\epsilon_{\mathrm{acc}})$.
\end{proposition}
\begin{proof}
By Eq.~\ref{eq:tdac-conditional-linear-rate}, it suffices that $(1-2\mu\kappa_m)^ke_0\leq\epsilon_{\mathrm{acc}}$. Taking logarithms, dividing by $-\log(1-2\mu\kappa_m)>0$, and rounding up gives Eq.~\ref{eq:conditional_action_rate}. Its denominator grows with $\kappa_m$. Zero updates suffice if $e_0\leq\epsilon_{\mathrm{acc}}$, and the Polyak--\L{}ojasiewicz condition gives $e_k=0$ at zero gradient.
\end{proof}

With update time at most $\bar t_m$ and setup and evaluation overhead at most $\bar T_m^{\mathrm{extra}}$, at most $N_m(\epsilon_{\mathrm{acc}})$ updates give $T_m(\epsilon_{\mathrm{acc}})\leq\bar t_mN_m(\epsilon_{\mathrm{acc}})+\bar T_m^{\mathrm{extra}}$. Excluding overhead, TDAC has the smaller unrounded time bound when
$\bar t_{\mathrm{TD}}/\bar t_{\mathrm W}<[-\log(1-2\mu\kappa_{\mathrm{TD}})]/[-\log(1-2\mu\kappa_{\mathrm W})]$.

Applying the training-objective bound to FID requires linking training energy to Inception-feature discrepancy and controlling finite-sample evaluation error. Under a shared protocol, $k_m(\xi)$ counts updates through the first evaluation meeting the common FID threshold $\xi$, with $k_m(\xi)=\infty$ if none does. Each $t_{m,k}$ includes all work in update $k$, including any rejected rollouts and state restoration. The separate term $A_m$ covers setup and evaluation through the first qualifying result. For positive finite counts,
\begin{equation}
 T_m^{\mathrm{meas}}(\xi)=\sum_{k=0}^{k_m(\xi)-1}t_{m,k}+A_m,\qquad
 \operatorname{Speedup}_{\mathrm{steps}}(\xi)=k_{\mathrm W}(\xi)/k_{\mathrm{TD}}(\xi).
 \label{eq:app_action_measured_time}
\end{equation}
At $k_m(\xi)=0$, the empty sum leaves total time $A_m$. We report unreached thresholds separately.
For positive finite counts, write $k_m=k_m(\xi)$ and $\widetilde t_m=(k_m)^{-1}\sum_{k<k_m}t_{m,k}$. Subtract overheads and divide by $k_{\mathrm W}\widetilde t_{\mathrm{TD}}>0$ to obtain
\begin{equation*}
 T_{\mathrm{TD}}^{\mathrm{meas}}<T_{\mathrm W}^{\mathrm{meas}}
 \iff
 \frac{k_{\mathrm{TD}}}{k_{\mathrm W}}<
 \frac{\widetilde t_{\mathrm W}}{\widetilde t_{\mathrm{TD}}}
 +\frac{A_{\mathrm W}-A_{\mathrm{TD}}}{k_{\mathrm W}\widetilde t_{\mathrm{TD}}}.
\end{equation*}
With equal overhead and $\widetilde t_{\mathrm{TD}}=(1+c)\widetilde t_{\mathrm W}$, $c\geq0$, TDAC must save more than $c/(1+c)$ of W-Flow's updates. GPU-hours sum device time.

\section{Limitations and Broader Impacts}
\label{app:limitations_impacts}

\paragraph{Limitations.}
We evaluate class-conditional ImageNet $256\times256$ generation and two-dimensional distributions. Higher resolutions, text-to-image synthesis, and video remain untested. ImageNet generation quality depends on the frozen pretrained feature encoders. The planner optimizes local generator action under a single-step linearization. Planning trajectories across updates is future work.

\paragraph{Broader impacts.}
TDAC introduces a training-time approach to one-step generative modeling without adding planning computations at inference. Like other one-step generators, it can reduce sampling latency relative to iterative diffusion and flow-based models, with potential applications in creative tools, data augmentation, and visual content generation. At the same time, more accessible image generation may increase risks associated with synthetic media, including impersonation, misinformation, biased representations, and privacy concerns. Responsible deployment should consider appropriate safeguards, such as content provenance, disclosure of generated images, and careful management of training data.

\end{document}